%% file: main.tex
\documentclass[11pt]{article}

\usepackage[letterpaper,margin=1in]{geometry}
\usepackage[parfill]{parskip}

\usepackage{authblk}

\usepackage[toc,page,header]{appendix}

\usepackage{rotating}
\usepackage{subfig}
\usepackage{enumerate}
\usepackage{graphicx, multirow, graphics}
\usepackage[ruled,vlined,linesnumbered]{algorithm2e}
\usepackage{amsmath,bbm,amssymb,mathrsfs}
\usepackage{amsthm,amscd,float}    
\usepackage{makecell}
\usepackage{float}

\usepackage[utf8]{inputenc} 
\usepackage[T1]{fontenc}    
\usepackage[colorlinks,citecolor=blue,urlcolor=blue,linkcolor=blue,linktocpage=true]{hyperref}       
\usepackage{nicefrac}       
\usepackage{microtype}      
\usepackage{xcolor}         
\label{key}

\usepackage{comment}

\usepackage{epsfig}
\usepackage{graphicx}
\usepackage{wrapfig}
\usepackage{multirow}
\usepackage{graphicx}
\usepackage{caption}
\usepackage{array}
\usepackage{bbm}
\usepackage{color}
\usepackage{enumerate}
\usepackage{enumitem}
\usepackage{mathtools}
\usepackage{amsmath,amssymb,amsthm,bm}
\usepackage{thmtools}
\usepackage{amsfonts,graphicx}
\usepackage{mathrsfs}
\usepackage{amsmath,amssymb,amsfonts}
\usepackage{textcomp}
\usepackage{bbold}
\usepackage{graphicx}
\usepackage{textcomp}
\usepackage{xcolor}
\usepackage[ruled,vlined,linesnumbered]{algorithm2e}

\usepackage[authoryear]{natbib}
\usepackage{cleveref}

\usepackage[export]{adjustbox}
\DeclareMathOperator\erf{erf}

\newtheorem{remark-star}{Remark}
\newtheorem{remark-star-1}{Remark}

\newtheorem*{proof-sketch}{Proof Sketch}
\newtheorem{theorem}{Theorem}

\usepackage{booktabs}

\input{math_commands}

\author[1]{Yingyan~Zeng}
\affil[1]{Virginia Tech}
\author[2]{Xiaoyu~Chen}
\affil[2]{University of Buffalo}
\author[1]{Ran~Jin}

\title{Ensemble Active Learning by Contextual Bandits for AI Incubation in Manufacturing}
\begin{document}



\maketitle

\begin{abstract}
An Industrial Cyber-physical System (ICPS) provide a digital foundation for data-driven decision-making by artificial intelligence (AI) models.
However, the poor data quality (e.g., inconsistent distribution, imbalanced classes) of high-speed, large-volume data streams poses significant challenges to the online deployment of offline-trained AI models.
As an alternative, updating AI models online based on streaming data enables continuous improvement and resilient modeling performance.
However, for a supervised learning model (i.e., a base learner), it is labor-intensive to annotate all streaming samples to update the model. 
Hence, a data acquisition method is needed to select the data for annotation to ensure data quality while saving annotation efforts.
In the literature, active learning methods have been proposed to acquire informative samples.
Different acquisition criteria were developed for exploration of under-represented regions in the input variable space or exploitation of the well-represented regions for optimal estimation of base learners.
However, it remains a challenge to balance the exploration-exploitation trade-off under different online annotation scenarios.
On the other hand, an acquisition criterion learned by AI adapts itself to a scenario dynamically, but the ambiguous consideration of the trade-off limits its performance in frequently changing manufacturing contexts.
To overcome these limitations, we propose an ensemble active learning method by contextual bandits (\cbeal).
\cbeal incorporates a set of active learning agents (i.e., acquisition criteria) explicitly designed for exploration or exploitation by a weighted combination of their acquisition decisions.
The weight of each agent will be dynamically adjusted based on the usefulness of its decisions to improve the performance of the base learner.
With adaptive and explicit consideration of both objectives, \cbeal efficiently guides the data acquisition process by selecting informative samples to reduce the human annotation efforts.
Furthermore, we characterize the exploration and exploitation capability of the proposed agents theoretically.
The evaluation results in a numerical simulation study and a real case study demonstrate the effectiveness and efficiency of \cbeal in the manufacturing process modeling of the ICPS.
\end{abstract}

\section{Introduction}\label{Sec:Intro}
Industrial Cyber-physical Systems (ICPSs) integrate the cyber and physical worlds, which serve as the backbone of the Fourth Industrial Revolution \citep{colombo2017industrial}. 
By embracing the Internet of Things (IoT), an ICPS interconnects manufacturing equipment with ubiquitous sensors, actuators, and computing units, forming a low-cost, high-availability, and high-accessibility network \citep{zhang2019fog}. 
The high-speed and large-volume sensing data collected from such a network have advanced many data-driven decision-making methods to support manufacturing efficiency, quality improvement, and cost reduction. 
For example, artificial intelligence (AI) models such as support vector machine (SVM) and deep neural networks have been employed for quality modeling and process monitoring of fused deposition modeling (FDM) processes \citep{ gobert2018application, wang2021pyramid}, Aerosol® Jet Printing processes \citep{sun2017quality}, etc.
However, most AI models are proposed following an offline-training-online-deployment (OTOD) strategy, which is only effective when the quality of the training data set is guaranteed (e.g., training data can provide adequate estimations of the underlying true model of the variable relationships in supervised learning).
In practice, various factors can change such an underlying model in manufacturing processes (e.g., the degradation of manufacturing equipment or change of product design), which results in an erratic performance of AI models during their online deployment.

To improve the modeling performance, one can either build a dynamic model \citep{jin2019dynamic} or investigate an online model training mechanism to adapt existing models to online data streams.
However, constructing a dynamic model highly depends on the prior knowledge of the distribution changing patterns and root causes.
Instead of focusing on creating a better model, data-centric AI has been proposed as a more general approach to engineer the data needed to successfully build an AI model \citep{datacentric}.
From a holistic view, we envision a resilient AI system to identify and mitigate the performance fluctuation of AI models caused by abrupt changes (i.e., data distribution, learning algorithms, and computational resources), which jointly consider managing the data quality as well as adapting the existing models during the online deployment.
Therefore, in this paper we focus on online model training by actively acquiring samples to ensure data quality such that the resilient AI performance can be achieved.

Here, we focus on the supervised learning model as the base learner, where the data quality is considered from the aspect of representativeness (i.e., consistent distribution between the training and testing data sets) and class imbalance \citep{gupta2021data}.
In the offline-training step, the high-quality data set can be defined when there are sufficient representative samples for training, such as sufficient samples for multimodal distributions \citep{he2007fault}, a distribution of the training samples close to that of the testing samples, and when there are balanced class distributions \citep{branco2016survey}.
However, OTOD cannot support the AI modeling since in the context of high-speed, large-volume streaming data, the data quality of training data sets needs to be evaluated continuously.
\begin{figure}[!htb]
    \centering
    \captionsetup{justification=centering}
    \includegraphics[width=0.6\textwidth]{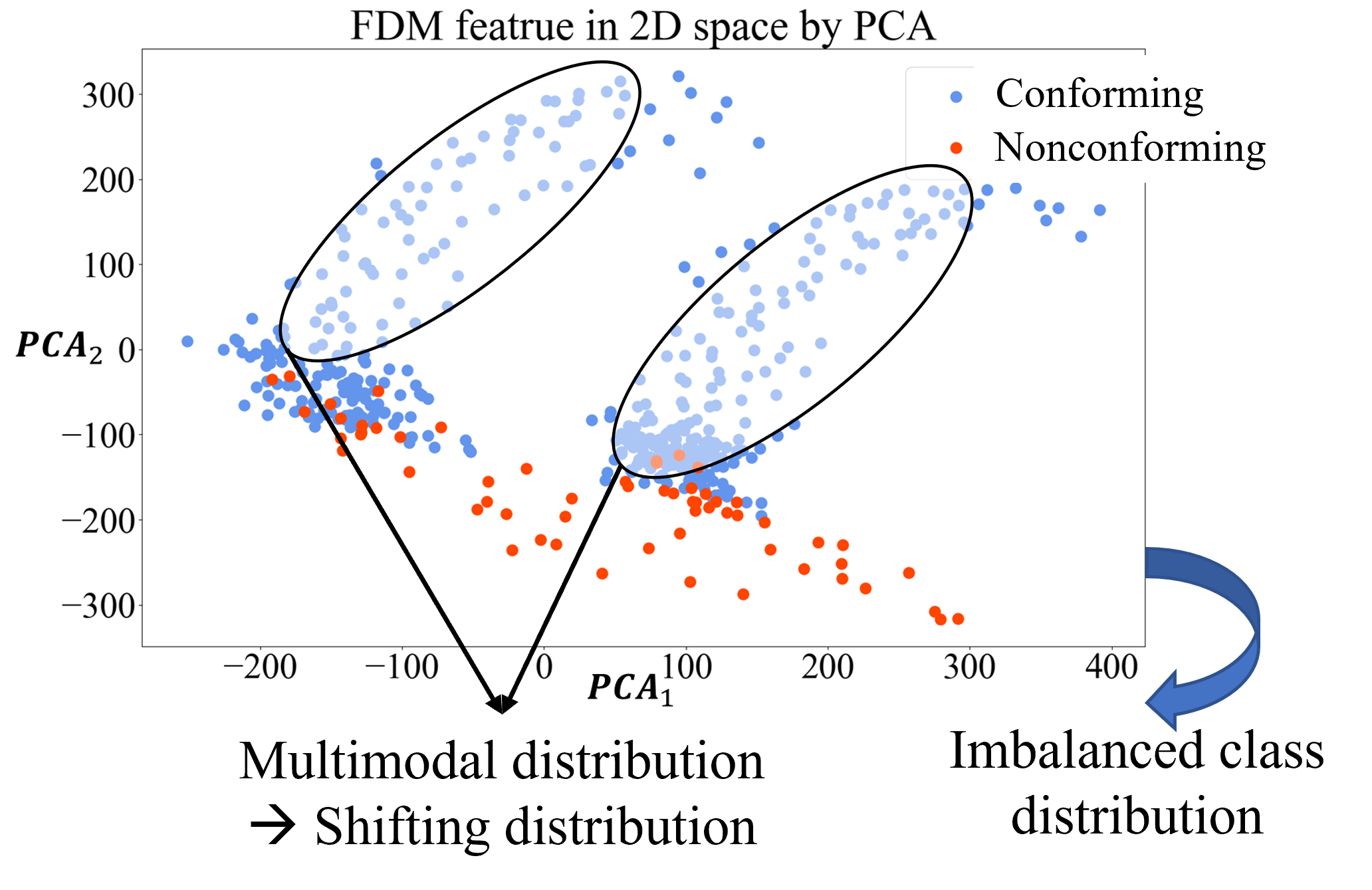}
    \caption{Distribution of FDM input data with reduced dimensions by Principle Component Analysis (PCA)} 
    \label{fig:motivation}
\end{figure}
As a motivation example, Fig.~\ref{fig:motivation} demonstrates the multimodal distribution and the imbalanced class of samples collected from a highly personalized FDM process, where the samples are projected to the principle component directions of the input variable space by Principle Component Analysis (PCA).
The two clusters are generated from two layers in FDM due to different product geometric designs.
The objective of the data analysis in Fig.~\ref{fig:motivation} is to use \textit{in situ} process variables to predict the layer-to-layer binary quality variable, which indicates the surface roughness as a classification problem.
Assume that the samples collected until time $t$ will be used to train the model, where the samples from only one cluster have been observed.
After time $t+1$, the streaming samples are from the other cluster, i.e., the FDM goes to the next layer with another design in that slice. 
The collected training data set is not representative due to the shift of distributions, thus resulting in a sudden decrease of the prediction performance of the pre-trained model.

Motivated by Fig. 1, it is important to actively select the streaming samples for annotation which ensures data quality.
In earlier studies, Design of Experiments (DoE) was proposed to improve the supervised learning models by generating samples to identify significant variables \citep{fisher1936design, jin2015ensemble}.
Recent efforts in data filtering, either model-based \citep{li2021cluster} or model-free methods \citep{trost1986statistically, liberty2016stratified} aim at accurately modeling the underlying system by sampling a subset with good representativeness of the population distribution.
However, DoE focus on actively generating the data while data filtering methods require a completely collected data set before the selection.
Neither method can be directly applied to acquire the streaming data in the ICPS.

To obtain high-quality data, humans also play an indispensable role in the data annotation with their domain knowledge.
In particular, the online data annotation requires real-time experimentation and a human-machine interface for domain experts to interact with \citep{tu2020better}, which is time-consuming and labor-intensive.
While automatic annotation methods employ semi-supervised learning methods to annotate the samples by the most confident predictions \citep{kostopoulos2018semi},
one potential disadvantage is the deteriorating performance of AI models caused by mislabelled cases \citep{carlini2021poisoning}.

Recognizing the significance of human-in-the-loop, we proposed an AI incubation framework \citep{chen2022inn}, designed to foster interaction between domain experts and AI systems throughout the training and deployment phases. 
This framework facilitates AI model development through human-AI collaboration, including model structure inspired by human decision-making processes \citep{chen2022inn}, and feature generation based on human visual searching patterns \citep{chen_2021}.
In this work, we further explore the aspect of training data selection and annotation within AI incubation.  
Here, the human annotator acts as an `incubator', labeling the acquired samples for training AI models.
Our aim is to reduce human efforts while efficiently improving the learning performance of the base learner.

To create an online data acquisition method, an acquisition criterion needs to be designed to determine whether a sample should be selected for human annotation in order to acquire only informative samples and to provide high-quality training data sets for the base learner.
This acquisition decision can be viewed as a dilemma between exploration and exploitation of the input variable space \citep{bondu2010exploration}.
Here, exploitation is defined as acquiring a sample around the conceptual boundary for boundary learning, whereas exploration is defined as acquiring a sample located in the under-represented region for the input variable space discovery. Exploitation-oriented criteria work well when the base learner can easily detect the important regions \citep{osugi2005balancing}.
Otherwise, exploration is required for more complex scenarios such as the exclusive XOR problem.
For the scenario of motivation example shown in Fig.~\ref{fig:motivation}, if we concentrate exclusively on the samples near the decision boundary for exploitation, samples from another cluster may be overlooked throughout the streaming process, leading to an inaccurate estimate of the decision boundary and poor performance in under-represented regions.
On the other hand, exclusive exploration will easily lead to a base learner with high uncertainty.
Therefore, a well-balanced exploration-exploitation trade-off is essential for guiding the online annotation of AI models with complex distribution in the input variable space \citep{loy2012stream}.

In this paper, we propose an ensemble active learning method by contextual bandits (\cbeal) to improve the exploration and exploitation trade-off under various scenarios,  thus guiding an efficient and effective human annotation process. 
In \cbeal, a set of active learning agents with human-designed criteria is incorporated by contextual bandits \citep{beygelzimer2011contextual}, where a joint acquisition decision is made by the weighted combination of individual decisions.
Here, we use "agents" as an umbrella term to refer the acquisition criteria in active learning methods.
The candidate active learning agents incorporated are designed to pursue an explicit objective of exploration or exploitation with theoretical justification, respectively.
Thus, during the annotation process, the weight (i.e., decision power) of each agent indicates the current tendency for exploration or exploitation, which will be updated dynamically by the bandits solver subject to the historical reward.
To improve the learning performance of the base learner, the reward is defined as the usefulness of the acquisition behaviour of \cbeal where acquiring a sample which would be wrongly predicted by the base learner is considered useful.
Therefore, the online data acquisition problem can be effectively addressed by \cbeal, which pursues the exploration and exploitation trade-off through the ensemble of a set of active learning agents.
In this sense, \cbeal is a generic active learning framework that reduces manual adjustment of active learning agents under frequently changed manufacturing data distributions.
In addition to improving learning performance, \cbeal also increases the interpretability of AI models from the data perspective \citep{zhao2019oui}, as the weight of each agent in each acquisition step explains whether the sample is annotated for exploration or exploitation.


The remainder of this paper is organized as follows. Section \ref{sec:reveiw} summarizes the related work. 
Section~\ref{sec:Method} introduces the proposed \cbeal method and provides the theoretical justification. 
Section~\ref{sec:sim} evaluates the performance of \cbeal by simulation studies. 
Section~\ref{sec:case} validates \cbeal via a real case study of online quality modeling of FDM in the ICPS. 
We conclude this work with some discussion of future work in Section~\ref{sec:conclusion}.

\section{Related Work}
\label{sec:reveiw}
\subsection{Online Model Updating in Industrial Cyber-physical Systems}
In the past decades, the ICPS has integrated physical manufacturing equipment with sensing and actuation networks as well as ubiquitous computational resources, which provides the digital foundation for the online updating of AI models \citep{rao2015online, wang2020online}.
With the streaming observational data and online computational resources, the online updating techniques of AI models have been investigated to enable the close modeling of manufacturing processes and facilitates the efficient decision-making in ICPSs.
For example, \citet{bastani2016online} proposed an online classification model for real-time monitoring in additive and semiconducting manufacturing processes.
\citet{wang2016large} developed a large-scale online multitask learning model to coordinate machine actions in the ICPS online.
Online model updating strategies have also been developed for model calibration and predictive maintenance \citep{li2018integration, xia2018recent}.
However, the aforementioned studies focus on developing the online updating algorithm of the AI model via Bayesian methods or distributed optimization methods to reduce the computational burden with large-volume streaming data, which are effective for unsupervised learning problems or supervised learning scenarios with easily collected responses.
Yet for many supervised learning scenarios in the ICPS, the passively collected data need to be annotated via real-time experimentation by domain experts, e.g., the inspection of a batch of 400 wafers may take more than 8 hours \citep{jin2012sequential}.
The lack of consideration of human annotation efforts renders these online updating methods inefficient for supervised AI models, especially in highly personalized manufacturing environments with rapid product and process changes \citep{alexopoulos2020digital}.


\subsection{Data Quality and Data Acquisition Methods}
Compared to the accuracy and efficiency of learning algorithms, validating and monitoring the quality of data fed to AI models is an equally important problem \citep{caveness2020tensorflow}.
Metrics for assessing the data quality for classification tasks include outlier detection, boundary complexity, label noise, shifting distribution, class imbalance, etc \citep{gupta2021data}. 
In the context of streaming data, \citet{caveness2020tensorflow} developed a data analysis and validation system to monitor significant changes between successive batches of the training data by summary statistics (i.e., mean, variance, etc.) with human investigation for a machine learning pipeline. 
However, without considering the informativeness of the data related to the AI model, the data collected by such a system cannot effectively improve the modeling performance.

On the other hand, to improve the performance of supervised learning models and reduce human annotation labor efforts, methods have been developed to facilitate effective data acquisition for high-quality informative data.
These methods include providing acquisition recommendations for human annotation (e.g., sequential design and active learning) and automatic annotation (e.g., semi-supervised learning) \citep{diete2017smart}, where limited approach suits the online streaming data. 
Sequential design focuses on selecting the samples in a sequential manner to achieve certain optimality criteria such as maximum entropy, maxmin distance \citep{lam2008sequential, stinstra2003constrained}.
For example, \citet{yan2020akm2d} proposed an adaptive sequential sampling method to balance sampling efforts between the exploration and exploitation of anomalous regions for anomaly detection in the ICPS.
However, these methods provide active recommendations that require experiments to be conducted at selected points in the input variable space, which cannot address passively collected data.
Semi-supervised learning has been employed to automate the annotating process such that the base learner can learn from both labeled and unlabeled data. \citep{zhu2009introduction}.
However, adding mislabelled cases by a semi-supervised learner to the training set may hamper the base learner's learning performance \citep{carlini2021poisoning}.
Due to the aforementioned limitations, we focus on active learning methods which provide acquisition recommendations for the passively collected data.
Active learning has been leveraged for minimizing the human effort as well as improving the modeling performance in various applications including human activity recognition \citep{adaimi2019leveraging}, threatening surveillance event detection \citep{loy2010stream}, wearable sensing platforms annotation \citep{solis2019human}, etc.

\subsection{Exploration and Exploitation in Active Learning}
Active learning reduces annotation efforts for supervised learning models by evaluating the informativeness of samples and acquiring the most informative ones \citep{settles2012active}.
The decision on whether one sample should be labelled can be viewed as a dilemma between the exploration and exploitation of the input variable space.
In the earlier work, most efforts exploit samples with large amount of information about the base learner as an exploitation-oriented strategy.
Metrics such as classification uncertainty\citep{lewis1994sequential}, margin\citep{balcan2007margin}, and entropy \citep{fu2013survey} of the base learner have been adopted to measure the informativeness and compared with corresponding thresholds to make the acquisition decision.
To achieve exploration for online streaming data, \citet{ienco2014high} modeled the local density of an sample to acquire the sample lying in a dense region with a small classification margin.
For trivial scenarios where only parts of the input variable space have to be known in order to perform optimally, exploitation-oriented acquisition criteria can be more effective to avoid exploring regions that are irrelevant for the decision boundary estimation \citep{thrun1995exploration}.
However, in nontrivial scenarios,  exploration is crucial to uncover relevant, unknown regions when the base learner's estimation of the decision boundary is imprecise.
Thus, the exploration-exploitation trade-off becomes vital under scenarios with exclusive XOR problem, clusterwise structure, imbalanced class distribution, etc \citep{osugi2005balancing, loy2012stream}.
Relying solely on either exploration or exploitation falls short in achieving optimal learning outcomes due to incompatibility with varied online annotation scenarios.
To achieve a compromise, a common acquisition strategy is to conduct exploration and exploitation simultaneously.
Considering two acquisition criteria which are dedicated to exploration (e.g., random sampling) and exploitation (e.g., uncertainty sampling) respectively, the compromise can be achieved by selecting one criterion with a certain probability for each streaming sample.
One typical example is the $\epsilon$-greedy policy which enforces the input variable space exploration with probability $\epsilon$ in each round \citep{thrun1995exploration}.
Representative sampling methods have also been designed as the combination of exploitation-oriented and exploration-oriented criteria \citep{wang2011active}.
However, the ambiguity of the weight of each objective requires further fine-tuning for each learning scenario.
To avoid ambiguity, \citet{loy2012stream} extended the Query-by-Committee (QBC) \citep{seung1992query} paradigm to a nonparametric Bayesian model to address unknown class discovery and imbalanced class distribution for the  online annotation.
However, without taking into account the informativeness (i.e., uncertainty) of a sample about the base learner, the proposed \textsc{QBC-PYP} cannot adjust the exploration-exploitation to the learning performance of the base learner.
In brief, the simple combination of active learning criterion
cannot handle challenging online data acquisition scenarios even with fine-tuning, which is due to the lack of compatibility with the data stream and the lack of adaptiveness to the learning performance of the base learner \citep{elreedy2019novel}.
Therefore, an ensemble of multiple criteria is desired to guide the dynamic exploration-exploitation trade-off in an adaptive and data-dependent manner.

As a promising approach, active learning has been recently formulated in the framework of reinforcement learning (RL) and multi-armed bandits where the objective is to learn the optimal acquisition criterion as a policy to maximize the cumulative reward \citep{ebert2012ralf}.
However, these methods can also lack systematic and explicit considerations for both exploration and exploitation objectives. 
\citet{wassermann2019ral} proposed Reinforced Active Learning (RAL) which modeled the stream-based active learning as a contextual bandits problem.
In RAL, a set of base learners was gathered as the committee to provide acquisition advice based on the certainty degree of the sample to each learner.
The acquisition criterion can be viewed as the weighted combination of different uncertainty sampling policies.
In spite of its adaptiveness to the data stream, RAL is highly exploitation-oriented since it mainly focuses on decision boundary learning for each learner.
\citet{baram2004online} first proposed \textsc{COMB} to blend multiple acquisition criteria as experts and consider samples as arms in multi-armed bandits.
Later, \citet{hsu2015active} refined \textsc{COMB} with \textsc{ALBL} using the bandits analogy, treating acquisition criteria as experts. While the bandits framework allows for dynamic adjustment of the exploration-exploitation trade-off, neither method provided guidance on selecting criteria nor addressed the explicit objectives of exploration and exploitation. 
A random selection of general acquisition criteria with a small size may not well address different online annotation scenarios, while a large size of experts may cause problematic performance of the bandits solver.

\section{Methodology}
\label{sec:Method}
To develop the active learning agents for \cbeal and derive the theoretical characterization of the agents, we make the following assumptions:
(i) The sample size of the initial training set $\mathcal{D}_0$ is not large enough  to guarantee satisfactory modeling performance and the samples in $\mathcal{D}_0$ are not uniformly distributed in the input variable space.
(ii) The streaming data have highly imbalanced class distribution.
(iii) There are multiple clusters in the input data distribution.
One common example is that the input data follow a Gaussian mixture distribution.
This assumption is validated by the simulation setup and validated in the case study.
Note that the proposed \cbeal framework is designed for general online annotation scenarios and does not require the assumptions on the input data distribution.

\subsection{Overview of the Proposed Methodology}
Consider the online data annotation scenario with a sample $\bm{x}_t$ collected at time $t, t=1,2,..., T$, where $\bm{x}_t \in \mathbb{R}^p$ is the input for the base learner (i.e., the classification model) $f_t$.
We assume that the classification problem has $c$ classes, and $y_t \in \mathcal{C}=\{1, 2,\dots, c\}$ is the label of the sample $\bm{x}_t$.
Denote the labelled data pool at time $t$ as $\mathcal{D}_t=\{(\bm{x}_1, y_1),...,(\bm{x}_{n_t}, y_{n_t})\}$ with $|\mathcal{D}_t| = n_t$. 
The base learner $f_0$ is pretrained by an initial $\mathcal{D}_0$, which contains a limited number of labelled samples.
Under the aforementioned setting, we propose an active learning strategy to make the acquisition decision \citep{wang2016learning}.
The strategy is applied with (\romannumeral 1) a data source from which one unlabelled sample $\bm{x}_t$ streams at each time stamp without a cost, (\romannumeral 2) a labelled data pool $\mathcal{D}_t$, (\romannumeral 3) human annotators who can provide the label of a sample $\bm{x}_t$ if an acquisition decision is made to acquire it,  (\romannumeral 4) a proposed ensemble acquisition method, and  (\romannumeral 5) the base learner $f_t$ to be updated online with the lableled data set $\mathcal{D}_t$.
We have a budget of $B$ samples for annotation during the streaming process.
\begin{figure}[!htb]
    \centering
    \captionsetup{justification=centering}
    \includegraphics[width=0.7\textwidth]{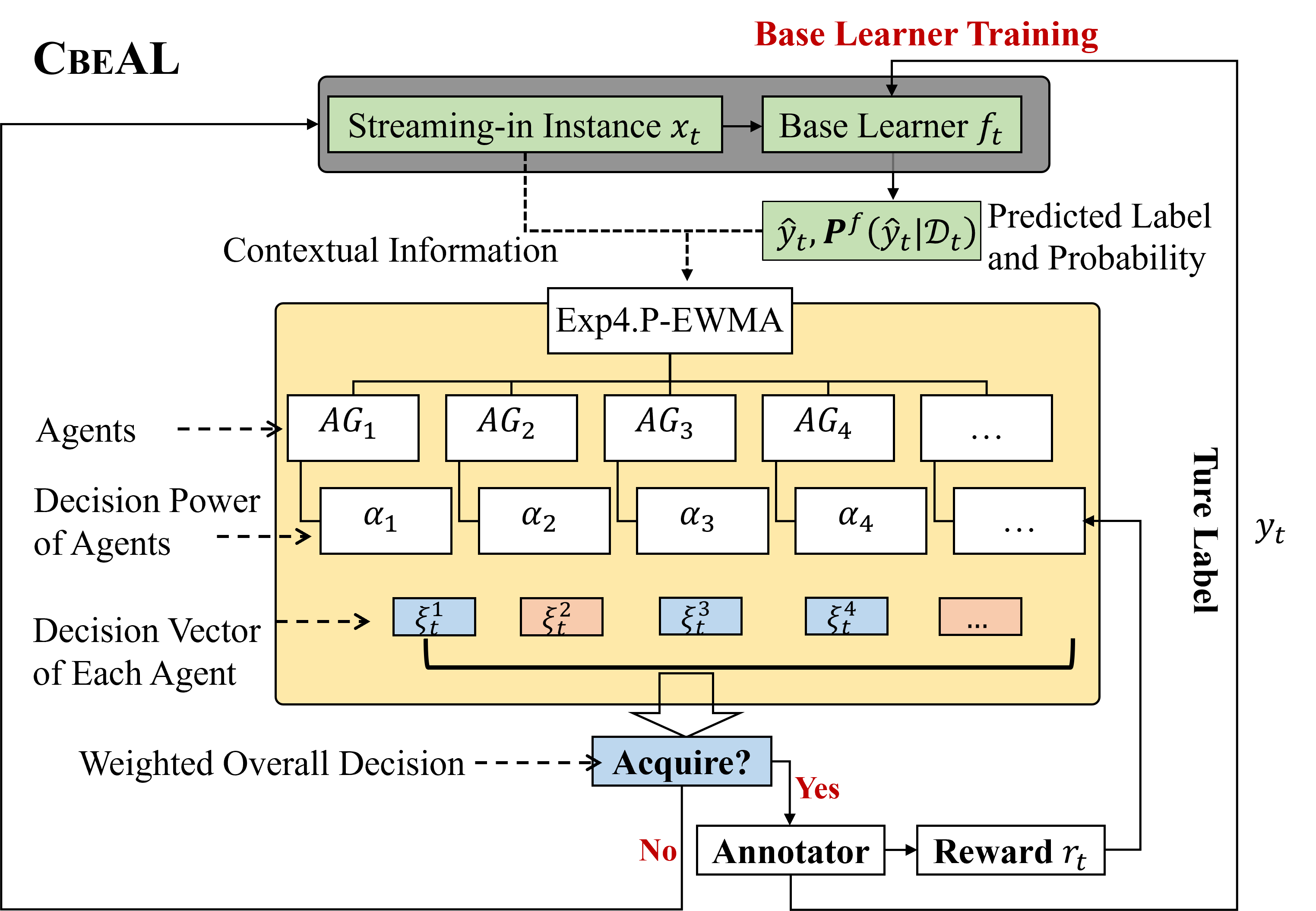}
    \caption{Overview of the proposed \cbeal framework}
    \label{fig:overview}
\end{figure}

As an overview (Fig.~\ref{fig:overview}), our key idea is to ensemble the acquisition decisions made by the exploration- and exploitation-oriented agents and adaptively balance the two aspects based on the context of incoming samples, the learning performance of the base learner, and the historical performance of the agents.
During the online annotation process, at each time point $t$, (i) we receive a sample $\bm{x}_t$. 
(ii) Afterwards, one can obtain the predicted label $\hat{y}_t$ and the side information, such as the predicted probability $\bm{P}^f(\hat{y}_t|\mathcal{D}_t)$ of each class from $f_t$ and take $\bm{x}_t$ and other information as the context input for the proposed contextual bandits solver \textsc{Exp4.P-EWMA}. 
And
(iii) we make the acquisition decision as a weighted majority of the decisions obtained from the set of candidate agents $\{AG1, AG2, ...\}$.
If the decision is to acquire the sample,  we acquire $y_t$ from human annotation and obtain the reward $r_t$, update \cbeal with $r_t$, and retrain the classifier with $(\bm{x}_t, y_t)$; otherwise, we pass this sample without annotation.
The advantage of the proposed framework lies in three aspects:
(i) \cbeal explicitly pursues the input variable space discovery and the decision boundary learning via incorporating exploration- and exploitation-oriented agents while it balances the overall exploration-exploitation trade-off adaptive to the data stream and the learning performance of the base learner by contextual bandits.
(ii) The systematic ensemble of multiple pairs of agents save the efforts for agents selection under various learning scenarios with different input data distribution, feature dimension, signal-to-noise ratio, etc. Therefore, \cbeal is a developed as a generic active learning framework to achieve an adaptive and well-balanced exploration-exploitation trade-off for the incubation of classification models with streaming data.
(iii) \cbeal is scalable in terms of the number of active learning agents to enhance exploration or exploitation.

\subsection{The Ensemble Active Learning by Contextual Bandits}


During the online annotation process, the needs for exploration and exploitation change over time, which depend on the observed samples, the incoming sample, and the updated learning performance of the base learner. 
Since there does not exist a consistent optimization criterion to adjust the trade-off, the shift between the two aspects is nontrivial.

To address the challenge of achieving a good exploration-exploitation trade-off adaptive to various online annotation scenarios,
we formulate the shift between two objectives as a contextual multi-armed bandits problem.
In the bandits problem,
the bandits solver needs to make the decision of pulling one of the $K$ arms as the action for time point $t$ based on the received contextual information.
With each arm characterized by an unknown reward distribution, the objective of the solver is to gain the highest cumulative reward $R = \sum_{t=1}^{\infty}r_t$.
Under the framework of \cbeal, we consider the decision to acquire or not acquire a sample as the arms and propose to ensemble exploration- and exploitation-oriented agents as candidate policies (i.e., experts).
The acquisition decision (i.e., decision of pulling one arm) is jointly made by a weighted combination of the individual decisions from each agent.
Thus, the adjustment of exploration-exploitation trade-off is converted to the selection among different types of agent with the goal of gaining a higher reward, which can be solved by some well-developed bandits solvers.

In \cbeal, the arm pulled at time $t$ (i.e., $a_t \in A = \{1, 2\}, |A|=K=2$) represents the overall acquisition decision, where $a_t=1$ refers to acquiring the sample $\bm{x}_t$ and otherwise $a_t=2$.
At each time point $t$, the incoming sample $\bm{x}_t$, prediction $\hat{y_t}$, and the predicted probability $\bm{P}^f(\hat{y}|\bm{x}_t) \in \mathbb{R}^c$ obtained by the base learner $f_t$ are considered as the observed contextual information that affects the exploration-exploitation trade-off.
Each incorporated active learning agent ($AG_i, i\in\{1, ..., N\}$) makes its own decision $\bm{\xi}^i_t \in \mathbb{R}^K$ based on these contextual information $\{\bm{x}_t, \hat{y_t}, \bm{P}^f(\hat{y}|\bm{x}_t)\}$.
Here ${\xi}^i_{a,t}$ represents the probability of the $i$-th agent taking the $a$-th action.
Specifically,  we have the decision vector $\bm{\xi}^i_t = [p^i_t, 1-p^i_t]$, where $p^i_t$ is the acquisition probability for sample $\bm{x}_t$.
Simultaneously, each agent is assigned with a decision power $\alpha_{i, t}$ at time $t$.
The overall decision is a majority voting of all agents' decisions weighted by their decision power, which leads to a decision vector $\bm{P}_t\in \mathbb{R}^K$.
If the overall decision asks for the ground-truth label, a reward $r_t$ is received after execution.
Afterwards, the proposed bandits solver \textsc{Exp4.P-EWMA} updates the decision power of each agent based on its decision $\bm{\xi}^i_t$ in this iteration and the reward $r_t$.
With the objective of gaining a high cumulative reward, the ensemble of agents is the same as to combine the decisions made by each agent such that the reward gained in each iteration is close to the highest we can get from the best agent in the agent set.
The execution of \cbeal is summarized as Algorithm~\ref{alg:CBEAL}. 
\begin{algorithm}[!htb]
    \textbf{Input}: set of agents $\{AG_1,\dots, AG_N\}$, $\mathcal{D}_0, f_0, B$ \\
    \textbf{Initialize}: $t=0$, $budget\_used=0$\\
    \While{budget\_used $< B$}{
    Receive sample $\bm{x}_t$ and contextual information $\{\bm{x}_t, \hat{y_t}, \bm{P}^f(\hat{y}|\bm{x}_t)\}$\\
    Obtain the decision vector $\bm{\xi}^1_{t}, \dots, \bm{\xi}^N_{t}$ from agents\\
    Execute \textsc{Exp4.P-EWMA solver} for one iteration and obtain the action $a_t$\\
    \uIf{$a_t=1$}{
    Acquire $y_t$, $\mathcal{D}_{t+1} = \mathcal{D}_{t}\cup (\bm{x}_t, y_t)$\\ 
    Train the base learner $f_{t+1}\xleftarrow{}D_{t+1}$\\
    $budget\_used=budget\_used+1$}
    \Else{
    $D_{t+1} = D_{t}, f_{t+1}=f_{t}$
    }
    Update  $\{AG_1,\dots, AG_N\}$\\
    $t = t+1$
    }
    \textbf{Output}: $\mathcal{D}_{t+1}, f_{t+1}$
    \caption{\cbeal}\label{alg:CBEAL}
\end{algorithm}

Notably, the online updating of the base learner is not the focus of this study and we simply retrain the base learner based on all annotated samples from $\mathcal{D}_{t+1}$.
For a more efficient updating, online learning algorithms, such as first-order algorithms \citep{zinkevich2003online} and Bayesian-based approaches \citep{chai2002bayesian}, can be adopted depending on the base learner.  

To integrate the goals of active learning and multi-armed bandits in order to provide informative acquisition, the design and characterization of the reward are critical.
We define the reward $r_t$ as suggested in \citep{wassermann2019ral}:
\begin{equation}
    r_t = \begin{cases}
    \rho^+,& \text{if } \hat{y_t}\neq y_t\\
    \rho^-, & \text{if } \hat{y_t} = y_t.
    \end{cases}\label{eq:CBEAL_reward}
\end{equation}
The reward $r_t$ can only be obtained if the overall decision made by \cbeal acquires the ground-truth label.
Otherwise, it is zero.
Intuitively, the acquisition action will be rewarded if the base learner would have made a wrong prediction, otherwise it will be penalized since this acquisition is considered unnecessary.  
Therefore, it measures both the informativeness and the usefulness of an acquisition decision.
Based on this design, the reward of acquiring a sample is determined by the performance of the current base learner $f_t$, the incoming sample $(\bm{x}_t, y_t)$, and also the performance of the bandits learner \cbeal.
Thus, the sequence of reward $\{r_1, r_2, ..., r_t\}$ is autocorrelated.
This characteristic is another reason that we adopt the setting of adversarial bandits \citep{auer2002nonstochastic}, where
one active learning agent ($AG_i$) is considered as one expert and decisions from each expert are simultaneously considered to make a joint decision, since no statistical assumption is made on reward generation in this setting.
Additionally, instead of considering one agent as one arm, incorporating agents as experts also makes the number of agents scalable. 
In summary, with the designed setting of contextual information and the reward, the updated decision power of each agent adjusts the exploration-exploitation trade-off to improve the learning performance of the base learner.

To solve the formulated contextual bandits problem in \cbeal, we propose the \textsc{Exp4.P-EWMA solver}, where we embed a control chart-based flipping mechanism to \textsc{Exp4.P}\citep{beygelzimer2011contextual}.
To balance the overall exploration-exploitation behaviour, the exploration- and exploitation-oriented agents are incorporated in \cbeal by pairs, which can be easily adjusted for a specific online annotation scenario.
However, with the pair ensemble, the direct application of \textsc{Exp4.P} can easily lead to a dominant agent (i.e., an agent consistently has the highest decision power) from an early stage, which makes \cbeal act no difference from a single active learning agent dedicated to exploration or exploitation.
This can be expected since the pure exploration strategy in the early stage may cause the acquisition of samples with low uncertainty, so that the decision power of exploration-oriented agents keeps decreasing until a level too low to contribute to the overall acquisition decision any longer.
To avoid the early convergence
in \textsc{Exp4.P}, a control chart-based flipping mechanism is integrated into the solver.
Denote the standardized weight (i.e., standardized decision power) of the $i$-th agent at time $t$ as $\alpha^s_{i, t}$, where $\alpha^s_{i, t} = \frac{\alpha_{i,t}}{\sum_{i=1}^N\alpha_{i,t}}$.
We monitor each standardized weight by an EWMA chart~\citep{hunter1986exponentially}, which detects weight drift over time.
The intuition is that if the decision power of one agent keeps decreasing or increasing from the beginning, the decision power of all pairs of agents will be flipped so that the agents with lower power have more chances to lead the decision in the following period.
Note this forced-exploration phase will only happen in a short period during the whole process, which is controlled by a hyperparameter $\gamma$.
 Denote the weighting factor for EWMA as $\lambda$, the size factor of shift to detect as $h$, and the estimated variance of $\alpha^s_{i, t}$ as $s^2_{i, t}$.
With the flipping mechanism, the proposed \textsc{Exp4.P-EWMA solver} is detailed as follows:
\begin{algorithm}[!htb]
    Parameters: $\delta, \gamma, h,  p_{\text{min}}\in[0, 1/K]$\\
    Initialization: Set $\alpha_{i,1}=1 ,\alpha^{EWMA}_{i, 1}=\frac{1}{N} $ for $ i=1, \dots, N, \mu=\frac{1}{N}$.\\
    \For{$t=1,2,\dots$}{
        \textbf{Input}: decision vector of each agent $\bm{\xi}^1_{t}, \dots, \bm{\xi}^N_{t}$\\
        \SetKwBlock{Fna}{\textnormal{\textbf{Step 1}: \textsc{Exp4.P}}}{}
        \Fna{
            For $a=1,\dots, K$ get the final decision probability $\bm{P}_{t}$:  $P_{t, a}=(1-Kp_{\text{min}})\sum_{i=1}^N\frac{\alpha_{i,t}\xi_{a,t}^i}{\sum_{i=1}^N\alpha_{i,t}}+p_{\text{min}}$\\
            Draw the action $a_t$ based on $\bm{P}_{t}$ and receive the reward $r_t$\\
            for $a=1,\dots, K$ set $\hat{q}_{a,t}=
            \begin{cases}
            r_t/P_{t, a},& \text{if } a=a_t\\
            0, &\text{otherwise}
            \end{cases}, \hat{\bm{q}}_t = [\hat{q}_{1,t}, ..., \hat{q}_{K,t}]\in\mathbb{R}^K$\\
            for $i=1,\dots, N$ set $\hat{g}_{i,t}=\bm{\xi}^i_{t}\cdot \hat{\bm{q}}_t^T$, $\hat{v}_{i,t}=\sum_{a}\xi^i_a/P_{t, a}$\\
            Update the decision power: $\alpha_{i, t+1} = \alpha_{i, t}\cdot \exp{(\frac{p_{\text{min}}}{2}(\hat{g}_{i,t}+\hat{v}_{i,t}\sqrt{\frac{\ln{N}}{KT}}))}$
            }
        \SetKwBlock{Fna}{\textnormal{\textbf{Step 2}: EWMA-based flipping mechanism}}{}
        \Fna{
            \For{$i=1,\dots,N$}{
                $\alpha^s_{i, t+1} = \frac{\alpha_{i,t+1}}{\sum_{i=1}^N\alpha_{i,t+1}}$\\
                $\alpha^{ewma}_{i, t+1} = \lambda \cdot \alpha^s_{i, t+1} + (1-\lambda)\alpha^{ewma}_{i, t}$\\
                $ LCL = \mu - h \cdot \frac{\lambda}{2-\lambda} \cdot s^2_{i,t} $,
                $ UCL = \mu + h \cdot \frac{\lambda}{2-\lambda} \cdot s^2_{i,t} $\\
                Set $\alpha^s_{i, t+1} = 
                \begin{cases}
                2\mu-\alpha^s_{i, t+1}, & \alpha^{ewma}_{i, t+1} > UCL \text{ or }  \alpha^{ewma}_{i, t+1} < LCL\\
                \alpha^s_{i, t+1}, & \text{otherwise}
                \end{cases}$,
                $\alpha_{i, t+1} = (\sum_{i=1}^N\alpha_{i,t+1})\cdot \alpha^s_{i, t+1}$
            }
            $h: = h\cdot \exp{\gamma}$
        }
        \textbf{Output}: $a_t$
    }
\caption{\textsc{Exp=4.P-EWMA solver}}
\label{alg:ModExp4P}
\end{algorithm}

As listed in Algorithm~\ref{alg:ModExp4P}, at each time point $t$, the solver will first execute \textsc{Exp4.P} to make the acquisition decision $a_t$ and update the decision power of each agent $\alpha_{i, t+1}$ based on its decision vector $\bm{\xi}^i_{t}$, the final decision probability $\bm{P}_t$, and the reward $r_t$.
In the second step, the flipping will be triggered if the standardized weight of any agent is outside the updated control limits.

\subsection{The Exploration and Exploitation Agents}
To balance the exploration-exploitation trade-off under different online annotation scenarios, we design distinguished active learning agents with exploration or exploitation objectives to be incorporated into \cbeal so that a systematic approach is developed without ambiguous selection.
Another advantage of this design is the tendency for exploration or exploitation can be directly implied by the decision power of different types of agent.
\subsubsection{Low-density Based Exploration Agent (LD-Agent)}
The objective of exploration is to identify the structure of the input data distribution during the learning process.
Two types of agents are proposed to encourage the exploration of the input variable space. 
The first type adopts a density-based criterion, which encourages the labelling efforts around
each cluster boundary to discover new clusters by annotating
samples lying in a sparse region with low density.  
We adopt the idea in \citep{ienco2014high}  to model the density of a sample.

Denote the set $\mathcal{W}$ as a sliding window of $L$ previously observed samples, $d(\cdot, \cdot)$ as the distance between two samples. 
Denote $MaxDist$ as a function, where $MaxDist(\bm{x}_i, \mathcal{W})$ returns the maximum distance between $\bm{x}_i$ and other samples in the sliding window $\mathcal{W}$.
To approximate the local density for a new coming sample, we define local sparsity (i.e., low-density) factor of a sample $\bm{x_i}$ as the number of times $\bm{x}_i$ is the farthest away from other samples in $\mathcal{W}$ as follows:
\begin{equation}
    lsf(\bm{x}_i) = \sum_{\bm{x}_j \in \mathcal{W}}\mathbb{I}\{MaxDist(\bm{x}_j,\mathcal{W} )<d(\bm{x}_i, \bm{x}_j)\}. \label{eq:ldf}
\end{equation}

Algorithm~\ref{alg:LDAL} provides the pseudocode to acquire samples with lower local density.
Given a streaming sample $\bm{x}_t$ at time $t$, low-density based exploration agent first calculates the local sparsity factor $lsf(\bm{x}_t)$ to determine the acquisition probability $p_t$ as the output. 
Then, the sliding window $\mathcal{W}$ and the maximum pairwise distance between each sample in $\mathcal{W}$ will be updated.
The sliding window mechanism is adopted to adjust the approximated density based on the most recent data stream.
Note the window length $L$, and the sparsity fraction $\delta_L$ are hyperparameters that affect the acquisition probability, which can be tuned to best suit the scenario.
\begin{algorithm}[!ht]
\caption{Low-density Based Exploration Agent (LD-Agent)}\label{alg:LDAL}
   \textbf{Input}: $\bm{x}_t, \mathcal{W}, L, \mathcal{D}_0, \delta_L$\\
   \text{Calculate} $lsf(\bm{x}_t)$\\
   \For{$j=1,2,\dots, L$}{
   \uIf{$d(\bm{x}_i, \bm{x}_j)> MaxDist(\bm{x}_j, \mathcal{W})$}{$MaxDist(\bm{x}_j, \mathcal{W})=d(\bm{x}_i, \bm{x}_j)$}}
    \textbf{Output}\text{: Acquisition probability } $p_t=\frac{lsf(\bm{x}_t)}{L\cdot \delta_L}$\\
    \uIf{$|\mathcal{W}|>L$}{$\mathcal{W}:=\mathcal{W}\setminus \bm{x}_{t-L}$}
    $\mathcal{W}: =\mathcal{W} \cup \bm{x}_t$
\end{algorithm}

\subsubsection{Space-filling Based Exploration Agent (SPF-Agent)}
The second type of exploration-oriented agents is based on a space-filling criterion.
In the DoE literature, space-filling designs are applied to fully explore the response surface of computer experiments \citep{shang2021fully}. 
Therefore, as an alternative strategy to explore the input variable space, a space-filling based exploration-oriented agent is developed to acquire samples uniformly distributed in the space.
We adopt the idea of minimum pairwise distance criterion \citep{kennard1969computer} and propose a corresponding criterion to minimize the pairwise distance between acquired samples during the online data acquisition.

Similarly, a sliding window $\mathcal{W}$ keeps the most recent $L$ samples.
Denote $MinDist$ as a function where $MinDist(\bm{x}_i, \mathcal{W})$ returns the minimum distance between $\bm{x}_i$ and all samples in $\mathcal{W}$. 

\begin{algorithm}[!htb]
\caption{Space-filling Based Exploration Agent (SPF-Agent)}\label{alg:SFAL}
    \textbf{Input}: $\bm{x}_t, \mathcal{W},L, \mathcal{D}_0$\\
    \For{i=1,2,\dots,L}{min$(d_i)=\min {d(\bm{x}_i, \bm{x}_j), \forall j \in \mathcal{W}}$}
    Calculate $MinDist(\bm{x}_t, \mathcal{W})$\\
    \textbf{Output}: Acquisition probability $p_t =\frac{MinDist(\bm{x}_t, \mathcal{W})}{\max_{i\in \mathcal{W}}{\text{min}(d_i)}}$\\
    \uIf{$|\mathcal{W}|>L$}{$\mathcal{W}:=\mathcal{W}\setminus \bm{x}_{t-L}$}
    $\mathcal{W}: = \mathcal{W} \cup \bm{x}_t$
\end{algorithm}

In Algorithm~\ref{alg:SFAL}, with a coming sample $\bm{x}_t$ at time $t$, its minimum distance from the samples in $\mathcal{W}$ is compared with the largest minimum pairwise distance of samples in $\mathcal{W}$ to obtain the acquisition probability, leaving a higher probability for samples distant from the observed ones in $\mathcal{W}$.

Intuitively, the density-based criterion will explore the boundary of the input variable space faster at an early stage, whereas the space-filling criterion allows for a more uniform exploration during the process.
The combination of two exploration criteria will enhance the compatibility and adaptiveness of \cbeal to various learning scenarios.
In practical, an $\epsilon$-greedy policy can also be embedded which forces a sample to be acquired with probability $\epsilon$ for further exploration.

\subsubsection{Reinforced Exploitation Agent (RAL-Agent)}
The goal of exploitation in active learning is to capture the decision boundary, which is generally achieved by acquiring samples with ambiguous class membership.
To enable the agent to intelligently identify the acquisition demand, we formulate it as a RL problem that aims at learning an adaptive threshold as the optimal policy to maximize the cumulative reward.
As suggested by~\cite{wassermann2019ral}, a RL-based controller is designed to adjust the certainty threshold $\theta$ based on the contribution of historical acquisition decisions.
In detail, upon receiving a sample $\bm{x}_t$ at time $t$, the prediction certainty $ct(\bm{x}_t)=\max{\bm{P}^f(\hat{y}|\bm{x}_t)}$ obtained by the base learner $f_t$ is compared with the current certainty threshold $\theta_t$ to make the acquisition decision.
The reward $r_t$ will be received if $\bm{x}_t$ is acquired, which
follows a consistent definition (i.e., $r_t\in\{0, \rho^+, \rho^-\}$) as defined in \eqref{eq:CBEAL_reward}.
Afterwards, the certainty threshold will be updated as:
\begin{equation}
    \theta_{t+1} = \min\left\{\theta_t (1+\eta\cdot(1-2^{\frac{r_t}{\rho^{-}}})), 1\right\}.
\end{equation}
Note that the threshold will increase slightly with a positive reward and vice versa, which enables a policy adaptive to the decision boundary learned by the base learner.
The algorithm of the reinforced exploitation agent is detailed in Algorithm~\ref{alg:RAL}.
\begin{algorithm}
\caption{Reinforced Exploitation Agent (RAL-Agent)}\label{alg:RAL}
    \textbf{Input}: $\bm{x}_t, \theta_0, \eta, \rho^+, \rho^-$\\
    \uIf{$ct(\bm{x}_t)<\theta_{t}$}{$p_t$ = 1, 
    obtain the reward $r_t$\\
    Update the certainty threshold $\theta_{t+1} = \min\left\{\theta_t (1+\eta\cdot(1-2^{\frac{r_t}{\rho^-}})), 1\right\}$}
    \uElse{$p_t$ = 0}
    \textbf{Output}: Acquisition probability $p_t$
\end{algorithm}

\subsection{Characterization of Agents}

To characterize the exploration and exploitation capability of the proposed agents,
the variance of the acquired samples $\mathcal{D}_t$ by one agent is selected as an appropriate metric for assessing its exploration and exploitation activity.
A higher variance suggests a learner's ability to explore the input variable space via acquiring samples in a larger region, whereas a lower variance implies a high frequency of acquisition in a small region for exploitation.
To compare the variance of $\mathcal{D}_t$, the probability of a single sample being acquired by the proposed agents is examined.

We assume that the streaming data belong to a mixture of Gaussian distributions.
Denote the previously observed samples stored in the sliding window $\mathcal{W}$ before time $t$ as the set $\{\bm{x}_1, \bm{x}_2, ..., \penalty 0 \bm{x}_L\}$, where $\bm{x}_i$ belongs to the $i$-th Gaussian distribution (i.e., $\bm{x}_i\sim \mathcal{N}_q(\bm{\mu}_i, \bm{\Sigma}^{(i)})$).  
Given a streaming sample $\bm{x}_t$ at time $t$ which follows another Gaussian distribution (i.e., $\bm{x}_t\sim \mathcal{N}_q(\bm{\mu}_k, \bm{\Sigma}^{(k)})$), with the Euclidean distance $d(\bm{x}_i, \bm{x}_j)=\sqrt{\norm{\bm{x}_i-\bm{x}_j}_2^2}$,
For probability that LD-Agent acquires $\bm{x}_t$, we proved:
\begin{theorem}
If the streaming samples follow an independent multivariate Gaussian distribution (i.e., $\bm{\Sigma^{(i)}}=\sigma_i^2\bm{I}$), then there exist $\bm{M}_1, \bm{M}_2 \in \mathbb{R}^L$ such that if $\norm{\bm{\mu}_i-\bm{\mu}_k}^2 > M_{2,i}, \forall i \in \{1,...,L\}$, then the expected acquisition probability of a LD-Agent $
\mathbb{E}_{\bm{x}_i, \bm{x}_j, \bm{x}_k}[p_t]$ will exceed $1$, where $\bm{M}_{1}, \bm{M}_{2}$ satisfies:
\begin{align}
    &M_{2,i} + \erf^{-1}{(1-2\delta_L)}\cdot \sqrt{2\cdot(4(\sigma_i^2+\sigma_k^2)M_{2,i}+2q(\sigma_i^2+\sigma_k^2)^2)} +
    (\sigma_i^2+\sigma_k^2)q - M_{1,i} = 0\nonumber\\
    & M_{1,i}>\norm{\bm{\mu}_i-\bm{\mu}_j}^2 + (\sigma_i^2+\sigma_j^2)q + \sqrt{4(\sigma_i^2+\sigma_j^2)\norm{\bm{\mu}_i-\bm{\mu}_j}^2+2q(\sigma_i^2+\sigma_j^2)^2} \nonumber\\
    &\cdot\bigg(\Phi^{-1}(1-\frac{1}{L-1})+\gamma\bigg[\Phi^{-1}(1-\frac{1}{L-1}\cdot e^{-1})-\Phi^{-1}(1-\frac{1}{L-1}) \bigg]\bigg),\forall i \in \{1,...,L\}.
\end{align}
\end{theorem}
This result illustrates that with the increasing of the distance between the center of the distribution of the observed samples and that of the incoming sample, the expected acquisition probability approaches and exceeds 1.
This ensures the acquisition of samples from a remote cluster, resulting in an increased variance and, thus, the exploration of the input variable space.

For the reinforced exploitation agent, assume that a logistic regression model is selected as the base learner and at time $t$  the base learner $f_t$ is parameterized by $\bm{\beta}_t$.
Given the labeled data pool $\mathcal{D}_t$ at time $t$, for the expectation of the probability that the RAL-Agent acquires $\bm{x}_t$, we proved:
\begin{theorem}
Given the labeled data pool $\mathcal{D}_t$ at time $t$, assume the center of the labeled samples in $\mathcal{D}_t$ is $\bm{\mu}_i\in\mathbb{R}^q$ and the incoming sample $\bm{x}_t\sim \mathcal{N}_q(\bm{\mu}_k, \sigma_k^2\bm{I})$.
With the increase of the distance between two centers $\norm{\bm{\mu}_i-\bm{\mu}_k}^2$, there does not exist $M_3 \in \mathbb{R}$ such that 
$P\{|\mathbb{E}_{\bm{x}_t}[p_{t}] - M_3|\geq\epsilon\}= 0, \forall \epsilon \in \mathbb{R}$.
\end{theorem}
Since $p_t$ belongs to $[0, 1]$ for a RAL-Agent, the  result implies that the acquisition probability of the incoming $\bm{x}_t$ will not converge with the increase of the distance between the center of the distribution of $\mathcal{D}_t$ and that of $\bm{x}_t$.
Hence, for one sample from a remote cluster, the acquisition decision made by a RAL-Agent does not necessarily lead to an increasing variance.

In summary, the theoretical analysis justifies the exploration and exploitation capability of the proposed agents.
Therefore, with the ensemble of two types of agents, the trade-off can be dynamically adjusted to the human annotation process.
We also include the theoretical justification on the EWMA mechanism where we prove it does not affect the regret bound of the \textsc{Exp4.p} solver.
The proof and numerical study can be found in the supplemental material due to the page limit.

\section{Numerical Simulation}\label{sec:sim}

\subsection{Simulation Setup}
Suppose that we have a binary classifier as the base learner that requires online updating.
Recall the third assumption that multiple clusters exist in the input variable space.
Therefore, we adopt a cluster-based classification data set generation method \citep{scikit-learn,Guyon2003DesignOE} to generate the input $\bm{X}\in\mathbb{R}^{n\times p}$ and the corresponding label $\bm{y}\in \mathbb{R}^n$, where $n$ is the sample size and $p$ is the dimension of the input variable.  
We assume that there are two clusters in each class, thus we have $2\times2=4$ clusters in total.
In brief, the centroids of 4 Gaussian clusters are first generated as the vertices of one polytope.
The input variables are then independently drawn from each Gaussian cluster with unit variance and then multiplied by a random matrix to introduce the random covariance.
Then, the samples in two of the four clusters will be assigned with the same label as $\bm{y}$.

To evaluate \cbeal comprehensively, four settings are varied to generate different online annotation scenarios: (i) training sample size $n$, which includes both the initial training set and the streaming training set; 
(ii) the percentage of samples in the positive class $pc$, which determines the balanceness of the two classes;
(\RN{3}) the percentage of disturbance $ds$; and 
(\RN{4}) the percentage of sparsity $sp$, which is defined as the percentage of insignificant input variables among total $p$ input variables.
Note that disturbances are added by flipping the labels of randomly selected samples. 
Additionally, to control the sparsity level, insignificant variables are randomly generated and concatenated to informative ones.

The data set generated for each online annotation scenario is subdivided into three subsets: the initial training set, the streaming training set, and the testing set.
The initial training set has a constant size of 20 and the testing set has a size of 500.
The number of samples in each class is balanced to be equal in the testing set to better illustrate the classification performance of the base learner.
For all simulation scenarios, the budget is set to be $10\%$ which gives the number of samples available to be labeled as $B = 10\%\cdot (n-20)$ during the streaming process.
All scenarios are replicated 10 times with a randomly generated data set in each replication.

Based on the suggestion in \citep{ienco2014high,wassermann2019ral} and grid search in simulation experiments, we set the following values for the hyperparameters in \cbeal: $p_{\text{min}}=\sqrt{\frac{\ln{N}}{KT}}, T = 2000, \delta=0.1, \lambda = 0.3, h=5, \gamma=t/T$, reward $\rho^+=1$, penalty $\rho^-=0.5$.
Meanwhile, three pairs of agents with recommended hyperparameter values are incorporated into \cbeal, forming the set of six agents in Table~\ref{tab:agents}.
Note that the hyperparmaters can be further tuned for different learning scenarios.
\begin{table}[!htb]
\centering
 \caption{Agent set adopted in \cbeal}\label{tab:agents}
\begin{tabular}{cccc}
\toprule
Pair Index & Agent Index & Agent & Hyperparameters \\\hline
\multirow{2}{*}{1} & $AG_1$ & $LD_1$ & $L=100, \delta_L = 0.01$\\
& $AG_2$ & $RAL_1$ & $\theta_0=0.95, \eta=0.005$\\
\multirow{2}{*}{2} & $AG_3$ & $LD_2$ & $L=150, \delta_L = 0.005$\\
& $AG_4$ & $RAL_2$ & $\theta_0=0.95, \eta=0.01$\\
\multirow{2}{*}{3} & $AG_5$ & $SPF_1$ & $L=60$\\
& $AG_6$ & $RAL_3$ & $\theta_0=0.90, \eta=0.01$\\\bottomrule
\end{tabular}
\end{table}



\begin{figure}[!htb]
    \centering
    \captionsetup{justification=centering}
    \subfloat[LD-Agent with final testing accuracy = 0.838, 36 acquired samples]{\includegraphics[width=0.48\textwidth]{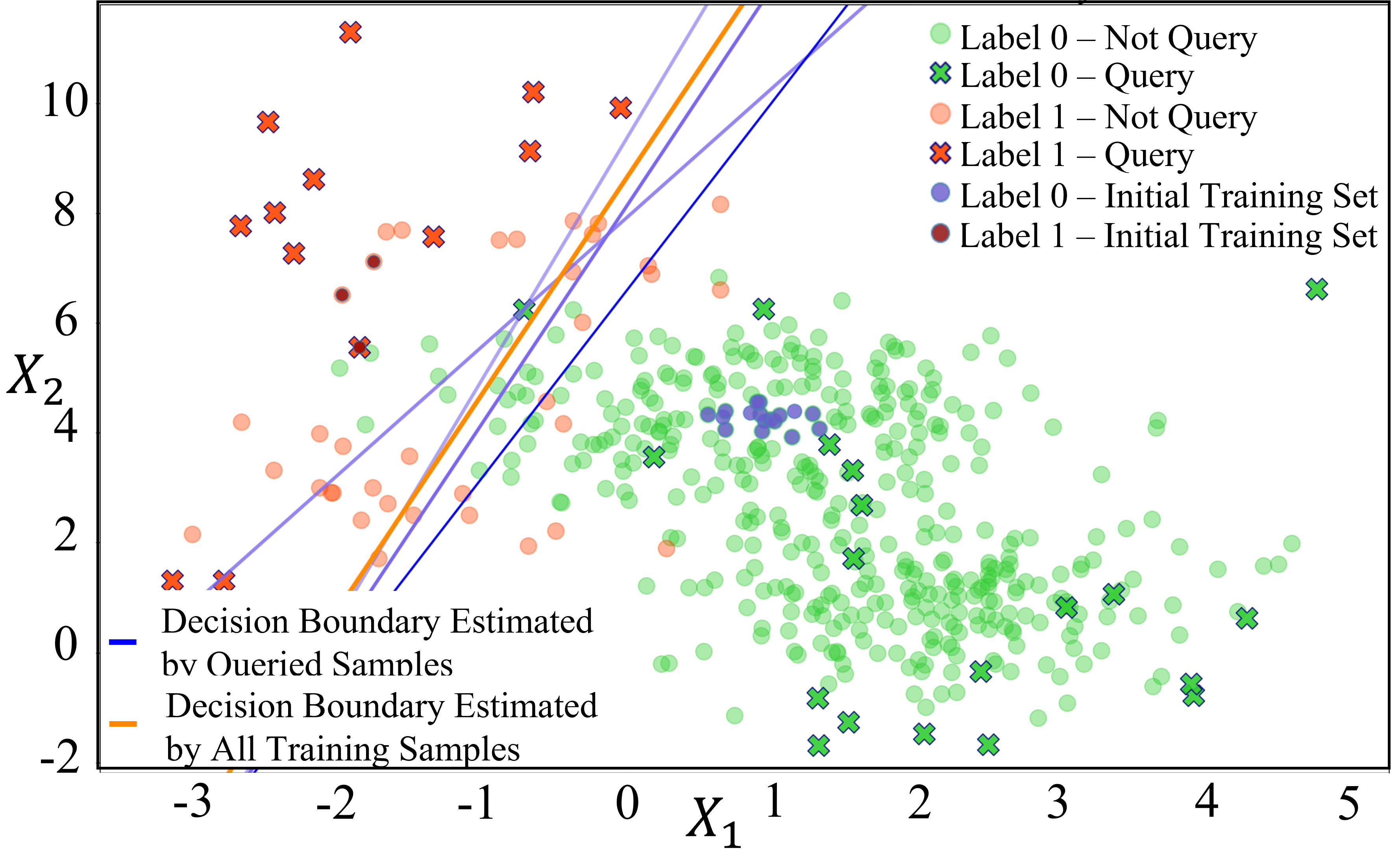}} \hfill 
    \subfloat[SPF-Agent with final testing accuracy = 0.862, 48 acquired samples ]{\includegraphics[width=0.48\textwidth]{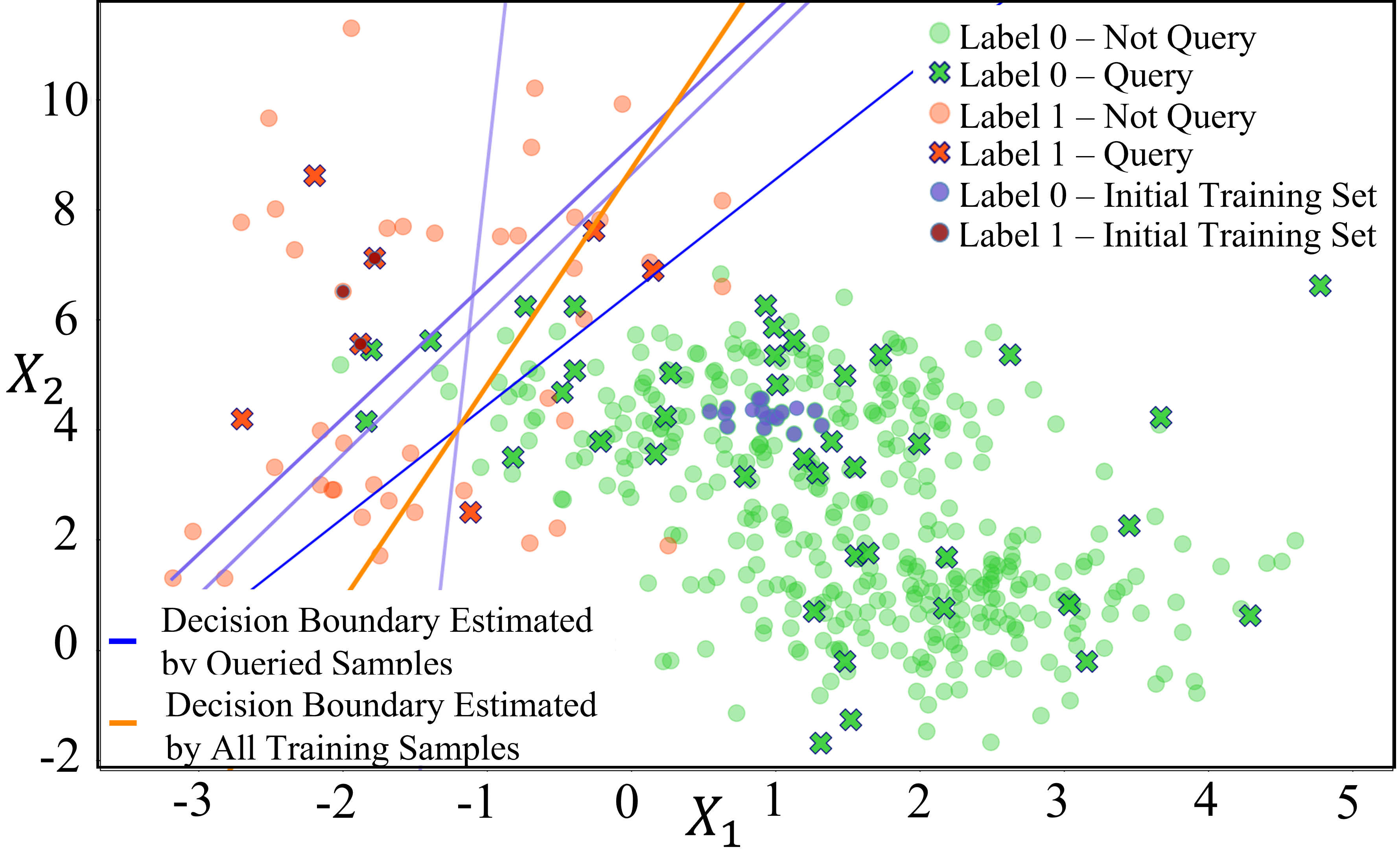}} \\
    \subfloat[RAL-Agent with final testing accuracy = 0.786, 31 acquired samples]{\includegraphics[width=0.48\textwidth]{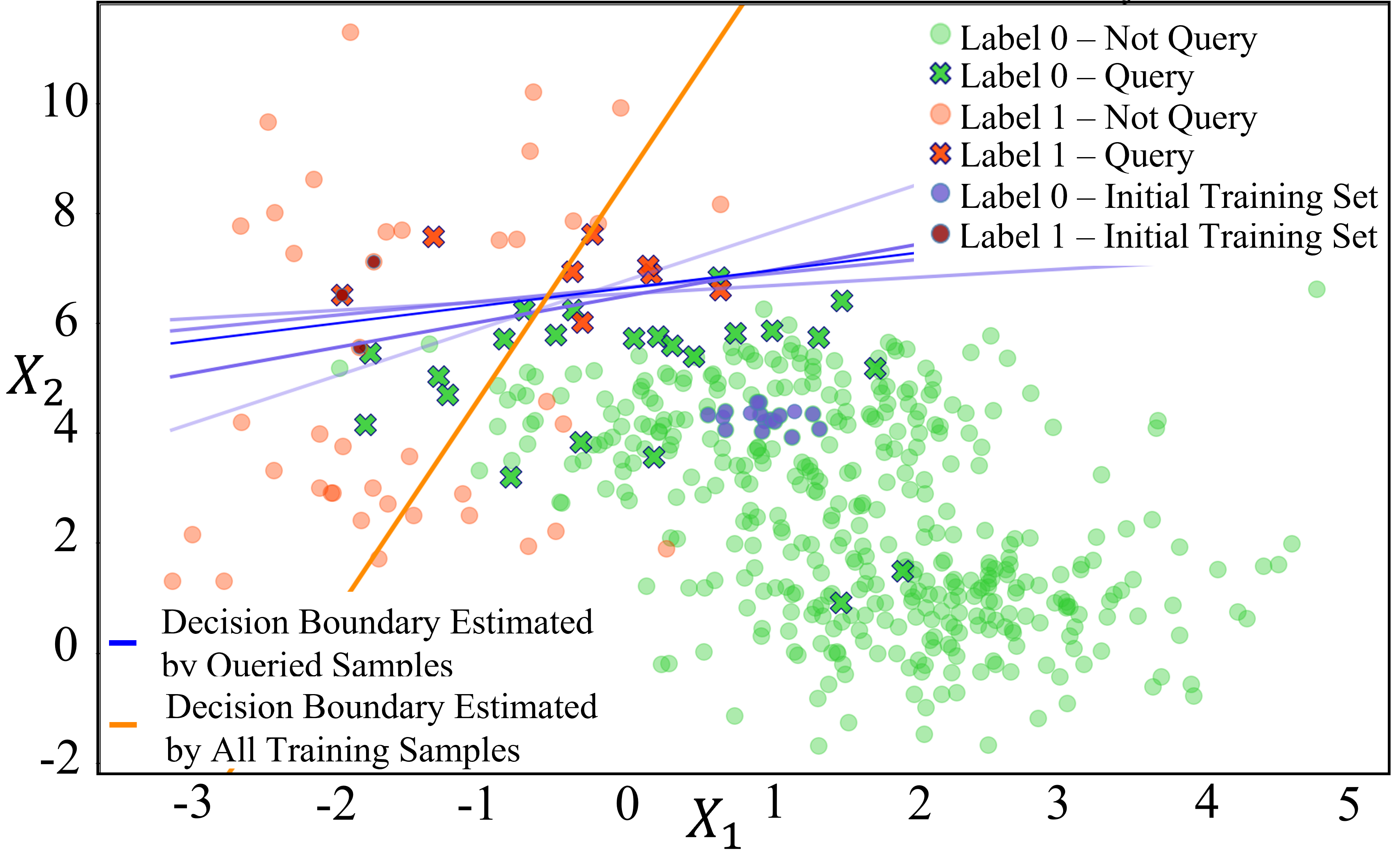}} \hfill
    \subfloat[CBEAL-2 with final testing accuracy = 0.886, 46 acquired samples]{\includegraphics[width=0.48\textwidth]{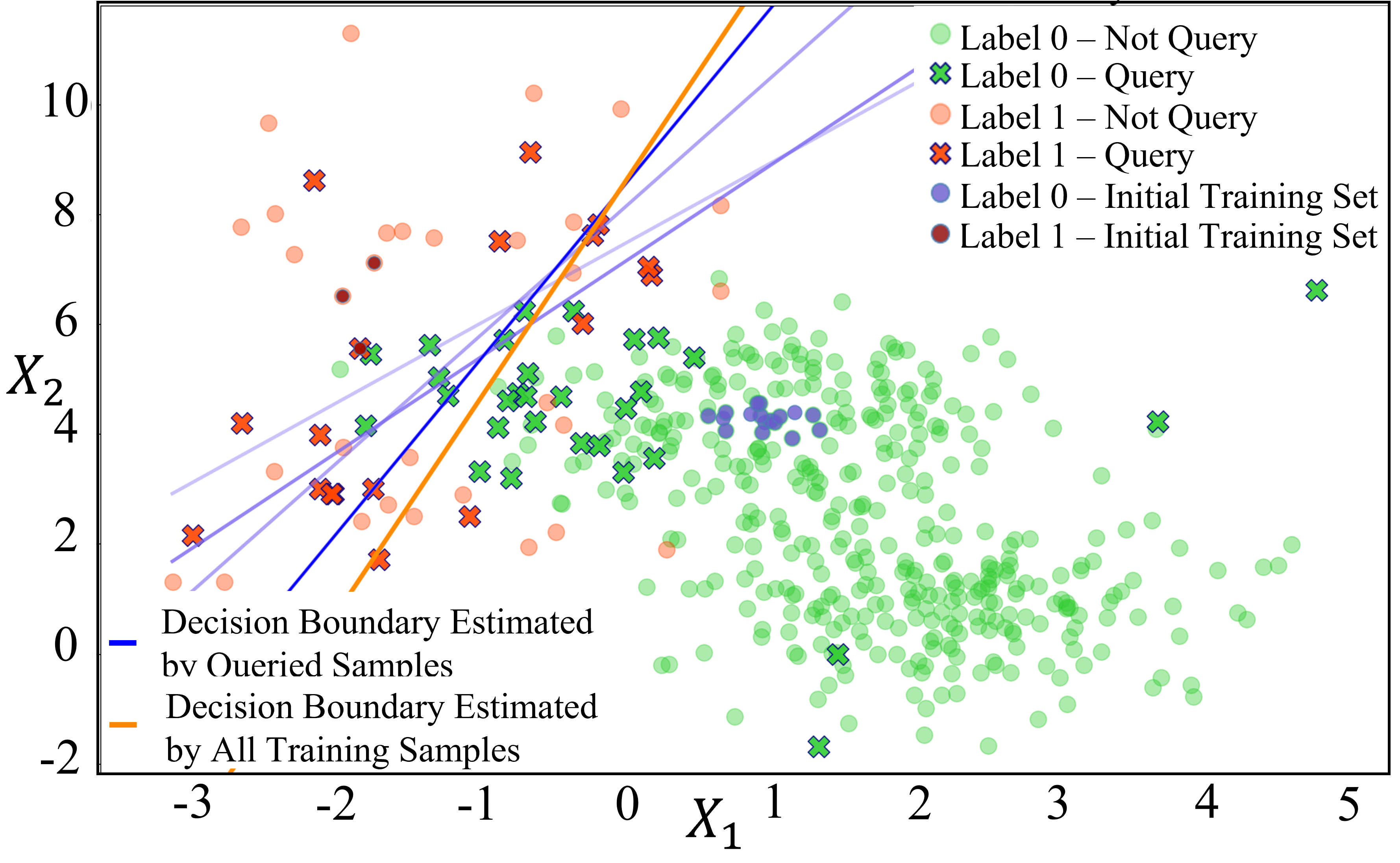}}
    \caption{Evolution of the base learner's decision boundary in the toy example: The set of blue lines represents the decision boundaries learned by the base learner every 100 time points where the color depth of the line is proportional to time $t$; The orange line is the ground-truth decision boundary.}
    \label{fig:simu_toy}
\end{figure}
To demonstrate the simulation setup, Fig.~\ref{fig:simu_toy} visualizes the generated imbalanced and clusterwise input data of a toy example with 2 input variables (i.e., $p=2, n=500, pc=10\%, sp=0\%, ds=0\%, B=48$).
The logistic regression model with default hyperparameters is selected as the base learner \citep{kleinbaum2002logistic, scikit-learn}.
The set of blue lines in Fig.~\ref{fig:simu_toy} traces the base learner's evolving decision boundary with samples acquired by the candidate agents and \cbeal. 
The line's color depth indicates time, with the darkest being the final boundary tested for accuracy. 
The decision boundary estimated with all training data is shown as orange lines.
The results of the agents with the best performance in the agent set are presented (i.e., Table~\ref{tab:agents}).
\cbeal ensembles one exploration agent and one exploitation agent with the best performance in pairs, marked as \cbeal-2.

It is clearly shown that the LD-Agent actively seeks samples around the boundary of the input variable space, whereas the samples acquired by SPF-Agent are more evenly distributed.
Both exploration-oriented agents successfully acquire the samples in both clusters of each class, but very limited samples around the ground-truth decision boundary are selected.
In regard to the RAL-Agent, although the uncertainty threshold is supposed to be updated adaptively, Fig.~\ref{fig:simu_toy} reveals that it is stuck in one cluster of the positive class while the agent keeps acquiring around the wrongly estimated decision boundary due to the lack of explicit exploration.
Combining the two strategies together with the proposed ensemble mechanism, \cbeal-2 acquires samples from both clusters in each class with a focus around the ground-truth decision boundary. 
This also validates that a well-balanced dynamic trade-off between exploration and exploitation is key to the active learning process with imbalanced and clusterwise input data distribution.


\subsection{A Comprehensive Simulation Study}
In the comprehensive simulation study, the settings are varied with the following levels: $n\in \{500, 1000, 1500\}$; $pc\in\{10\%, 5\%\}$; $ds\in\{0\%, 3\%\}$; $sp\in\{30\%, 70\%\}$.
The dimension of the input variable is set as $p=15$.
SVM is selected as the base learner with default parameters \citep{gunn1998support, scikit-learn}, which validates the effectiveness of \cbeal as a generic framework for classification models.
Note that the prediction probability $\bm{P}^f(\hat{y}_t|x_t)$ of SVM is estimated and calibrated by Platt scaling \citep{platt1999probabilistic}.
Denote \cbeal-2 as the ensemble of the first pair of agents in Table~\ref{tab:agents} (i.e., $AG_1$ and $AG_2$), \cbeal-4 as the ensemble of the first two pairs, and so on for \cbeal -6.
Specifically, \cbeal-6 is proposed as the recommended configuration due to its superior performance enhanced by the ensemble of multiple distinguished agents, which will be detailed in the scalability study.

In this study, \cbeal is firstly compared with the incorparated agents and their variants to study the effectiveness of the ensemble mechanism.
Then, we investigate the impact of the  number of agents in the ensemble to study the scalability of \cbeal.
Finally, as the recommended configuration,  \cbeal-6 is compared with four benchmark methods from literature (i.e., uncertainty sampling~\citep{lewis1994sequential} and random sampling, \textsc{DBALStream} \citep{ienco2014high} and \textsc{QBC-PYP} \citep{loy2012stream}) to test the general performance. 

In summary, nine benchmark methods are compared with \cbeal-6 where the first three are the candidate agents (i.e., \textsc{LD}, \textsc{SPF} and \textsc{RAL}-Agent), the middle two are the ensemble models with different number of agents (i.e., \cbeal-2 and \cbeal-4), and the last four are methods from the literature.
Among the benchmarks, \textsc{RAL}-Agent employs an acquisition criterion learned by multi-armed bandits as a cutting-edge AI-guided active learning method; \cbeal-2 and \cbeal-4 adopt the proposed ensemble framework; \textsc{LD}-Agent, \textsc{SPF}-Agent and Random Sampling (RS) focus on the exploration of the input variable space while Uncertainty Sampling (US)~\citep{lewis1994sequential} caters to exploitation; \textsc{DBALStream}~\citep{ienco2014high} and \textsc{QBC-PYP} \citep{loy2012stream} are two state-of-the-art composite active learning methods which integrate the objective of exploration and exploitation in their design of acquisition criteria. 
In detail, in RS, each streaming-in sample is annotated based on a probability derived from the
ratio of the budget $b$ to the training sample size $n$.
QBC-PYP and \textsc{DBALStream} are configured with default parameters. 
US employs an uncertainty threshold of 0.7, determined through cross-validation.

To demonstrate the effectiveness of \cbeal in achieving high learning performance of the base learner with limited budgets, the classification accuracy of the base learner  trained by the samples acquired by each method on the testing set is evaluated as a metric.
We further investigate the percentage of positive samples acquired during the learning process as another metric to illustrate the exploration-exploitation trade-off in the supplemental material.



\subsubsection{Compared with Individual Agents}

First, the comparison of the classification accuracy of the base learners between \cbeal, the incorporated agents, and their variants is shown in Table~\ref{tab:simu_acc_agents}.
The performance of the agents that achieve the  highest accuracy on average among the exploration- and exploitation-oriented agents in the agent set is selected to be reported as "Opt. Explor." and "Opt. Exploit." 
The results of other individual agents are omitted here for better readability. 
It is observed that in the toy example, some of the agents do not use up the budget $B$.
To validate that it is a fair comparison with different numbers of acquired samples, we create "Opt. Explor. (Full)" and "Opt. Exploit. (Full)" as two variants where random sampling is used to artificially acquire from the unselected samples after the agents finish their acquisition of the streaming data, until the budget is used up.
Besides, the $\epsilon$-greedy policy with $\epsilon=0.01$ is applied to RAL-Agents to effectively improve their learning performance to be a more competitive benchmark \citep{wassermann2019ral}.
They are also applied to \cbeal methods for a fair comparison.

\begin{table}[!htb]
\caption{The average values and standard errors (in parenthesis) of classification accuracy in the simulation study reported over 10 replications. Best results (excluding \cbeal-2 and \cbeal-4) are highlighted in \textbf{bold}.}\label{tab:simu_acc_agents}
\resizebox{\textwidth}{!}{
\begin{tabular}{ccccccccc}
\toprule
\multicolumn{2}{c}{Level} & \multirow{3}{*}{Method} & \multicolumn{3}{c}{Sparsity   = 30\%} & \multicolumn{3}{c}{Sparsity   = 70\%} \\ \cline{1-2} \cline{4-9} 
\multirow{3}{*}{Disturbance} & \multirow{3}{*}{\shortstack{Percentage \\ of Positive \\ Samples}} &  & \multicolumn{3}{c}{Training Sample Size} & \multicolumn{3}{c}{Training Sample Size} \\ \cline{4-9} 
 &  &  & \multirow{2}{*}{500} & \multirow{2}{*}{1000} & \multirow{2}{*}{1500} & \multirow{2}{*}{500} & \multirow{2}{*}{1000} & \multirow{2}{*}{1500}\\
&  & \\\hline
\multirow{14}{*}{0\%} & \multirow{7}{*}{10\%} & Opt. Explor. & 60.1\% (0.02) & 61.5\% (0.02) & 63.3\% (0.03) & 69.8\% (0.04) & 72.2\% (0.03) & 69.7\% (0.03) \\
 &  & Opt. Explor. (Full) & 60.1\% (0.02) & 61.6\% (0.02) & 63.3\% (0.03) & 69.8\% (0.04) & 72.2\% (0.03) & 69.7\% (0.03) \\
 &  & Opt. Exploit. & 58.1\% (0.02) & 70.2\% (0.02) & 70.4\% (0.02) & 67.4\% (0.03) & 72.6\% (0.03) & 73.7\% (0.01) \\
 &  & Opt. Exploit. (Full) & 58.5\% (0.03) & 70.1\% (0.03) & 70.4\% (0.02) & 68.7\% (0.02) & 72.9\% (0.03) & 74.3\% (0.01) \\
 &  & \textbf{CBEAL-2} & 58.1\% (0.02) & 70.3\% $(0.03)^{\bm{*}}$ & 71.8\% $(0.03)^{\bm{*}}$ & 67.5\% (0.03) & 75.9\% $(0.02)^{\bm{*}}$ & 74.3\% $(0.03)^{\bm{*}}$ \\
 &  & \textbf{CBEAL-4} & 61.5\% (0.03) & 67.3\% (0.03) & 74.3\% (0.02) & 69.6\% (0.03) & 77.4\% (0.02) & 73.1\% (0.03) \\
 &  & \textbf{CBEAL-6} & \textbf{61.2\% (0.02)} & \textbf{73.9\% (0.03)} & \textbf{72.5\% (0.02)} & \textbf{69.9\% (0.03)} & \textbf{73.9}\% (0.03) & \textbf{76.7\% (0.02)} \\\cline{2-9} 
 & \multirow{7}{*}{5\%} & Opt. Explor. & 52.6\% (0.01) & 60.8\% (0.03) & 60.5\% (0.03) & 59.8\% (0.02) & 64.3\% (0.03) & 66.3\% (0.03) \\
 &  & Opt. Explor. (Full) & 52.6\% (0.01) & 60.8\% (0.03) & 60.5\% (0.03) & 59.8\% (0.02) & 64.3\% (0.03) & 66.2\% (0.03) \\
 &  & Opt. Exploit. & 61.0\% (0.03) & 70.2\% (0.03) & 64.6\% (0.03) & 64.5\% (0.03) & 68.4\% (0.04) & 70.5\% (0.03) \\
 &  & Opt. Exploit. (Full) & 61.1\% (0.03) & 70.2\% (0.03) & 64.6\% (0.03) & 65.2\% (0.03) & 68.8\% (0.04) & \textbf{70.6\% (0.03)} \\
 &  & \textbf{CBEAL-2} & 60.6\% (0.02) & 70.4\% $(0.03)^{\bm{*}}$ & 65.2\% $(0.02)^{\bm{*}}$ & 63.9\% (0.03) & 68.7\% $(0.04)^{\bm{*}}$ & 67.4\% (0.03) \\
 &  & \textbf{CBEAL-4} & 60.7\% (0.03) & 67.9\% (0.03) & 64.8\% (0.02) & 63.5\% (0.03) & 67.6\% (0.03) & 66.7\% (0.03) \\
 &  & \textbf{CBEAL-6} & \textbf{65.3\% (0.02)} & \textbf{72.0\% (0.02)} & \textbf{66.7\% (0.02)} & \textbf{68.9\% (0.03)} & \textbf{69.4\% (0.03)} & 69.0\% (0.03) \\ \hline
\multirow{14}{*}{3\%} & \multirow{7}{*}{10\%} & Opt. Explor. & 60.4\% (0.02) & 60.8\% (0.02) & 62.1\% (0.03) & 62.0\% (0.02) & 62.1\% (0.02) & 71.7\% (0.03) \\
 &  & Opt. Explor. (Full) & 60.5\% (0.02) & 60.6\% (0.03) & 62.1\% (0.03) & 62.0\% (0.03) & 62.4\% (0.02) & 71.7\% (0.03) \\
 &  & Opt. Exploit. & 68.1\% (0.03) & 69.3\% (0.03) & 72.9\% (0.03) & \textbf{68.9\% (0.03)} & 69.2\% (0.02) & 75.0\% (0.04) \\
 &  & Opt. Exploit. (Full) & 60.5\% (0.03) & 69.8\% (0.03) & 73.3\% (0.02) & \textbf{68.9\% (0.03)} & 69.3\% (0.02) & 75.3\% (0.04) \\
 &  & \textbf{CBEAL-2} & 67.7\% (0.03) & 72.2\% $(0.03)^{\bm{*}}$ & 73.1\% $(0.03)^{\bm{*}}$ & 65.1\% (0.04) & 69.3\% $(0.02)^{\bm{*}}$ & 77.5\% $(0.01)^{\bm{*}}$ \\
 &  & \textbf{CBEAL-4} & 68.2\% (0.03) & 67.0\% (0.03) & 73.2\% (0.03) & 66.7\% (0.03) & 66.4\% (0.02) & 74.5\% (0.03) \\
 &  & \textbf{CBEAL-6} & \textbf{69.0\% (0.04)} & \textbf{71.6\% (0.03)} & \textbf{73.8\% (0.02)} & 68.0\% (0.03) & \textbf{70.8\% (0.02)} & \textbf{79.4\% (0.02)} \\\cline{2-9}
 & \multirow{7}{*}{5\%} & Opt. Explor. & 53.5\% (0.01) & 57.6\% (0.02) & 60.8\% (0.03) & 54.8\% (0.02) & 63.9\% (0.03) & 65.8\% (0.02) \\
 &  & Opt. Explor. (Full) & 53.5\% (0.01) & 57.5\% (0.02) & 60.8\% (0.03) & 54.8\% (0.02) & 63.9\% (0.03) & 65.8\% (0.02) \\
 &  & Opt. Exploit. & 61.0\% (0.03) & 67.4\% (0.03) & \textbf{69.2\% (0.02)} & \textbf{62.5\% (0.03)} & 70.7\% (0.03) & 77.1\% (0.02) \\
 &  & Opt. Exploit. (Full) & 62.4\% (0.03) & 67.4\% (0.03) & 69.0\% (0.02) & 62.1\% (0.03) & 71.3\% (0.04) & 78.3\% (0.02) \\
 &  & \textbf{CBEAL-2} & 61.5\% $(0.02)^{\bm{*}}$ & 65.2\% (0.04) & 65.6\% (0.03) & 60\% (0.03) & 72.3\% $(0.03)^{\bm{*}}$ & 77.9\% $(0.02)^{\bm{*}}$ \\
  &  & \textbf{CBEAL-4} & 57.7\% (0.02) & 63.2\% (0.03) & 65.5\% (0.03) & 58.7\% (0.04) & 72.7\% (0.02) & 73.5\% (0.03) \\
 &  & \textbf{CBEAL-6} & \textbf{61.9\% (0.02)} & \textbf{68.3\% (0.03)} & 64.5\% (0.03) & 58.8\% (0.03) & \textbf{74.6\% (0.03)} & \textbf{80.1\% (0.03)}\\\bottomrule
\end{tabular}}
\end{table}

Table~\ref{tab:simu_acc_agents} summarizes the averages of the classification accuracy and standard errors over 10 replications of the base learner trained by $\mathcal{D}_t$. 
It can be observed that the proposed \cbeal-6 outperforms the best individual agent in 20 of the 24 scenarios, which verifies that the ensemble of multiple agents with explicit consideration for both exploration and exploitation can effectively enhance the learning performance under a highly imbalanced class distribution.
The advantage on learning performance compared to benchmarks is more significant when there is no disturbance and the class proportion is more balanced (i.e., $ds = 0\%, pc=10\%$).
However, with a more severe imbalance (i.e., $pc=5\%$), "Opt. Exploit." and its variants sometimes achieve slightly higher accuracy. 
One possible reason is that under such scenarios, it will be more efficient to only focus on decision boundary learning since the number of positive samples is limited. 
Besides, the inferior performance of \cbeal-6 under the scenario $ds = 3\%, pc=10\%, n=500, sp=70\%$ can be caused by the disturbance.
It can be found that, in general, the learning performance of the base learner is improved with a data stream with a large size and a higher sparsity.
However, when the sparsity is low, the accuracy sometimes decreases as the training sample size increases, which can be caused by the high imbalance and the lack of degree of freedom.

Another finding is \cbeal-2 achieves comparable performance with the better of the two incorporated agents, which implies that the proposed ensemble framework enables the intelligent selection among candidate agents in an adaptive manner.
Under 14 out of 24 scenarios, it outperforms both "Opt. Explor." and "Opt. Exploit.", where the results are marked with $^{\bm{*}}$.

Comparing the "Opt. Explor." with "Opt. Explor. (Full)" and "Opt. Exploit." with "Opt. Exploit. (Full)", we find that consuming the remaining budget by random acquisition will not make a significant improvement on the learning performance under most scenarios. 
Sometimes it will select less contributive samples, which leads to a lower accuracy because of the highly imbalanced distribution.
Therefore, it validates that the agents have acquired the most informative samples based on their criteria. 
Thus, this variant will not be considered in the following analysis.

We also find that \cbeal -6 obtains a significantly better balanced labeled data set $\mathcal{D}_t$ with a higher percentage of positive samples compared to the benchmarks under most scenarios.
Detailed results can be found in Table A1 in the supplementary material.

\subsubsection{Scalability Study}
It has been observed in the previous results (i.e., Tables~\ref{tab:simu_acc_agents}) that \cbeal-6 achieves better performance than \cbeal-2 in general.
Here, we further investigate the following two questions: What will be the impact of the ensemble of varying numbers of agent pairs and how should the agents be selected. 
Rechecking the results in Table~\ref{tab:simu_acc_agents}, We observe that \cbeal-6 demonstrates a dominate superiority in the learning performance, which indicates the advantage brought by multiple agents.
However, comparing the result of \cbeal-4 with \cbeal-2, \cbeal-4 achieves better performance in fewer than half of all scenarios. 
The counterintuitive result indicates that the ensemble of more agents may not improve the performance.
To identify the reason, we investigate the acquisition decision made by each agent in the agent set and their standardized weights in \cbeal-6 under one scenario in Fig.~\ref{fig:simu_246} where \cbeal-4 shows inferior performance than \cbeal-2 but \cbeal-6 performs better. 

\begin{figure}[!htb]
    \centering
    \captionsetup{justification=centering}
    \subfloat[Acquisition decisions made by $LD_1, LD_2, SPF_1$ agents incorporated in CBEAL-6]{\includegraphics[height=4cm, width=0.40\textwidth]{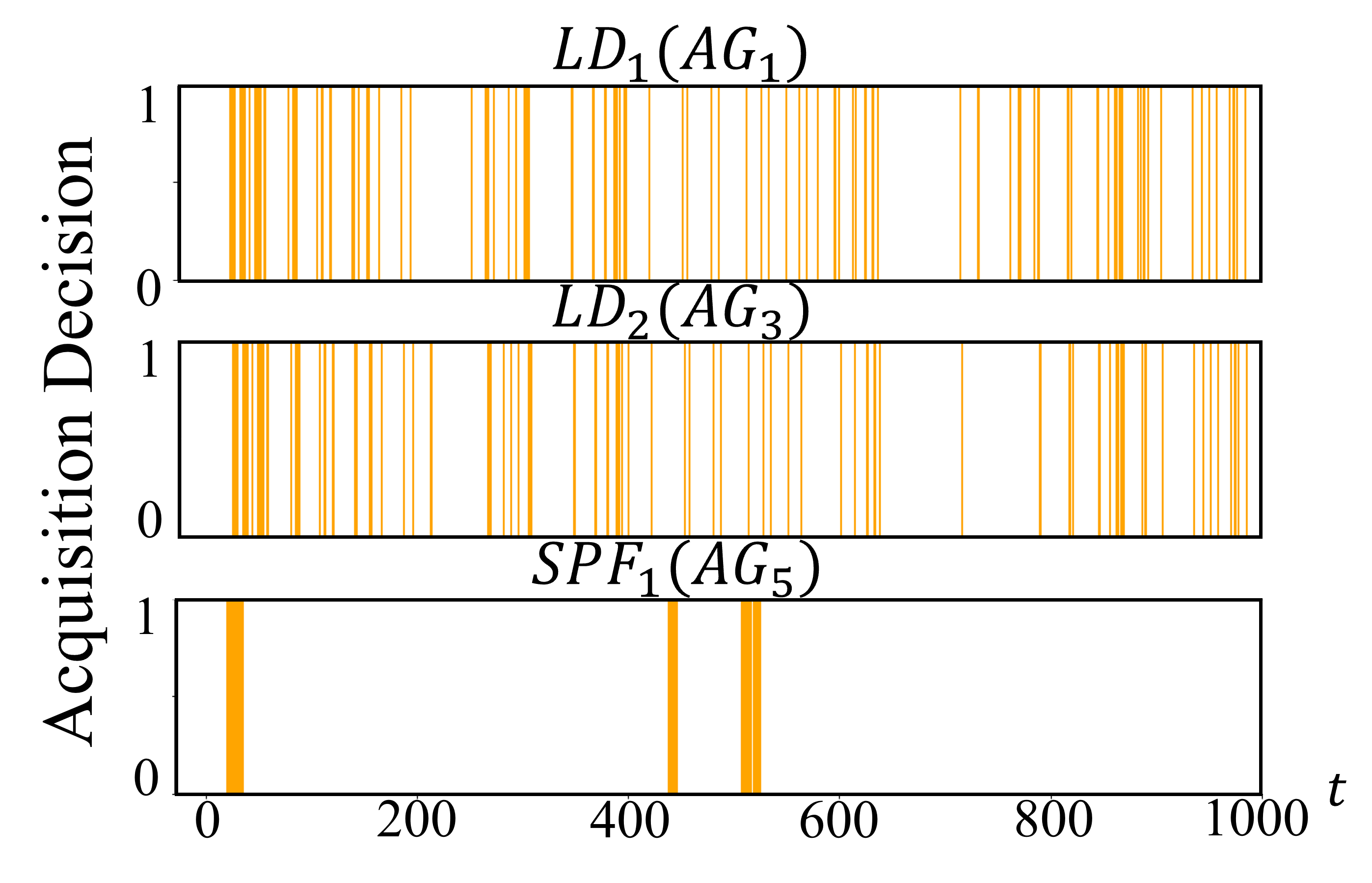}} 
    \subfloat[Acquisition decisions made by $RAL_1, RAL_2, RAL_3$ agents incorporated in CBEAL-6]{\includegraphics[height=4cm, width=0.40\textwidth]{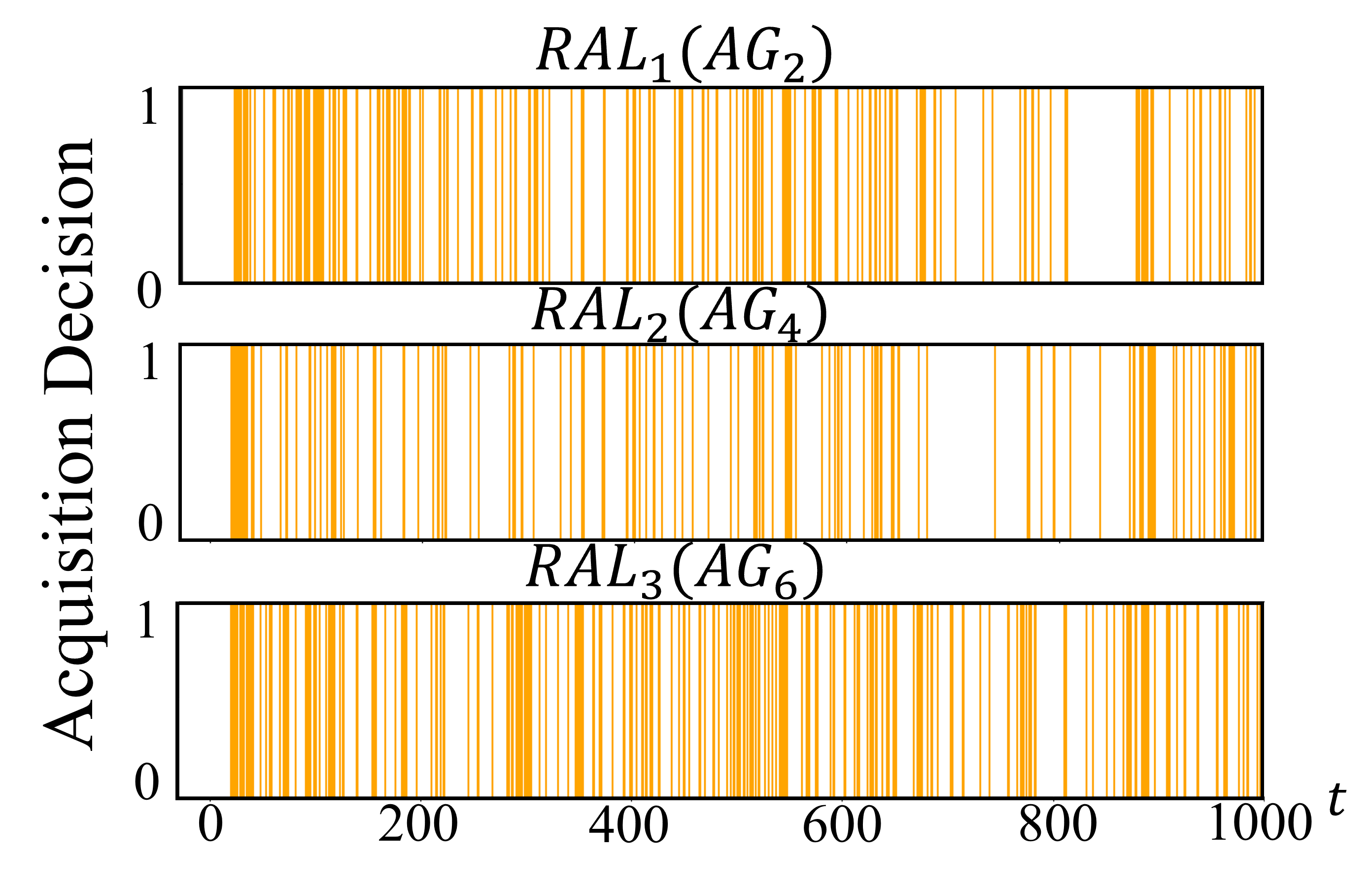}} \\
    \subfloat[Acquisition decisions made by \textsc{CBEAL-2, CBEAL-4, CBEAL-6} and their testing accuracy]{\includegraphics[height=4cm, width=0.40\textwidth]{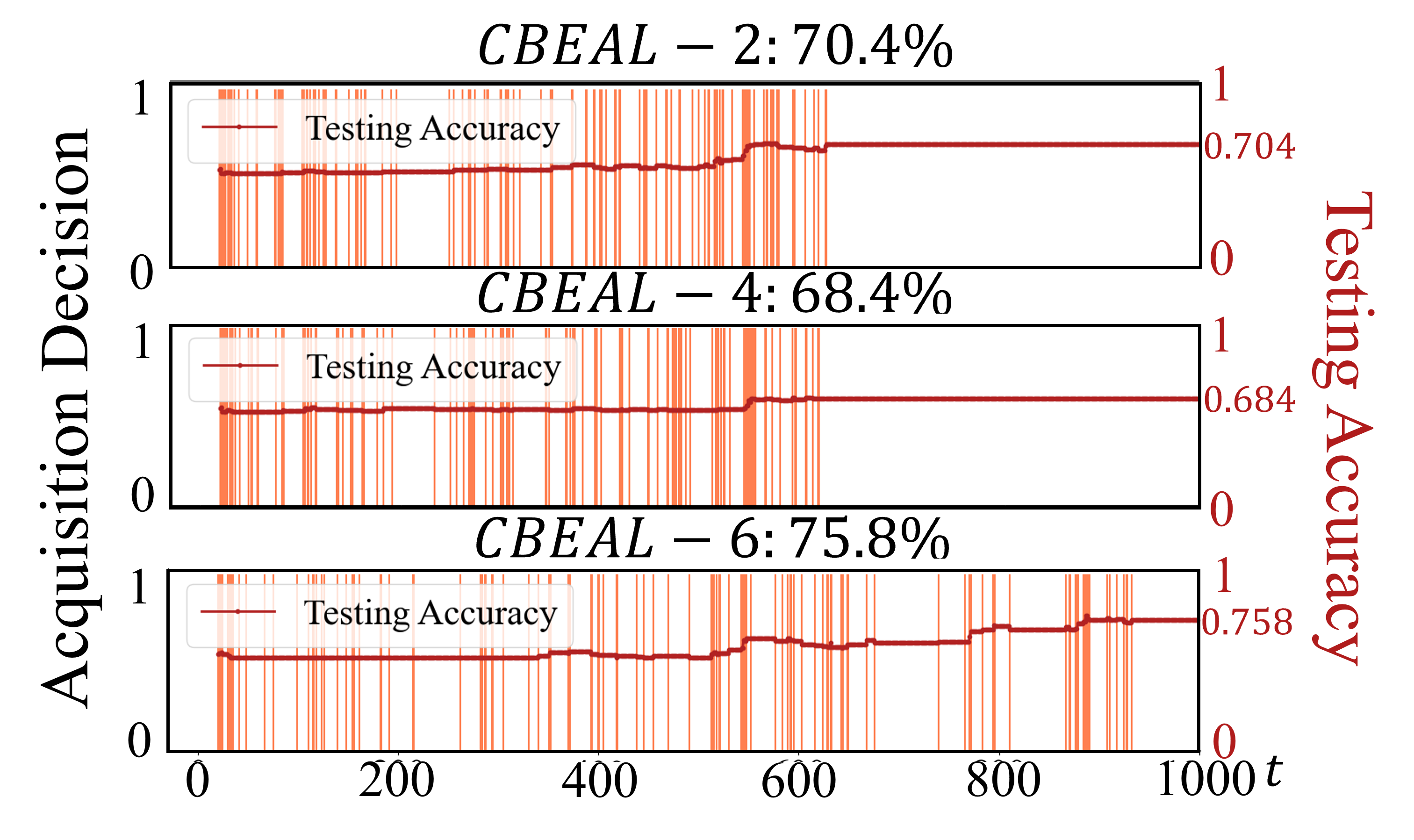}} 
    \subfloat[Standardized weight of each agents in \cbeal-6]{\includegraphics[height=4cm, width=0.40\textwidth]{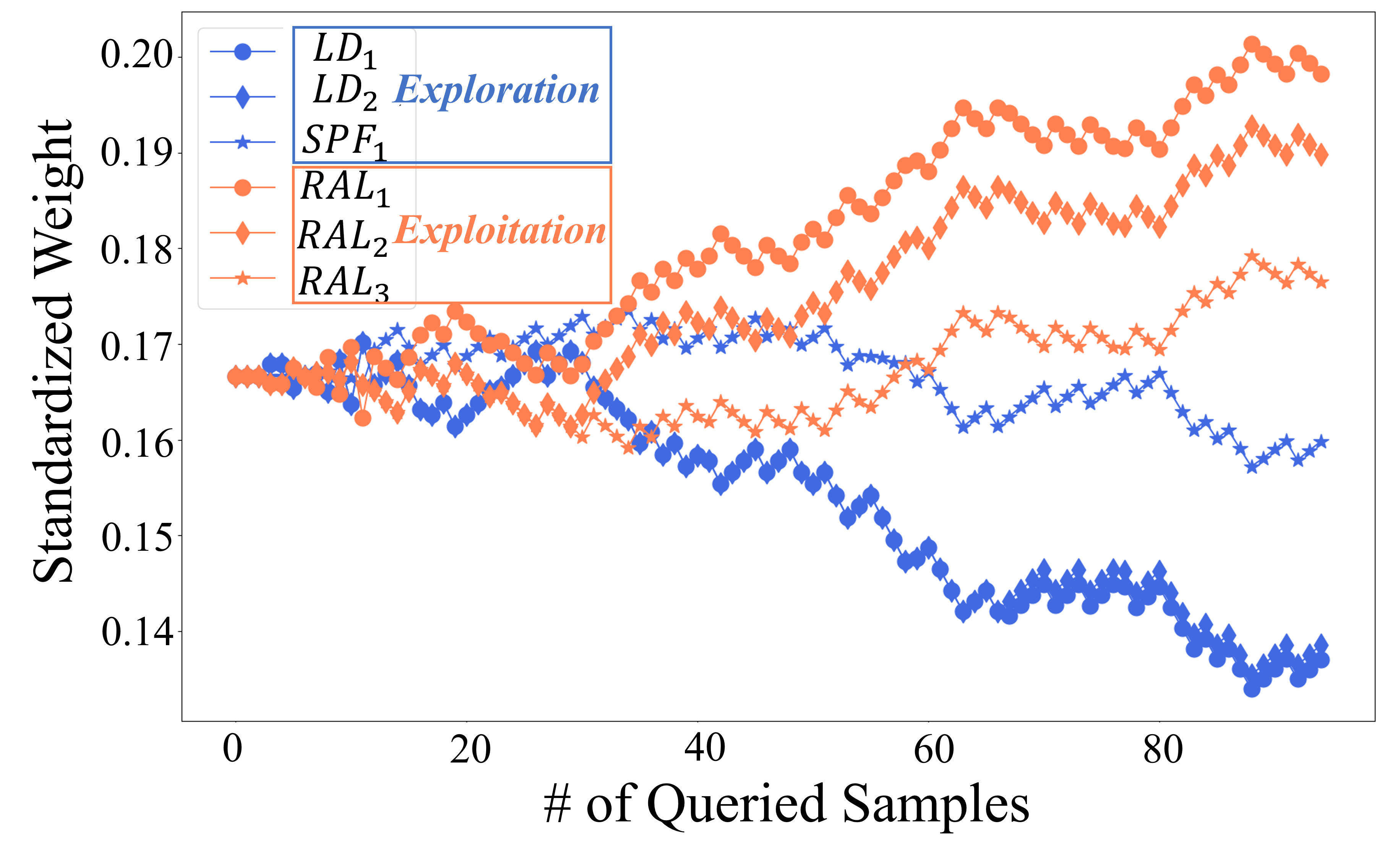}}
    \caption{Testing accuracy of \cbeal methods under the learning scenario $n=1000$, $ds=0\%$, $sp=30\%, pc=10\%$. (a)-(c): Bar charts of the acquisition decisions made by all candidate agents and \cbeal methods; (d) standardized weights of each candidate agent in \cbeal-6.}
    \label{fig:simu_246}
\end{figure}

It can be observed from the bar charts (Fig.~\ref{fig:simu_246}(a)-(c)) that in \cbeal-6, the first two pairs of agents ($LD_1$ and $LD_2$, $RAL_1$ and $RAL_2$) make similar acquisition decisions while the third pair behaves differently.
As a direct result of homogeneity, the standardized weights of the first two pairs of exploration- and exploration-oriented agents will be close and change in a similar pattern in \cbeal-6 (i.e., Fig.~\ref{fig:simu_246}(d)), which also causes a comparable performance of \cbeal-4 and \cbeal-2.
This also explains that the superior performance of \cbeal-6 lies in the heterogeneous decisions brought by the third pair of agents ($SPF_1$ and $RAL_3$).
Besides, the weights of agents in \cbeal-6 indicate that at the beginning, the exploration dominates the active learning process.  
Later, the proposed \cbeal switches its tendency to exploitation, and the exploration capability still remains adaptive to the data stream, which contributes to the effective and efficient online annotation.

In summary, the ensemble of distinguished agents provides comprehensive criteria to evaluate the informativeness of each streaming sample in terms of exploration and exploitation, thus achieving a well-balanced trade-off.
Since the acquisition behaviour of one active learning agent varies under different scenarios and there does not exist one overall winner, \cbeal-6 is recommended as a default configuration to solve these challenging learning tasks.

\subsubsection{Comparison Study with Benchmark Methods}
Finally, \cbeal is compared with other four benchmark methods (i.e., RS, US, \textsc{DBALStream} and \textsc{QBC-PYP}) and  Table~\ref{tab:simu_benchmark} summarizes the classification accuracy of the base learner trained by samples acquired by each method.

\begin{table}[!htb]
\caption{The average values and standard errors (in parenthesis) of the classification accuracy in the simulation study reported over 10 replications. Significant best results are highlighted in \textbf{bold}.}\label{tab:simu_benchmark}
\resizebox{\textwidth}{!}{
\begin{tabular}{ccccccccc}
\hline
\multicolumn{2}{c}{Level} & \multirow{4}{*}{Method} & \multicolumn{3}{c}{Sparsity   = 30\%} & \multicolumn{3}{c}{Sparsity  = 70\%} \\ \cline{1-2} \cline{4-9} 
\multirow{3}{*}{Disturbance} & \multirow{3}{*}{\shortstack{Percentage \\ of Positive \\ Samples}} &  & \multicolumn{3}{c}{Training Sample Size} & \multicolumn{3}{c}{Training Sample Size} \\ \cline{4-9} 
 &  &  & \multirow{2}{*}{500} & \multirow{2}{*}{1000} & \multirow{2}{*}{1500} & \multirow{2}{*}{500} & \multirow{2}{*}{1000} & \multirow{2}{*}{1500}\\
&  & \\\hline
\multirow{14}{*}{0\%} & \multirow{7}{*}{10\%} & Initial & 51.4\% (0.00) & 52.3\% (0.01) & 52.2\% (0.01) & 51.0\% (0.00) & 51.1\% (0.01) & 50.9\% (0.00) \\
 &  & RS & 52.0\% (0.01) & 57.3\% (0.02) & 55.6\% (0.03) & 57.4\% (0.02) & 60.6\% (0.02) & 62.1\% (0.02) \\
 &  & US & 60.6\% (0.03) & 70.3\% (0.04) & 67.9\% (0.04) & 65.6\% (0.04) & 72.8\% (0.04) & 72.4\% (0.04) \\
 &  & \textsc{DBALStream} & 54.6\% (0.01) & 60.6\% (0.02) & 58.1\% (0.02) & 56.7\% (0.01) & 60.7\% (0.02) & 59.4\% (0.02) \\
 &  & QBC-PYP & 51.2\% (0.01) & 67.8\% (0.04) & 68.1\% (0.03) & 63.5\% (0.04) & 70.3\% (0.04) & 66.1\% (0.04) \\
 &  & \textbf{CBEAL-6} & \textbf{61.2\% (0.02)} & \textbf{73.9\% (0.03)} & \textbf{72.5\% (0.02)} & \textbf{69.9\% (0.03)} & \textbf{73.9\% (0.03)} & \textbf{76.7\% (0.02)} \\
 &  & All Training Data & 76.7\% (0.01) & 78.9\% (0.02) & 81.5\% (0.01) & 77.7\% (0.01) & 81.5\% (0.01) & 80.7\% (0.01) \\\cline{2-9}
 & \multirow{7}{*}{5\%} & Initial & 52.5\% (0.01) & 52.9\% (0.01) & 51.1\% (0.00) & 51.7\% (0.01) & 52.1\% (0.01) & 51.5\% (0.01) \\
 &  & RS & 51.9\% (0.01) & 53.7\% (0.01) & 55.1\% (0.02) & 54.9\% (0.02) & 55.6\% (0.03) & 53.8\% (0.01) \\
 &  & US & 62.2\% (0.04) & 70.3\% (0.04) & 61.6\% (0.04) & \textbf{69.8\% (0.03)} & 66.6\% (0.04) & 68.3\% (0.04) \\
 &  & \textsc{DBALStream} & 54.7\% (0.01) & 56.3\% (0.00) & 54.2\% (0.01) & 54.4\% (0.01) & 55.2\% (0.01) & 54.6\% (0.01) \\
 &  & QBC-PYP & 53.4\% (0.03) & 61.9\% (0.03) & 60.2\% (0.03) & 58.0\% (0.03) & 64.0\% (0.06) & 63.8\% (0.05) \\
 &  & \textbf{CBEAL-6} & \textbf{65.3\% (0.02)} & \textbf{72.0\% (0.02)} & \textbf{66.7\% (0.02)} & 68.9\% (0.03) & \textbf{69.4\% (0.03)} & \textbf{69.0\% (0.03)} \\
 &  & All Training Data & 75.1\% (0.02) & 78.7\% (0.01) & 78.4\% (0.01) & 77.7\% (0.01) & 80.2\% (0.01) & 80.7\% (0.01) \\\hline
\multirow{14}{*}{3\%} & \multirow{7}{*}{10\%} & Initial & 51.9\% (0.01) & 51.3\% (0.01) & 50.7\% (0.00) & 51.1\% (0.00) & 50.9\% (0.00) & 51.2\% (0.00) \\
 &  & RS & 53.0\% (0.01) & 56.2\% (0.02) & 57.7\% (0.01) & 54.3\% (0.01) & 53.8\% (0.01) & 60.6\% (0.03) \\
 &  & US & 60.7\% (0.04) & 68.0\% (0.04) & 70.5\% (0.04) & 67.6\% (0.04) & 66.2\% (0.04) & 71.0\% (0.05) \\
 &  & \textsc{DBALStream} & 56.9\% (0.02) & 60.1\% (0.03) & 60.7\% (0.02) & 56.1\% (0.01) & 57.1\% (0.02) & 61.3\% (0.02) \\
 &  & QBC-PYP & 61.2\% (0.03) & 56.8\% (0.03) & 68.8\% (0.04) & 64.5\% (0.04) & 63.6\% (0.04) & 66.4\% (0.04) \\
 &  & \textbf{CBEAL-6} & \textbf{69.0\% (0.04)} & \textbf{71.6\% (0.03)} & \textbf{73.8\% (0.02)} & \textbf{68.0\% (0.03)} & \textbf{70.8\% (0.02)} & \textbf{79.4\% (0.02)} \\
 &  & All Training Data & 79.2\% (0.02) & 78.8\% (0.01) & 81.0\% (0.01) & 80.9\% (0.01) & 76.4\% (0.01) & 81.1\% (0.01) \\\cline{2-9}
 & \multirow{7}{*}{5\%} & Initial & 52.2\% (0.01) & 53.2\% (0.01) & 54.1\% (0.01) & 51.5\% (0.01) & 51.2\% (0.01) & 51.5\% (0.00) \\
 &  & RS & 52.9\% (0.01) & 54.4\% (0.01) & 52.8\% (0.02) & 53.2\% (0.02) & 54.5\% (0.02) & 52.6\% (0.01) \\
 &  & US & 59.4\% (0.03) & 67.9\% (0.04) & \textbf{68.0\% (0.03)} & 59.9\% (0.04) & 62.6\% (0.04) & 73.9\% (0.04) \\
 &  & \textsc{DBALStream} & 54.3\% (0.01) & 56.0\% (0.01) & 54.7\% (0.01) & 55.6\% (0.01) & 55.1\% (0.01) & 57.6\% (0.02) \\
 &  & QBC-PYP & \textbf{62.3\% (0.04)} & 58.1\% (0.04) & 62.5\% (0.04) & \textbf{60.2\% (0.04)} & 64.2\% (0.04) & 61.3\% (0.04) \\
 &  & \textbf{CBEAL-6} & 61.9\% (0.02) & \textbf{68.3\% (0.03)} & 64.5\% (0.03) & 58.8\% (0.03) & \textbf{74.6\% (0.03)} & \textbf{80.1\% (0.03)} \\
 &  & All Training Data & 73.4\% (0.02) & 78.0\% (0.01) & 76.6\% (0.01) & 73.1\% (0.02) & 78.1\% (0.01) & 80.5\% (0.01)\\\hline
\end{tabular}}
\end{table}

By investigating the results in Table~\ref{tab:simu_benchmark}, it is concluded that \cbeal-6 achieves significantly better performance compared to the benchmarks under most scenarios.
With a limited budget, \cbeal-6 can achieve high classification accuracy close to that of using all the training data with a data stream of larger size and higher sparsity (i.e., $n=1500, sp=70\%$).

Considering the benchmark methods, \textsc{US} demonstrates its competitive performance compared to other benchmarks, but with higher standard errors. 
This indicates the importance of exploitation for the online annotation scenarios, and this also explains the superiority of RAL-Agents in Table~\ref{tab:simu_acc_agents}.
However, the lack of adaptiveness to the data stream causes its inferior performance compared to \cbeal-6.
The inferior performance of \textsc{DBALStream} might attribute to its concentration on samples with both high local density and large margin, which does not perform effective input variable space exploration. 
On the contrary, the proposed LD-Agent is able to complete this exploration task.
\textsc{QBC-PYP} underperforms in comparison to US, primarily due to its inadequate exploitation capability under a highly imbalanced class distribution. During the streaming-in process, it takes a mixture of Gaussians to quantify the ambiguity of the class membership of one sample. The lack of connection between its acquisition decision and the evolving performance of the base learner results in the lack of effective exploitation.


Additional results of other benchmarks on the toy example are included in the supplementary material, where \textsc{DBALStream} and \textsc{QBC-PYP} outperform \textsc{US} and \textsc{RS} in terms of learning performance. 
The result implies that \textsc{US} will achieve poor performance if the base learner is confident about its classification result at the beginning with the initial training data.

Overall, the results validate that the proposed method can acquire samples effectively and efficiently in an adaptive manner under various circumstances, confirming the benefits of the ensemble of designed exploration-oriented and exploitation-oriented agents.

\section{Case Study}\label{sec:case}
The proposed \cbeal method is applied to a FDM process for online quality modeling and inspection \citep{chen2016variation}, which is introduced as the motivation example in Section~\ref{Sec:Intro}.
During the printing process, various \textit{in situ} process variables (i.e., vibration, nozzle temperature, etc.) are collected in the ICPS to monitor the process and predict the quality of the FDM part \citep{sun2016logistic}.
Here, we focus on the layerwise surface roughness as a binary quality indicator, which is judged and annotated as conforming/nonconforming by domain experts. 
Fig.~\ref{fig:cs_part} shows the example normal and rough surfaces of the printed FDM part.
To enable real-time quality prediction and online modeling, the  \textit{in situ} measurements are registered and divided into $10$-second windows as samples.
The online updating of the quality model requires experts to consistently observe and examine the surface roughness during the printing process for the window-wise annotation, which is labor-intensive and time-consuming.
Therefore, \cbeal is employed to develop an accurate quality model with less labeling efforts and high-quality training data through wisely selecting the samples for annotation.

\begin{figure}[!htb]
    \centering
    \captionsetup{justification=centering}
    \subfloat[FDM part modified from NAS 979 part ]{\includegraphics[width=0.25\textwidth]{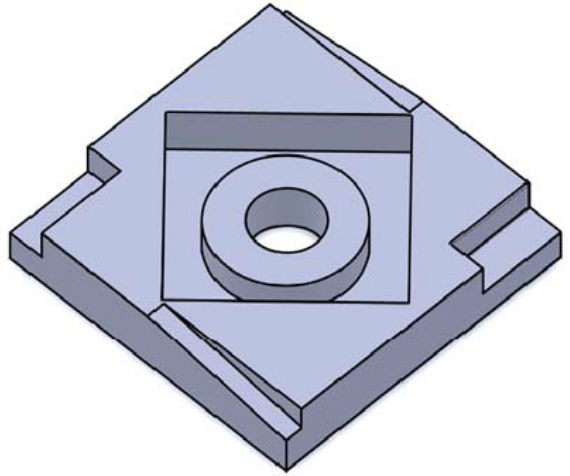}} \hfill
    \subfloat[Normal surface]{\includegraphics[width=0.25\textwidth]{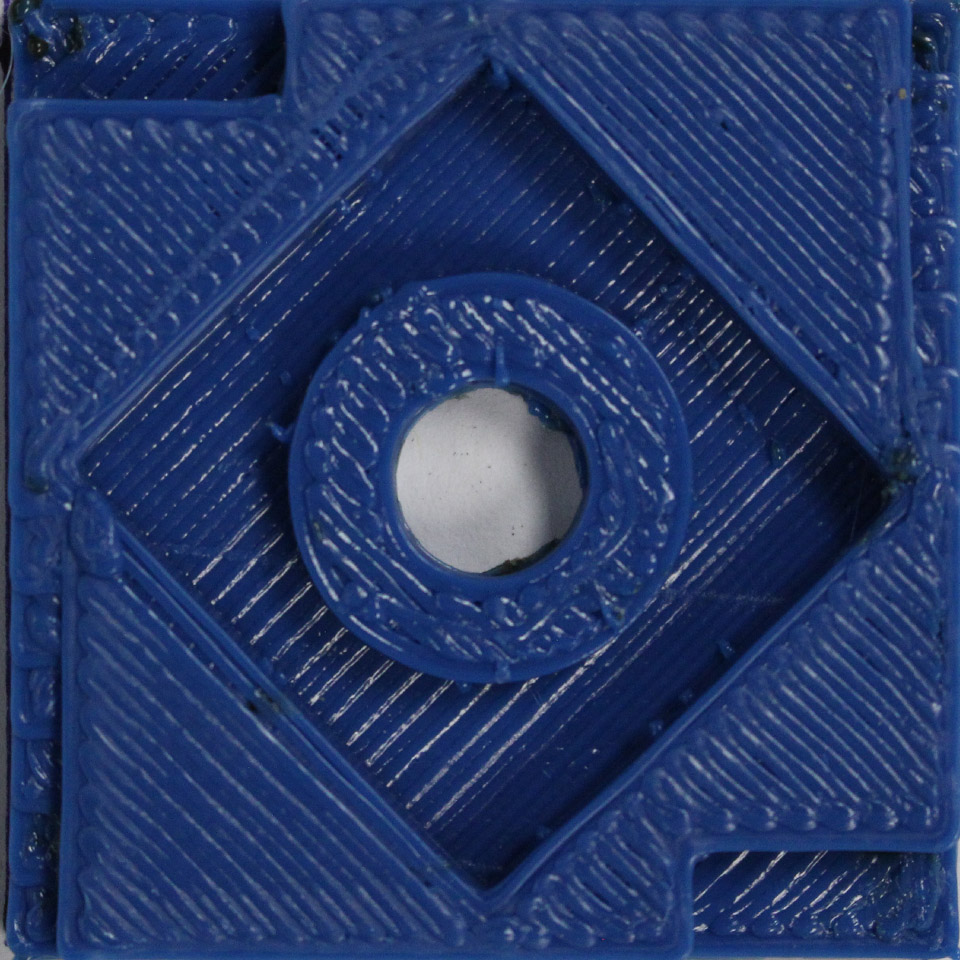}} \hfill
    \subfloat[Rough surface]{\includegraphics[width=0.25\textwidth]{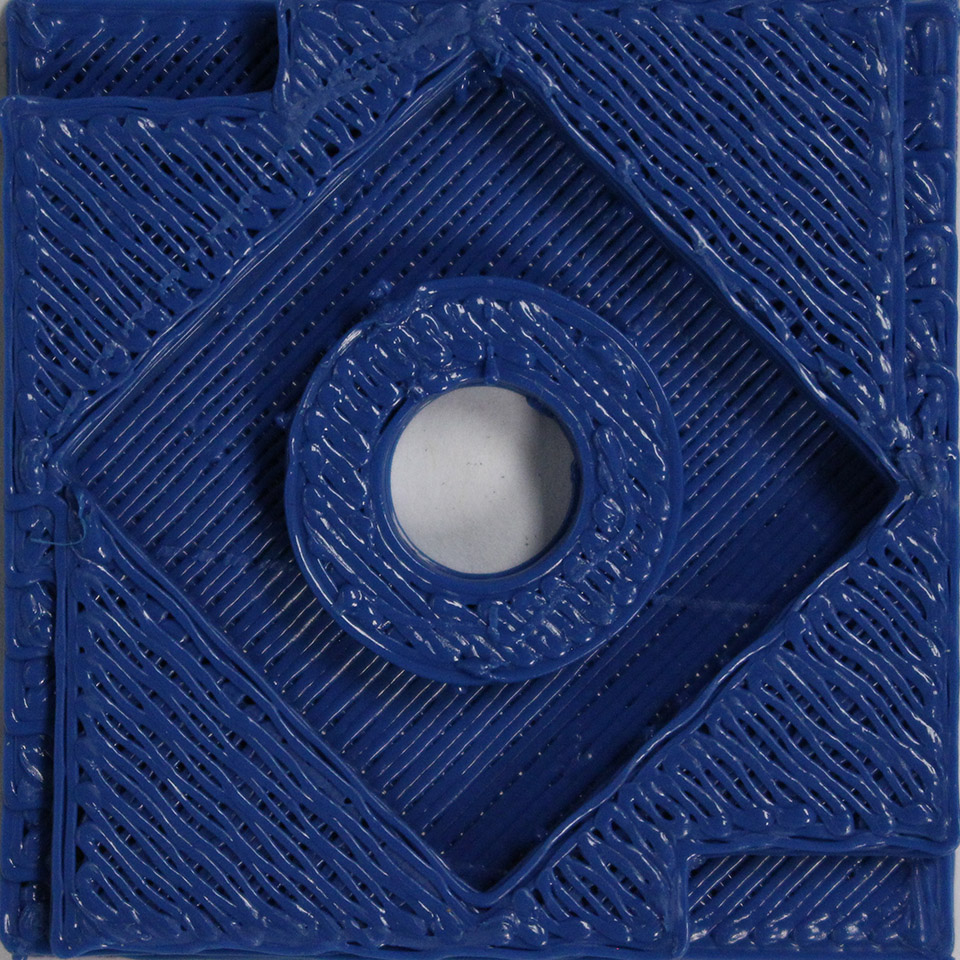}}
    \caption{Printed part in the case study and the example surfaces, (a) is modified from \citet{huang2020detecting} with authors' permission}
    \label{fig:cs_part}
\end{figure}

We refer \cite{chen2016variation} for the details of the data collection.
During the process, the \textit{in situ} extruder vibration, table vibration, nozzle temperature and table temperature are measured and collected in a functional data format.
Considering the wavelet analysis applied to the funcational measures, the process setting variables (i.e., feed/flow ratio and layer thickness) and summary statistics for each functional measurement (i.e., mean, standard deviation, skewness, and kurtosis), 519 features are obtained in total as the model input.
48 FDM parts are printed successfully in total with 1588 samples (i.e., windows of measurements).
Label the nonconforming samples by 1 and the conforming samples by 0.
As a common scenario in quality inspection, we find there are 180 nonconforming samples in total, which implies a highly imbalanced class distribution.
We also notice that the positive labels appear in succession during the process (i.e., $\dots0000011100\dots$) because the malfunction of the printer in a period of time will affect the quality of several consecutive layers. 
Therefore, we maintain the original order of samples in sequence when we separate the training and testing set.
A logistic regression model with $L1$ penalty is adopted as the base learner for the quality online modeling.
In brief, the highly imbalanced data stream with patterns in sequence and an underlying multimodal distribution (i.e., Fig.\ref{fig:motivation}) brings a challenging online annotation scenario.

We evaluate the classification accuracy of the base learner of the proposed \cbeal-6 under different level of budgets (i.e., $\{3\%, 5\%, 10\%, 15\%, 20\%\}$) with 10 replications.
In each replication, 1/3 samples are extracted with a random starting point from the whole data set in time order as the testing data set, with the remaining for training.
The first 10 samples in the training set will be used for the model pretraining as $D_0$.
The best performed individual agents in Table~\ref{tab:agents} and the other four benchmarks in Section~\ref{sec:sim} are employed for the comparison.

\begin{figure}[!htb]
    \centering
    \captionsetup{justification=centering}
    \subfloat[Averaged training classification accuracy]{\includegraphics[width=0.48\textwidth]{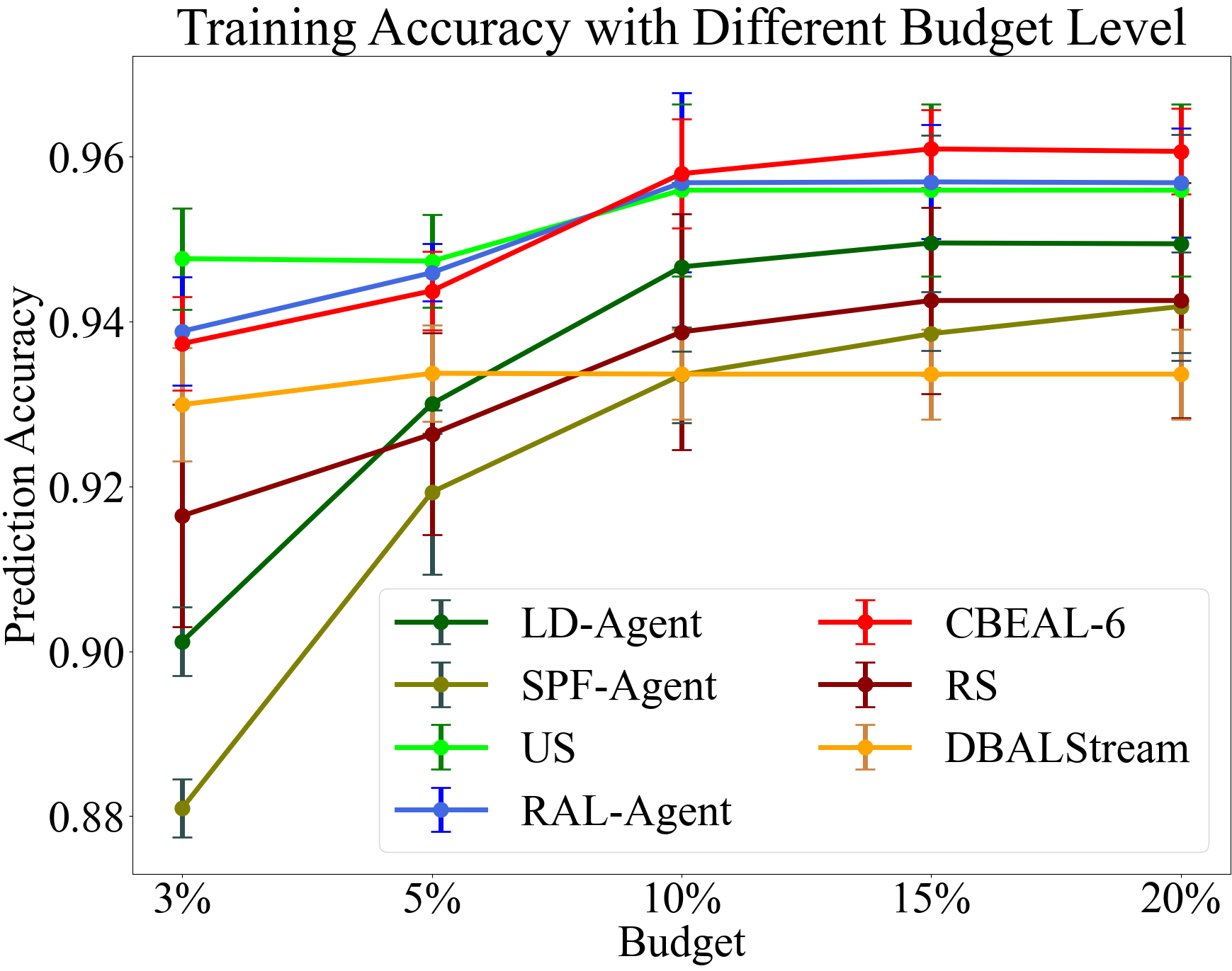}}
    \subfloat[Averaged testing classification accuracy]{\includegraphics[width=0.48\textwidth]{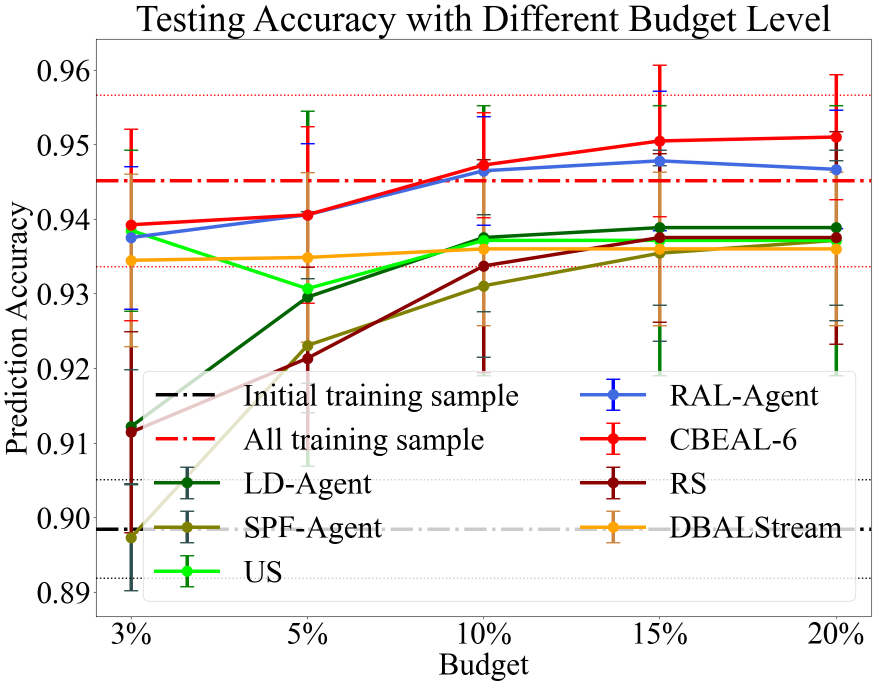}} 
    \caption{The average values of training and testing classification accuracy of \cbeal and benchmark methods for the case study reported over 10 replications}
    \label{fig:cs_acc}
\end{figure}

Fig.~\ref{fig:cs_acc} demonstrates the averages of the training and testing classification accuracy of the base learners, where the error bars represent the standard errors over 10 replication.
The dash-dotted red line represents the testing accuracy of the base learner trained by all training data and the dotted red line is the standard error.
Correspondingly, the testing accuracy of the base learner 
trained by the initial training data is marked with black lines in a similar way.
Here, the result of \textsc{QBC-PYP} is not included since it consistently refuses to acquire samples during the streaming process, which might be caused by the pattern of positive labels in the sequence that requires the method to adapt its exploration-exploitation tendency to the data stream.

By investigating the results, it can be observed that the proposed method consistently outperforms its incorporated agents and the rest of benchmark methods in testing accuracy under different levels of budgets.
The testing classification accuracy of most benchmark methods has an increasing trend as the budget increases from $3\%$ to $10\%$ and then the trend goes smoother, which indicates the samples acquired in the $10\%$ budget make the most of the contribution to the quality modeling. 
However, the testing accuracy of \cbeal-6 keeps increasing with a higher budget, which validates the continuous acquisition of informative samples. 
Notably, the testing accuracy of \cbeal-6 exceeds the accuracy of the base learner trained by all available samples when the budget is higher than $10\%$.
This implies, despite the high dimension, a training set with a smaller sample size but better balanced samples has better quality, thus improving the modeling accuracy. 
Therefore, the proposed method not only reduces the labelling efforts but also contributes to the online modeling performance via acquiring high-quality samples.

In conclusion, the case study verifies that \cbeal achieves a well-balanced exploration-exploitation trade-off during the streaming process in an adaptive manner, which enables the highly accurate online quality modeling with limited labelling efforts for the FDM quality inspection.

\section{Conclusions}\label{sec:conclusion}
While the high-speed, high-volume streaming data brought by the ICPS enhance the data-driven decision-making by AI models, the quality of online data may hamper the modeling performance for manufacturing.
To provide resilient AI modeling performance, informative samples need to be acquired from the streaming data to provide a high-quality training data set as well as reducing the human annotation efforts required for AI incubation in an online manner.
Existing active learning methods cannot balance the exploration-exploitation trade-off in the challenging online annotation scenario.
In this work, we propose an ensemble of exploration- and exploitation-oriented
active learning agents as \cbeal. 
With the ensemble of agents considering each objective explicitly and the proposed \textsc{Exp4.P-EWMA solver}, \cbeal adjusts the exploration and exploitation tendency adaptive to different online annotation
scenarios. 
Simulation studies and the case study in FDM processes demonstrate the advantage of CBEAL over benchmarks
in the learning accuracy with a limited acquisition budget.

We notice some limitations of the proposed method. 
First, \cbeal shows its superiority under learning scenarios with a complex input data distribution.
Under other generic scenarios without a demanding exploration-exploitation trade-off, \cbeal may lose its advantage due to the encoded explicit exploration objective.
In this case, \cbeal can enhance its concentration on exploitation by adding exploitation-oriented agents or removing exploration-oriented agents.
Second, the hyperparameters for the agents are optimized with the grid search. 
We will formulate the hyperparameter tuning as a meta-learning problem for different base learners as future work \citep{thrun2012learning}. 

The work leaves us with several future research directions.
Firstly, the reward function in \cbeal can be adjusted to quantify regression model performance, such as the root mean squared error of importance-weighted training samples \citep{beygelzimer2009importance}. 
Concurrently, the exploitation-oriented agent's reward function could be set as either the base learner's bootstrap uncertainty \citep{endo2015confidence} or the expected model change \citep{cai2013maximizing}. 
These adjustments extend the framework to generic supervised models.
Second, we will investigate formulating the budget as a hard constraint  to maintain strict control over the acquisition cost. 
To realize this, we propose modifying our \textsc{Exp4.P-EWMA} algorithm and integrating it with the knapsack bandit algorithm \citep{badanidiyuru2018bandits} to effectively address this hard constraint. 
Furthermore, we consider the ensemble of multiple modalities as the second level actions in \cbeal, enabling the learner to decide not only whether to acquire a sample but also its data source, inspired by \citep{wang2021moss}.

\newpage

\bibliographystyle{plainnat}
\bibliography{ref.bib}

\newpage
\onecolumn

\appendix


\section{Proofs}

\section{Characterization of Agents}
To validate the exploration and exploitation capability of the proposed agents, theoretical justification is provided for the designed acquisition criteria.
The variance of the acquired samples $\mathcal{D}_t$ by one agent serves as an appropriate metric for assessing its exploration and exploitation activity during the online annotation process.
A higher variance suggests a learner's ability to explore the input variable space via acquiring samples in a larger region, whereas a lower variance implies a high frequency of acquisition in a small region for exploitation.
To compare the variance of $\mathcal{D}_t$, we examine the probability of a single sample being acquired by the proposed agents.

\subsection{Exploration-oriented Agent}
We assume the streaming data belong to a mixture of Gaussian distributions.
Denote the previously observed samples stored in the sliding window $\mathcal{W}$ before time $t$ as the set $\{\bm{x}_1, \bm{x}_2, ..., \penalty 0 \bm{x}_L\}$, where $\bm{x}_i$ belongs to the $i$-th Gaussian distribution (i.e., $\bm{x}_i\sim \mathcal{N}_q(\bm{\mu}_i, \bm{\Sigma}^{(i)})$).  
Given a streaming sample $\bm{x}_t$ at time $t$ which follows another Gaussian distribution (i.e., $\bm{x}_t\sim \mathcal{N}_q(\bm{\mu}_k, \bm{\Sigma}^{(k)})$),
the acquisition probability of the proposed low-density based exploration agent is:
 \begin{align}
    p_t & = \frac{lsf(\bm{x}_t))}{L\delta_L} = \frac{\sum_{i=1}^{L}\mathbb{I}\{\max d(\bm{x}_i, \bm{x}_j), j\neq i, j\in\{1, ..., L\}<d(\bm{x}_i, \bm{x}_t)\}}{L\delta_L}.
\end{align}
 With the distance $d(\bm{x}_i, \bm{x}_j)=\sqrt{\norm{\bm{x}_i-\bm{x}_j}_2^2}$, the acquisition probability can be rewritten as: 
 \begin{align}
    p_t & = \sum_{i=1}^L{P\{\max_{j\neq i, j\in\{1,...,L\}}{\norm{\bm{x}_i-\bm{x}_j}_2^2<\norm{\bm{x}_i-\bm{x}_t}_2^2}\}}/(L\cdot\delta_L).
\end{align}

To evaluate the acquisition decision, we calculate the expectation of the acquisition probability:
\begin{align}\label{eq:7.3}
        \mathbb{E}_{\bm{x}_i, \bm{x}_j, \bm{x}_k}[p_t] & = \mathbb{E}\bigg[\sum_{i=1}^L{P\{\max_{j\neq i, j\in\{1,...,L\}}{\norm{\bm{x}_i-\bm{x}_j}_2^2<\norm{\bm{x}_i-\bm{x}_t}_2^2}\}}/(L\cdot\delta_L)\bigg]\nonumber\\
        & = \sum_{i=1}^L\mathbb{E}[P\{\max_{j\neq i, j\in\{1,...,L\}}{\norm{\bm{x}_i-\bm{x}_j}_2^2<\norm{\bm{x}_i-\bm{x}_t}_2^2}\}]/(L\cdot\delta_L).
\end{align}

Denote $X_{i,j}$ and $Y_{i,t}$ as the squared Euclidean distance between $\bm{x}_i, \bm{x}_j$ and $\bm{x}_i, \bm{x}_t$, (i.e., $X_{i,j}=\norm{\bm{x}_i-\bm{x}_j}_2^2\in \mathbb{R}$ where $j\neq i, i,j\in\{1, ..., L\}$, and $Y_{i,t}=\norm{\bm{x}_i-\bm{x}_t}_2^2 \in \mathbb{R},i\in\{1, ..., L\}$).
Let $Z_i = \max_{j\neq i, j\in\{1,...,L\}}{\norm{\bm{x}_i-\bm{x}_j}_2^2}$, $ F_Z(z) = P(X_{i,1}<z,X_{i,2}<z, \dots, X_{i,L}<z) = \Pi_{j=1}^{L}F_{X_{i,j}}(z).$
\begin{align}
    \mathbb{E}_{\bm{x}_i, \bm{x}_j, \bm{x}_k}[p_t] & = \sum_{i=1}^L \mathbb{E}[\mathbb{E}[P\{Z_i<Y_{i,t}\}|Y_{i,t}]]/(L\cdot\delta_L)\nonumber\\
    & = \sum_{i=1}^L \mathbb{E}[F_Z(Y_{i,t})]/(L\cdot\delta_L)\nonumber\\
     & = \sum_{i=1}^L \int_{0}^{\infty}\bigg[\Pi_{j=1, j\neq i}^{L-1}F_{X_{i,j}}(y))\bigg]f_{Y_{i,t}}(y)dy/(L\cdot\delta_L),
\end{align}
where $F_{X_{i,j}}(\cdot)$ is the cumulative density function for $X_{i,j}$ and $f_{Y_{i,t}}$ is the probability density function for $Y_{i,t}$.

Since the samples in $\mathcal{W}$ and the incoming sample $\bm{x}_t$ can be considered as independent draws from different Gaussian distributions, we have $\bm{x}_i-\bm{x}_j\sim \mathcal{N}(\bm{x}_i-\bm{x}_j, \bm{\Sigma}^{(i)}+\bm{\Sigma}^{(j)})$, and $\bm{x}_i-\bm{x}_t\sim \mathcal{N}(\bm{x}_i-\bm{x}_t, \bm{\Sigma}^{(i)}+\bm{\Sigma}^{(k)})$.
Therefore, $X_{i,j}$ and $Y_{i,t}$ are quadratic forms of random normal variables and follow a generalized chi-square distribution.
The probability density function for $X_{i,j}$ is \citep{mathai1992quadratic}:
\begin{align}
    f_{X_{i,j}=y} = \sum_{k=1}^{\infty}(-1)^kc_k\frac{y^{\frac{q}{2}+k-1}}{\Gamma(\frac{q}{2}+k)}, 0\leq y\leq \infty.
\end{align}
 The cumulative density function is:
 \begin{align}
    F(X_{i,j}<y) = \sum_{k=0}^{\infty}(-1)^kc_k\frac{y^{\frac{q}{2}+k}}{\Gamma(\frac{q}{2}+k+1)}, 0<y<\infty,
\end{align}
where $c_0$ and $c_k$ are defined by:
\begin{align}
    &H^T(\bm{\Sigma}^{(i)}+\bm{\Sigma}^{(j)})H=\diag (\lambda_1,...,\lambda_q)\\
    & H^TH  = I,\\
    & b = H^T(\bm{\Sigma}^{(i)}+\bm{\Sigma}^{(j)})^{-\frac{1}{2}}(\bm{\mu}_1-\bm{\mu}_2),\\
    & c_0 = \exp{(-\frac{1}{2}\sum_{j=1}^qb_j^2)\Pi_{j=1}^q(2\lambda_j)^{-\frac{1}{2}}}\\
    & d_k = \frac{1}{2}\sum_{j=1}^{q}(1-kb_j^2)(2\lambda_j)^{-k}, k\geq 1\\
    & c_k = \frac{1}{k}\sum_{r=0}^{k-1}d_{k-r}c_r, k\geq 1
\end{align}
 
For tractability of the computation, we assume the input variables in the Gaussian distributions to be independent (i.e., $\bm{\Sigma}^{(i)}=\sigma_i^2\bm{I})$, then the distance can be simplified and approximated by a normal distribution when $q$ is large based on the central limit theorem:
\begin{align}
X_{i,j}& \sim \mathcal{N}_q\bigg(\norm{\bm{\mu}_i-\bm{\mu}_j}^2 + (\sigma_i^2+\sigma_j^2)q, 4(\sigma_i^2+\sigma_j^2)\norm{\bm{\mu}_i-\bm{\mu}_j}^2+2q(\sigma_i^2+\sigma_j^2)^2\bigg)\\
Y_{i,t}& \sim \mathcal{N}_q\bigg(\norm{\bm{\mu}_i-\bm{\mu}_k}^2 + (\sigma_i^2+\sigma_k^2)q, 4(\sigma_i^2+\sigma_k^2)\norm{\bm{\mu}_i-\bm{\mu}_k}^2+2q(\sigma_i^2+\sigma_k^2)^2\bigg).
\end{align}

This gives us the expectation of acquisition probability as:
\begin{align}\label{eq:ld_prob}
    \mathbb{E}_{\bm{x}_i, \bm{x}_j, \bm{x}_k}[p_t] & = \sum_{i=1}^L \Bigg\{\int_0^{\infty}\Bigg[\Pi_{j=1, j\neq i}^{L-1} \Phi\Bigg(\frac{y_i-\norm{\bm{\mu}_i-\bm{\mu}_j}^2-(\sigma_i^2+\sigma_j^2)q}{(\sqrt{4(\sigma_i^2+\sigma_j^2)\norm{\bm{\mu}_i-\bm{\mu}_j}_2^2+ 2q(\sigma_i^2+\sigma_j^2)^2}}\Bigg)\Bigg]\cdot\nonumber\\
    &\phi\bigg(\frac{y_i-\norm{\bm{\mu}_i-\bm{\mu}_k}^2-(\sigma_i^2+\sigma_k^2)q}{\sqrt{4(\sigma_i^2+\sigma_k^2)\norm{\bm{\mu}_i-\bm{\mu}_k}_2^2+2q(\sigma_i^2+\sigma_k^2)^2}}\bigg)dy_i\Bigg\}/(L\cdot\delta_L),
\end{align}
where $\Phi(\cdot)$ is the cumulative distribution function (CDF) for the standard normal distribution and $\phi(\cdot)$ is the probability density function (PDF) for the standard normal distribution.

\begin{theorem}
If the streaming samples follow an independent multivariate Gaussian distribution (i.e., $\bm{\Sigma^{(i)}}=\sigma_i^2\bm{I}$), then there exist $\bm{M}_1, \bm{M}_2 \in \mathbb{R}^L$ such that if $\norm{\bm{\mu}_i-\bm{\mu}_k}^2 > M_{2,i}, \forall i \in \{1,...,L\}$, then the expected acquisition probability of a LD-Agent $
\mathbb{E}_{\bm{x}_i, \bm{x}_j, \bm{x}_k}[p_t]$ will exceed $1$, where $\bm{M}_{1}, \bm{M}_{2}$ satisfies:
\begin{align}
    &M_{2,i} + \erf^{-1}{(1-2\delta_L)}\cdot \sqrt{2\cdot\big(4(\sigma_i^2+\sigma_k^2)M_{2,i}+2q(\sigma_i^2+\sigma_k^2)^2\big)} +
    (\sigma_i^2+\sigma_k^2)q - M_{1,i} = 0\nonumber\\
    & M_{1,i}>\norm{\bm{\mu}_i-\bm{\mu}_j}^2 + (\sigma_i^2+\sigma_j^2)q + \sqrt{4(\sigma_i^2+\sigma_j^2)\norm{\bm{\mu}_i-\bm{\mu}_j}^2+2q(\sigma_i^2+\sigma_j^2)^2} \nonumber\\
    &\cdot\bigg(\Phi^{-1}(1-\frac{1}{L-1})+\gamma\bigg[\Phi^{-1}(1-\frac{1}{L-1}\cdot e^{-1})-\Phi^{-1}(1-\frac{1}{L-1}) \bigg]\bigg),\forall i \in \{1,...,L\}.
\end{align}
\end{theorem}

\begin{proof}
Based on ~\eqref{eq:7.3}, if 
$$[1-F_{Y_{i,t}}(\max_{j\neq i, j\in\{1,...,L\}}{\norm{\bm{x}_i-\bm{x}_j}_2^2)}]/\delta_L\geq 1 \ \forall i \in\{1,...,L\},$$ then the expectation of the acquisition probability
$$\mathbb{E}_{\bm{x}_i, \bm{x}_j, \bm{x}_k}[p_t]  = \sum_{i=1}^{L}[1-F_{Y_{i,t}}(\max_{j\neq i, j\in\{1,...,L\}}{\norm{\bm{x}_i-\bm{x}_j}_2^2)}]/(\delta_L\cdot L) \geq 1.
$$
Thus, to ensure $\mathbb{E}_{\bm{x}_i, \bm{x}_j, \bm{x}_k}[p_t]\geq 1$, we have:
\begin{align}
    & F^{-1}_{Y_{i,t}}(1-\delta_L) \geq \max_{j\neq i j\in\{1,...,L\}}{\norm{\bm{x}_i-\bm{x}_j}_2^2}=Z_i, \forall i \in \{1,...,L\}\\
    & \frac{Z_i - \big(\norm{\bm{\mu}_i-\bm{\mu}_k}^2 + (\sigma_i^2+\sigma_k^2)q \big)}{\sqrt{2\cdot(4(\sigma_i^2+\sigma_k^2)\norm{\bm{\mu}_i-\bm{\mu}_k}^2+2q(\sigma_i^2+\sigma_k^2)^2)}}\leq \erf^{-1}{(1-2\delta_L)}\nonumber\\
    \to & g(\norm{\bm{\mu}_i-\bm{\mu}_k}^2) = \norm{\bm{\mu}_i-\bm{\mu}_k}^2 + \erf^{-1}{(1-2\delta_L)}\cdot \sqrt{2}\cdot\nonumber\\
    & \sqrt{4(\sigma_i^2+\sigma_k^2)\norm{\bm{\mu}_i-\bm{\mu}_k}^2+2q(\sigma_i^2+\sigma_k^2)^2} +
    (\sigma_i^2+\sigma_k^2)q - Z_i \geq 0.
\end{align}

By Fisher–Tippett–Gnedenko theorem \citep{frechet1927loi} and the independent assumption on the pairwise distances, $Z_i$ can be approximated by generalized extreme value (GEV) distribution where the expectation of $Z_i$ can be estimated with:
\begin{align}
     \mathbb{E}[Z_i] & \approx \norm{\bm{\mu}_i-\bm{\mu}_j}^2 + (\sigma_i^2+\sigma_j^2)q + \sqrt{4(\sigma_i^2+\sigma_j^2))\norm{\bm{\mu}_i-\bm{\mu}_j}^2+2q(\sigma_i^2+\sigma_j^2)^2} \nonumber\\
     &\cdot\bigg(\Phi^{-1}(1-\frac{1}{L-1})+\gamma\bigg[\Phi^{-1}(1-\frac{1}{L-1}\cdot e^{-1})-\Phi^{-1}(1-\frac{1}{L-1}) \bigg]\bigg),
     \forall i, j \in \{1,...,L\}, i \neq j,
\end{align}
where $\gamma$ is the Euler–Mascheroni constant.
Accordingly, the variance is:
\begin{align}
    \mathbb{V}[Z_i] & \approx \sqrt{4(\sigma_i^2+\sigma_j^2)\norm{\bm{\mu}_i-\bm{\mu}_j}^2+2q(\sigma_i^2+\sigma_j^2)^2}\cdot \nonumber\\
    &\bigg[\Phi^{-1}(1-\frac{1}{L-1}\cdot e^{-1})-\Phi^{-1}(1-\frac{1}{L-1}) \bigg], \forall i, j \in \{1,...,L\}, i \neq j.
\end{align}

By Chebyshev's inequality, we have:
\begin{align}
    P(|Z_{i, n}-\mathbb{E}[Z_i]|\geq\epsilon)\leq \frac{\mathbb{V}[Z_i]}{n\epsilon^2}.
\end{align}
Therefore, there exist $M_{1,i}>\mathbb{E}[Z_i]$ such that $Z_{i,n} < M_{1,i}, \forall n\in \mathbb{Z}$.
Since $0<\delta_L<1$, $\erf^{-1}{(1-2\delta_L)}>0$, $g(\norm{\bm{\mu}_i-\bm{\mu}_k}^2)$ is monotonic with $\norm{\bm{\mu}_i-\bm{\mu}_k}^2$.
Hence, there exist $\bm{M}_{1}, \bm{M}_{2} \in \mathbb{R}^L$ such that if $\norm{\bm{\mu}_i-\bm{\mu}_k}^2 > M_{2,i} \ \forall i \in \{1,...,L\}$, $
\mathbb{E}_{\bm{x}_i, \bm{x}_j, \bm{x}_k}[p_t] \geq 1$, where $\bm{M}_{1}, \bm{M}_{2}$ satisfies:
\begin{align}
    &M_{2,i} + \erf^{-1}{(1-2\delta_L)}\cdot \sqrt{2\cdot \big(4(\sigma_i^2+\sigma_k^2)M_{2,i}+2q(\sigma_i^2+\sigma_k^2)^2\big)} +
    (\sigma_i^2+\sigma_k^2)q - M_{1,i} = 0\nonumber\\
    & M_{1,i}>\mathbb{E}[Z_i],\forall i, j \in \{1,...,L\}, i \neq j.
\end{align}
\end{proof}

This result illustrates that with the increasing of the distance between the center of the distribution of the observed samples and that of the incoming sample, the expected acquisition probability approaches and exceeds 1.

Numerical analysis is further investigated for the LD-Agent.
To be more intuitive, we set the samples in $\mathcal{W}$ all belonging to the same Gaussian distribution while $\bm{x}_t$ belongs to another (i.e., $\bm{\mu}_i=\bm{\mu}_j=\bm{\mu}_1$, $\sigma_i=\sigma_j=\sigma_1$, $\sigma_k = \sigma_2$).
Set $q=15$ and $\delta_L=0.05$, the expectation of the acquisition probability calculated by \eqref{eq:ld_prob} is shown as follows:
\begin{figure}[H]
    \centering
    \captionsetup{justification=centering}
    \subfloat[Set $\sigma_1=\sigma_2 = 1$, change number of points in $\mathcal{W}$ and $\norm{\bm{\mu}_1-\bm{\mu}_2}^2$. ]{\includegraphics[width=0.32\textwidth]{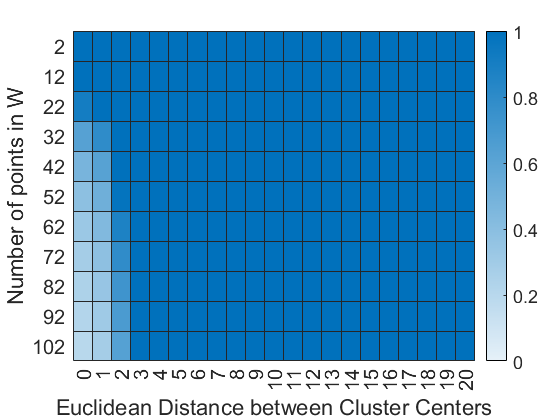}}
    \subfloat[Set $\sigma_1 = 1, \norm{\bm{\mu_i}-\bm{\mu_j}}_2^2=4$, change number of points in $\mathcal{W}$ and $\sigma_2/\sigma_1 $ ratio.]{\includegraphics[width=0.32\textwidth]{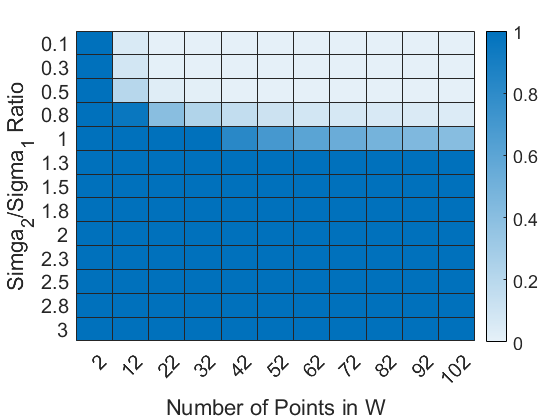}}
    \subfloat[Set $\sigma_1=1$, change number of points in $\mathcal{W}$ and $\norm{\bm{\mu}_1-\bm{\mu}_2}^2. $]{\includegraphics[width=0.32\textwidth]{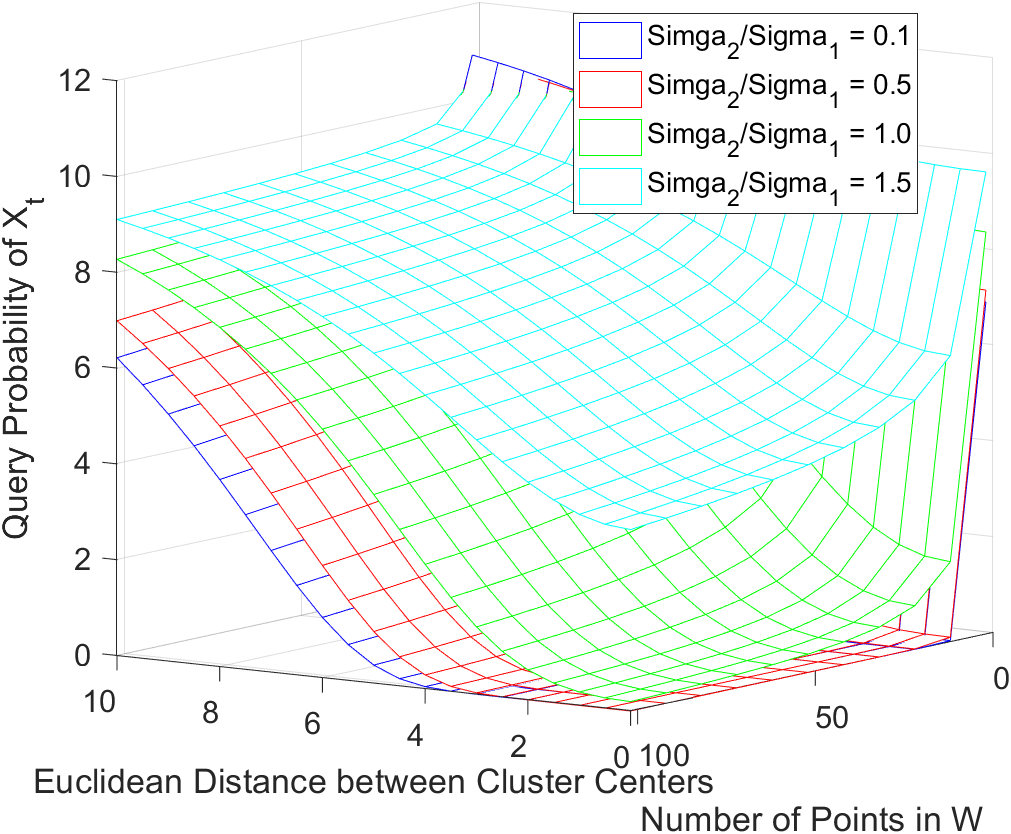}}
    \caption{Expectation of the probability of acquiring $\bm{x}_t$ by LD-Agent with $\sigma_1=1, q=15$ with different $\norm{\bm{\mu_i}-\bm{\mu_j}}_2^2$ and  $\sigma_2/\sigma_1 $ ratio }
    \label{fig:LD_char}
\end{figure}
As shown in Fig.\ref{fig:LD_char}, with a larger distance between two centers of the clusters, with less samples in the sliding window $\mathcal{W}$ and with higher $\frac{\sigma_2}{\sigma_1}$ ratio, the expectation of the acquisition probability will increase and approach to 1.
This ensures the acquisition decision of samples from a remote cluster, resulting in an increased variance and thus, exploration of the input space.

\subsection{Numerical Comparison}
For further comparison between exploration-oriented agents and exploitation oriented agents, we assume $\mathcal{W}=\mathcal{D}_t$ at time $t$, and the samples in $\mathcal{W}$ and $\mathcal{D}_t$ all  belong to the same Gaussian distribution (i.e., $\bm{x_1}\sim\mathcal{N}_q(\bm{\mu_1},\sigma_{1}^2\bm{I})$). 
Set $\sigma_1=\sigma_2=1, q=15, \delta_L=0.5$, we have the following result:

\begin{figure}[H]
    \centering
    \captionsetup{justification=centering}
    \subfloat[Set $\sigma_1=\sigma_2 = 1$, change number of points in $\mathcal{W}$ and $\norm{\bm{\mu}_1-\bm{\mu}_2}^2$. ]{\includegraphics[width=0.32\textwidth]{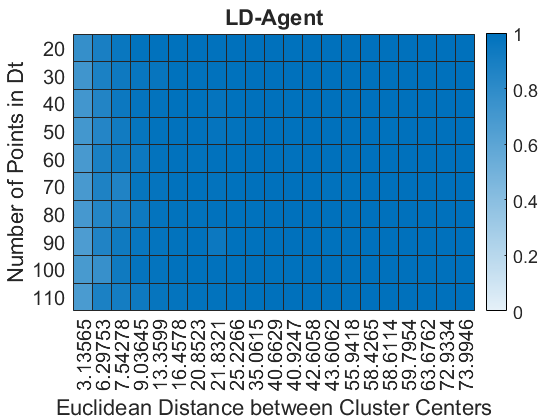}}
    \subfloat[Set $\sigma_1 = 1, \norm{\bm{\mu_i}-\bm{\mu_j}}_2^2=4$, change number of points in $\mathcal{W}$ and $\sigma_2/\sigma_1 $ ratio.]{\includegraphics[width=0.32\textwidth]{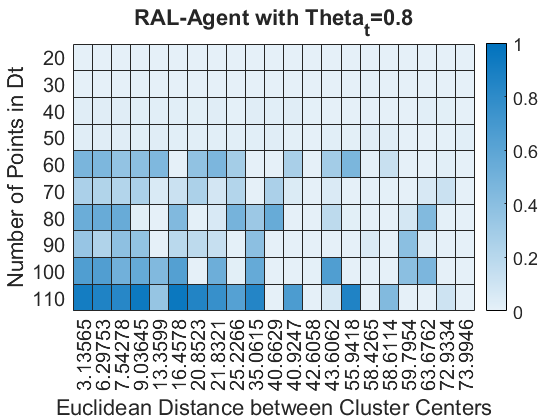}}
    \subfloat[Set $\sigma_1=1$, change number of points in $\mathcal{W}$ and $\norm{\bm{\mu}_1-\bm{\mu}_2}^2. $]{\includegraphics[width=0.32\textwidth]{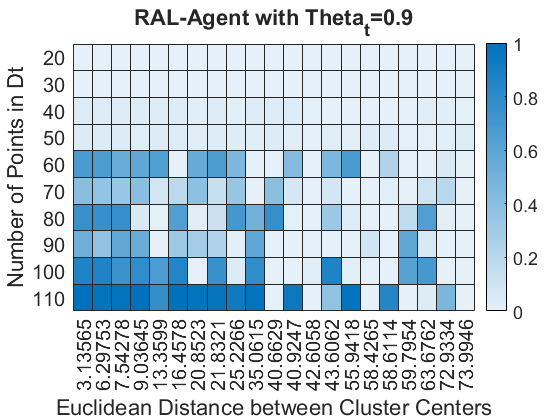}}
    \caption{Expectation of the probability of acquiring $\bm{x_t}$ of the exploration-oriented agent and exploitation-oriented agent with $\sigma_1=\sigma_2=1, q=15$, vertical axis as $||\bm{\mu_1}-\bm{\mu_2}||^2$, and horizontal axis as the window size $L$}
    \label{fig:LD_RAL_comp}
\end{figure}

In summary, the theoretical analysis and numerical results illustrate that the proposed exploration-oriented agent is more likely to acquire a sample from a distribution distinct from the observed one, thus increasing the variance of the labelled samples. 
On the other hand, the acquisition decision of exploitation-oriented agent is determined by the base learner's uncertainty instead of the distance between the observed and incoming distribution, resulting in a relative smaller variance. 
This  justifies the exploration and exploitation capability of the proposed agents.
Therefore, with the ensemble of two types of agents, the exploration and exploitation can be dynamically adjusted to balance the trade-off during the human annotation process.
To this end, in the proposed \cbeal, one RAL-agent will be paired with one SPF-Agent or LD-Agent to balance the effort spent on exploration and exploitation.

\section{Theoretical Justification on \textsc{Exp4.p-EWMA}}
To justify applying the Exponentially Weighted Moving Average (EWMA) chart to monitor the weight of each agent in \cbeal, we follow the derivation in \textsc{Exp4.P} paper \citep{beygelzimer2011contextual} to investigate its impact on the regret bound of the adversarial bandits solver.

As EWMA tracks the moving average of all previous sample means—specifically, the means of the weight $\alpha_{i, t}$ for each active learning agent ($i=1, \dots, N, t=1,2, \dots$)—it creates fluctuating control limits. Consequently, deriving a specific bound for the monitored weight $\alpha_{i, t}$ is not straightforward. 
To resolve this, we equate the effect of applying EWMA to establishing a bound for each weight, i.e., 
\begin{align}
    e < \alpha_{i, t} < f, \ 0<e<f.
\end{align}
Note that the elements in decision vector are bounded by $(0, 1)$, i.e., $0<\xi^{i}_{j}<1, \forall i=1,\dots, N, \forall a = 1, \dots, K$.
Based on Algorithm 2 (i.e., \textsc{Exp4.P-EWMA solver}), the bounds for $\hat{v}_{i,t}$ can be derived as:
\begin{align}
    0 < \hat{v}_{i,t} < \frac{K}{(1-Kp_{min})\cdot \frac{f}{e} + p_{min}},
\end{align}
where $i = 1,\dots,N, t=1, 2,\dots$. 
We found that this will only affect the bound of $\hat{\sigma}_i$ in Lemma 3 in \textsc{Exp4.P} \citep{beygelzimer2011contextual}, where $\hat{\sigma}_i \doteq \sqrt{K T}+\frac{1}{\sqrt{K T}} \sum_{t=1}^T \hat{v}_{i,t}$.
However, this will not affect the results of Lemma 3 or other theorems in the paper \citep{beygelzimer2011contextual}.
Thus, the proposed \textsc{Exp4.P-EWMA} will enjoy the same regret bound as \textsc{Exp4.P}, theoretically.

Considering the integrity, we summarize Theorem 2, Lemma 3, and Lemma 4 in \citep{beygelzimer2011contextual} in this section. 
The main result Theorem 2 is proved by Lemma 3, and Lemma 4:\\

\textbf{Theorem 2.} 
\textit{
Assume that $\ln (N / \delta) \leq K T$, and that the set of experts includes one which, on each round, selects an action uniformly at random. Then, with probability at least $1-\delta$,
$$
G_{\text {Exp4.P }} \geq G_{\max }-6 \sqrt{K T \ln (N / \delta)},
$$
where $G$ is the cumulative reward of the solver.
}
\textbf{Lemma 3.} 
\textit{
Under the conditions of Theorem 2,
$$
\operatorname{Pr}\left[\exists i: G_i \geq \hat{G}_i+\sqrt{\ln (N / \delta)} \hat{\sigma}_i\right] \leq \delta .
$$
}
\textbf{Lemma 4.} 
\textit{Under the conditions of Theorem 2,
$$
\begin{aligned}
G_{\text {Exp4.P }} \geq & \left(1-2 \sqrt{\frac{K \ln N}{T}}\right) \hat{U}-2 \sqrt{K T \ln (N / \delta)}S \\
& =\sqrt{K T \ln N}=\ln (N / \delta),
\end{aligned}
$$
where $
\hat{U}=\max _i\left(\hat{G}_i+\hat{\sigma}_i \cdot \sqrt{\ln (N / \delta)}\right)
$.
}

\section{Empirical Justification on \textsc{Exp4.p-EWMA}}
To provide empirical justification for adding the EWMA-based flipping mechanism, we visualize the standardized weight of each agent in \cbeal without the flipping mechanism (i.e., solved by \textsc{Exp4.P}) under the learning scenario $n=1000, ds = 0\%, sp = 30\%, PC = 10\%$ in Fig.\ref{fig:exp4_weight}.

\begin{figure}[H]
    \centering
    \captionsetup{justification=centering}
    \includegraphics[width=0.6\textwidth]{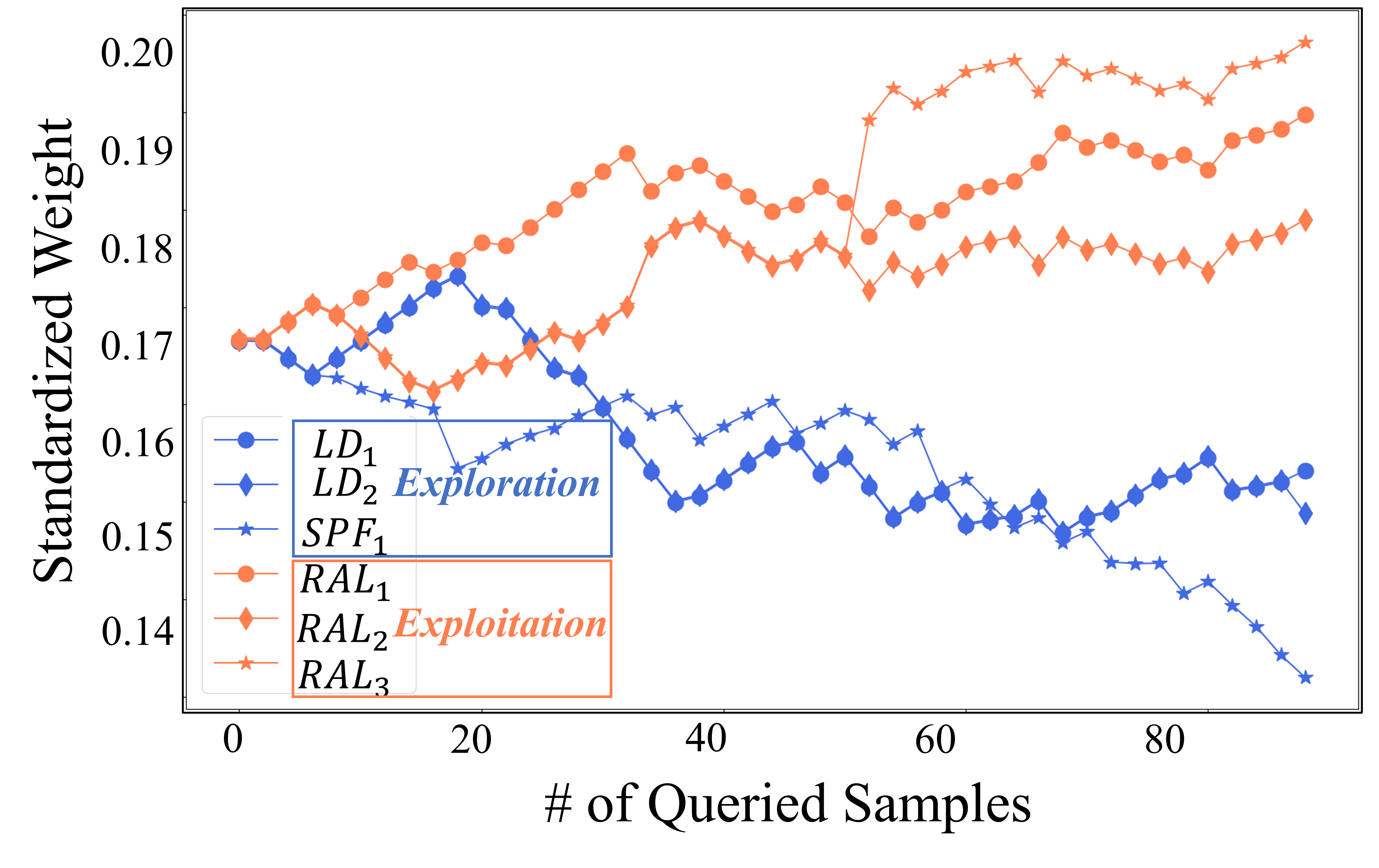}
    \caption{Standardized weight of each agents in \textsc{CbeAL-6} under the learning scenario $n = 1000, ds = 0\%, sp = 30\%, pc = 10\%$}
    \label{fig:exp4_weight}
\end{figure}

By comparing Fig.4(d) (i.e., \textsc{Exp4.p-EWMA}) and Fig.\ref{fig:exp4_weight} (i.e., \textsc{Exp4.p}), it shows that without the flipping mechanism, the decision power of most exploration-oriented agents keeps decreasing since the early stage.
Conversely, exploitation-oriented agents experience a continual increase in decision-making power. This pattern is attributed to the less effective acquisitions made by exploration-oriented agents at the start of the learning process. However, this scenario eventually results in exploitation-oriented agents dominating the active learning process, leading to insufficient exploration in later stages, particularly in the presence of shifting distributions. 
Consequently, the final testing accuracy under this scenario is considerably reduced to 0.64, as opposed to the 0.704 achieved with \textsc{Exp4.p-EWMA}. 
These results validate the effectiveness of the EWMA monitoring and flipping mechanism in the proposed \textsc{Exp4.p-EWMA}. It prevents early convergence of decision power and better balances the exploration-exploitation trade-off throughout the process.

\section{Additional Experimental Results}
\subsection{Benchmark Results in Toy Example}
To provide more concrete insights, we applied Uncertainty Sampling (US), Random Sampling (RS), DBALStream, and QBC-PYP to the two-dimensional toy example and visualized the acquired decisions, as well as the evolution of the base learners’ decision boundary in Fig.\ref{fig:simu_toy_app}.

\begin{figure}[!htb]
    \centering
    \captionsetup{justification=centering}
    \subfloat[DBALStream with final testing accuracy = 0.85, 36 acquired samples]{\includegraphics[width=0.48\textwidth]{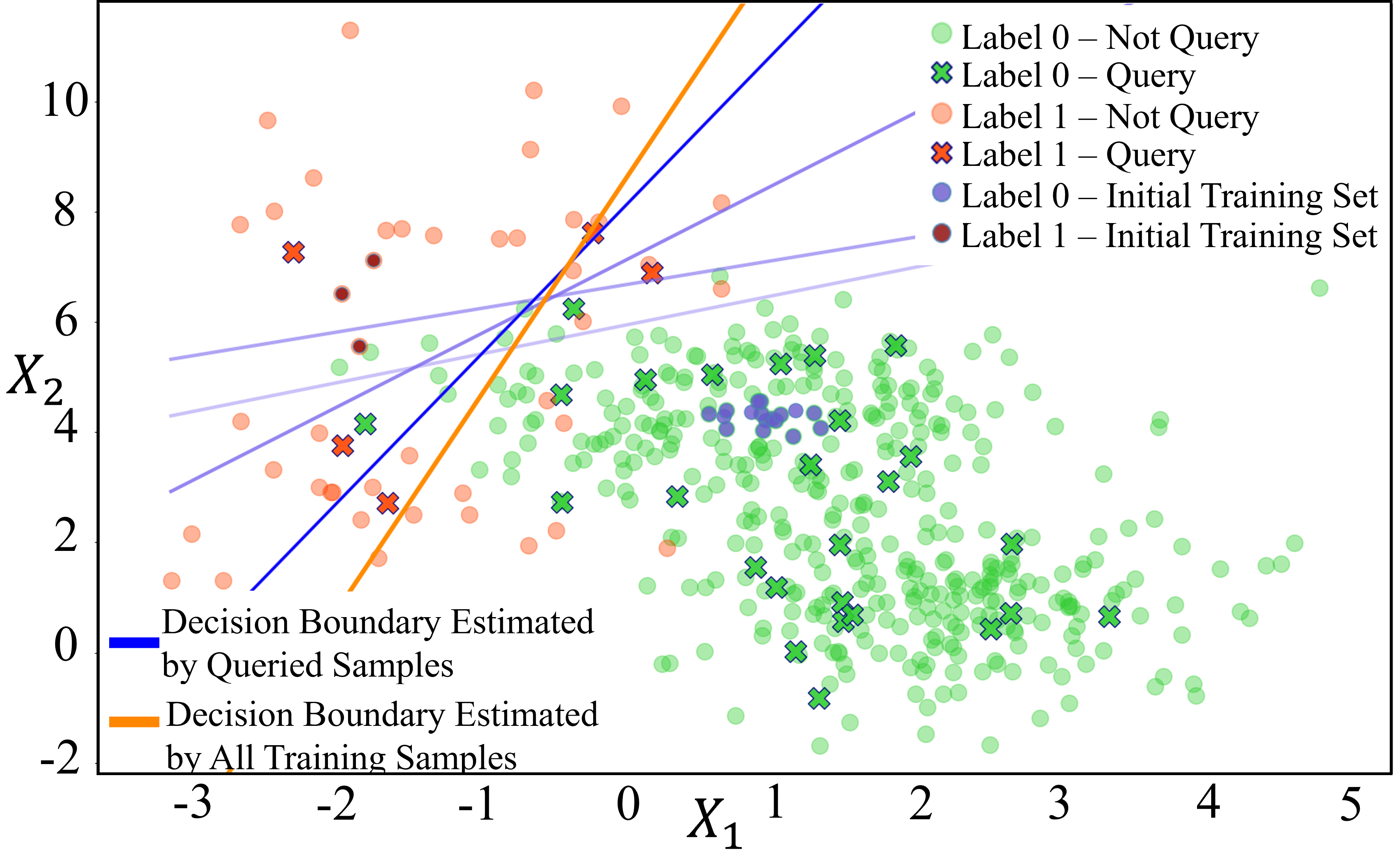}} 
    \subfloat[QBC-PYP with final testing accuracy = 0.85, 34 acquired samples ]{\includegraphics[width=0.48\textwidth]{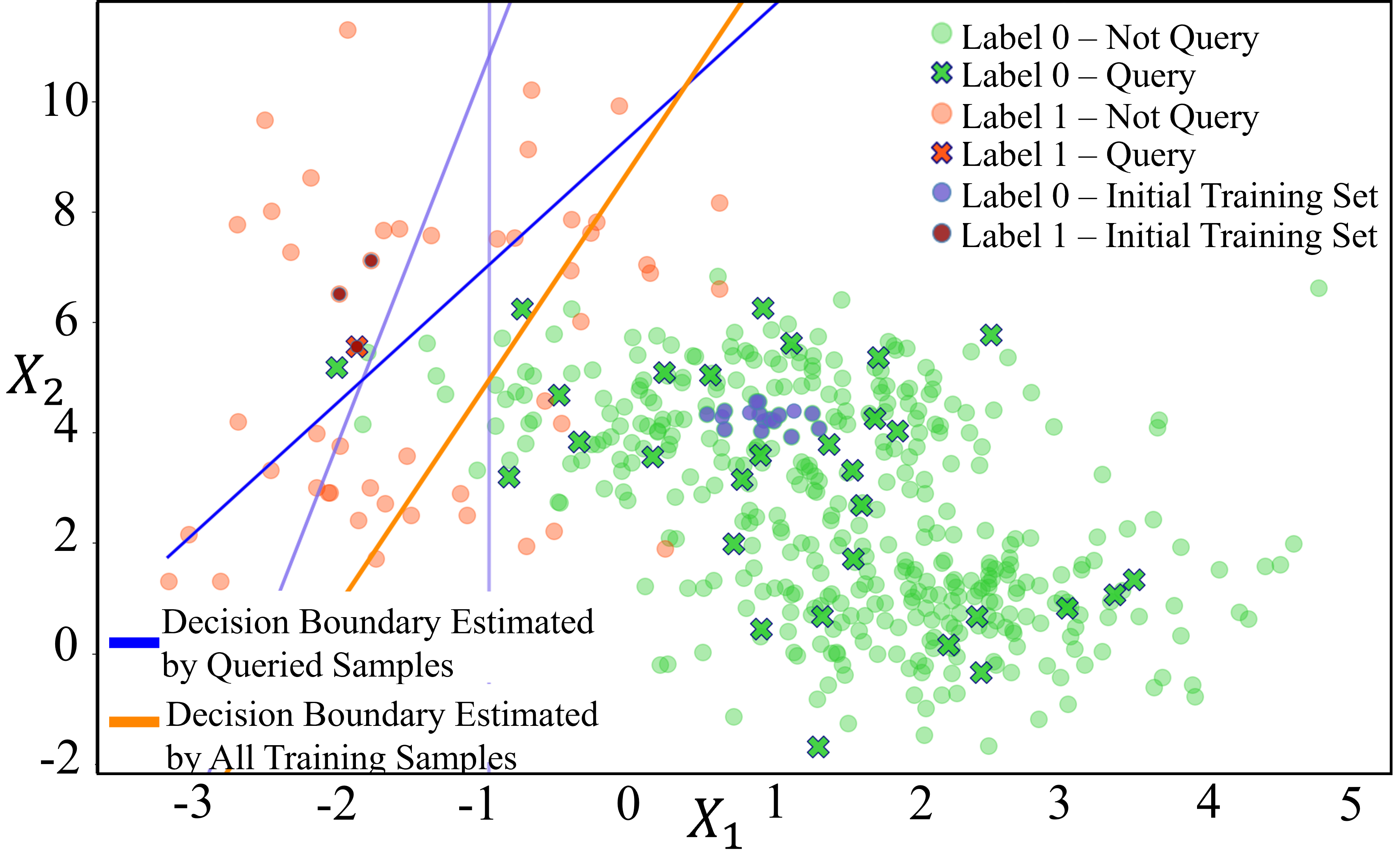}} \\
    \subfloat[US with final testing accuracy = 0.808, 18 acquired samples]{\includegraphics[width=0.48\textwidth]{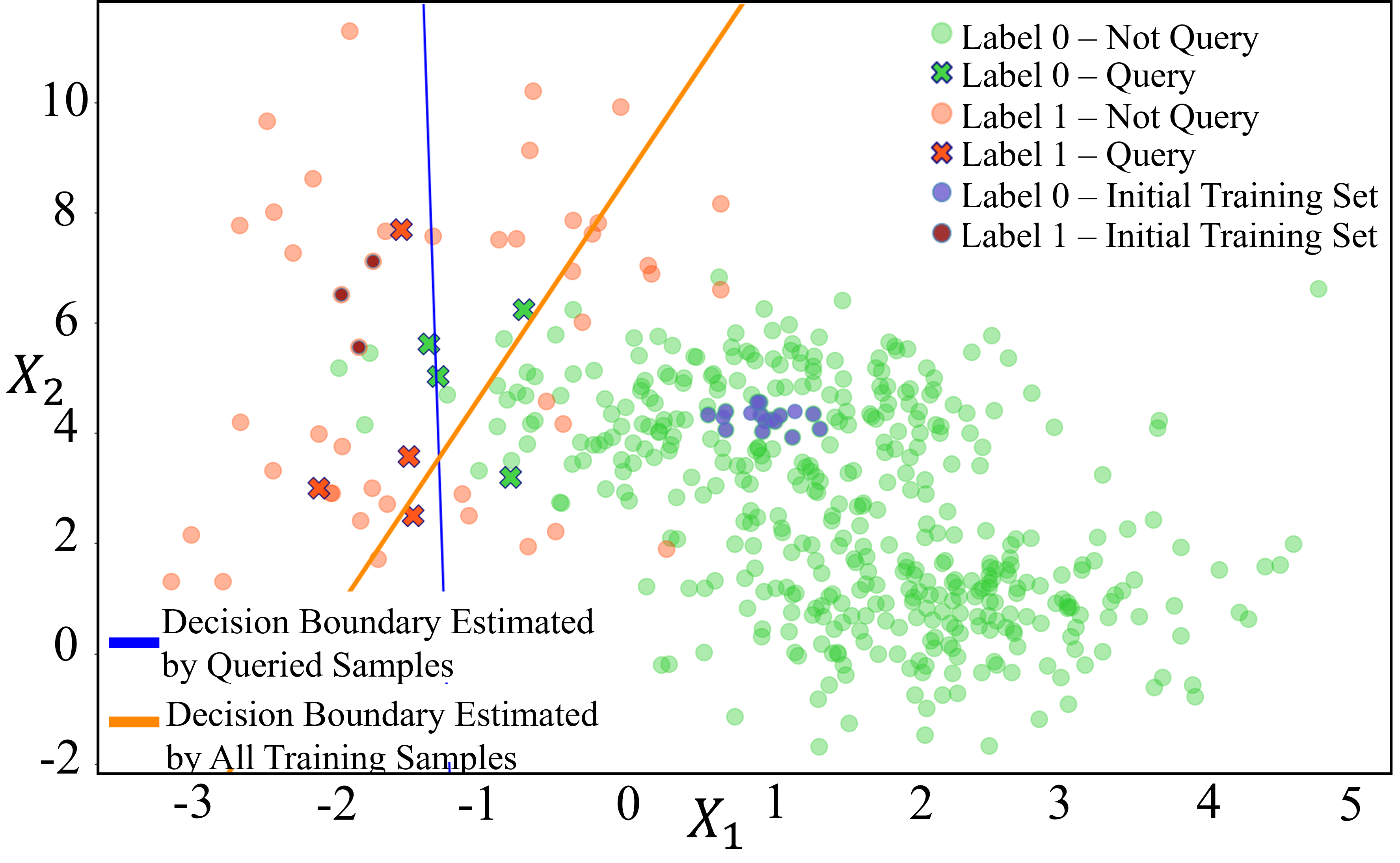}} 
    \subfloat[RS with final testing accuracy = 0.796, 48 acquired samples]{\includegraphics[width=0.48\textwidth]{Image/RAL.png}}
    \caption{Evolution of the base learner's decision boundary in the toy example: The set of blue lines represents the decision boundaries learned by the base learner every 100 time points where the color depth of the line is proportional to time $t$; The oragne line is the ground-truth decision boundary}
    \label{fig:simu_toy_app}
\end{figure}
Firstly, the results demonstrate that both DBALStream and QBC-PYP outperform US and RS in terms of learning performance. The samples procured by DBALStream are located near the classification boundary and exhibit high local density. 
However, comparing Fig.\ref{fig:simu_toy_app}(a) and Fig.3(a), it's evident that LD-agent is more effective in acquiring samples located near the boundary of the input variable space compared to DBALStream, thus facilitating more effective exploration.

Secondly, by checking \ref{fig:simu_toy_app}(b), QBC-PYP does not perform good exploitation since it acquires considerably more normal samples than abnormal samples. Upon investigating the annotation decision process, we observed that QBC-PYP often acquires sample annotations during the initial stage but seldom requests labels post the early stages, a trend that persists across both toy examples and comprehensive simulation studies. This behavior can be attributed to the disconnection between the acquisition decision and the base learner's performance, which inhibits effective learning.

Finally, US achieve inferior performance compared to DBALStream, QBC-PYP in this top example. By investigating the prediction uncertainty, it is found that the small number of acquired samples is caused by the low uncertainty of the base learner. Therefore, if the base learner is confident about its classification result at the beginning with the initial training data, then the US will achieve poor performance. As per the case study depicted in Fig.5, due to the high initial training accuracy, US exhibits poorer performance relative to other benchmarks. However, in the comprehensive simulation study shown in Table 6, the low initial training accuracy prompts US to acquire more samples with high prediction uncertainty, thereby enhancing performance.

\subsection{Scalability Study}
The classification accuracy of the base learners of \cbeal-2, \cbeal-4, \cbeal-4*, \cbeal-6, \cbeal-8 are compared to study the scalability of \cbeal.
Here, \cbeal-4* incorporate pair 1 and pair 3 in Table 1. 
We further added a new pair with the setting in Table~\ref{tab:agents4} to be incorporated into \cbeal-8.
\begin{table}[!htb]
\centering
 \caption{Agent Set Adopted in \cbeal}\label{tab:agents4}
\begin{tabular}{cccc}
\toprule
Pair Index & Agent Index & Agent & Hyperparameters \\\hline
\multirow{2}{*}{4} & $AG_7$ & $LD_3$ & $L=100, \delta_L = 0.005$\\
& $AG_8$ & $RAL_4$ & $\theta_0=0.90, \eta=0.005$\\\bottomrule
\end{tabular}
\end{table}

The results of all the variants of \cbeal in the comprehensive simulation study are individually presented in Table~\ref{tab:simu_pair} for better readability.

\begin{table}[!ht]
\caption{The average values and standard errors (in parenthesis) of the classification accuracy in the simulation study over 10 replications. Significant best results are highlighted in \textbf{bold}.}\label{tab:simu_pair}
\resizebox{\textwidth}{!}{
\begin{tabular}{ccccccccc}
\toprule
\multicolumn{2}{c}{Level} & \multirow{3}{*}{Method} & \multicolumn{3}{c}{Sparsity   = 30\%} & \multicolumn{3}{c}{Sparsity   = 70\%} \\\cline{1-2} \cline{4-9} 
\multirow{3}{*}{Disturbance} & \multirow{3}{*}{\shortstack{Percentage \\ of Positive \\ Samples}} &  & \multicolumn{3}{c}{Training Sample Size} & \multicolumn{3}{c}{Training Sample Size} \\ \cline{4-9} 
 &  &  & \multirow{2}{*}{500} & \multirow{2}{*}{1000} & \multirow{2}{*}{1500} & \multirow{2}{*}{500} & \multirow{2}{*}{1000} & \multirow{2}{*}{1500}\\
&  & \\\hline
\multirow{10}{*}{0\%}        & \multirow{5}{*}{5\%}                                                                                          & CBEAL-2                      & 58.1\% (0.02)                           & 70.3\% (0.03)                           & 71.8\% (0.03)                           & 67.5\% (0.03)                           & 75.9\% (0.02)                           & 74.3\% (0.03)                           \\
                             &                                                                                                               & CBEAL-4                      & \textbf{61.5\% (0.03)} & 67.3\% (0.03)                           & \textbf{74.3\% (0.02)} & 69.6\% (0.03)                           & \textbf{77.4\% (0.02)} & 73.1\% (0.03)                           \\
                             &                                                                                                               & CBEAL-6                      & 61.2\% (0.02)                           & \textbf{73.9\% (0.03)} & 72.5\% (0.02)                           & \textbf{69.9\% (0.03)} & 73.9\% (0.03)                           & \textbf{76.7\% (0.02)} \\
                             &                                                                                                               & \multicolumn{1}{l}{CBEAL-4*} & 59.6\% (0.03)                           & 71.2\% (0.03)                           & 71.9\% (0.02)                           & 68.4\% (0.03)                           & 75.3\% (0.02)                           & 76.5\% (0.03)                           \\
                             &                                                                                                               & \multicolumn{1}{l}{CBEAL-8}  & 58.6\% (0.02)                           & 65.2\% (0.03)                           & 70.2\% (0.02)                           & 69.2\% (0.03)                           & 76.2\% (0.03)                           & \textbf{76.7\% (0.03)} \\ \cline{2-9} 
                             & \multirow{5}{*}{10\%}                                                                                         & CBEAL-2                      & 60.6\% (0.02)                           & 70.4\% (0.03)                           & 65.2\% (0.03)                           & 63.9\% (0.03)                           & 68.7\% (0.04)                           & 67.4\% (0.03)                           \\
                             &                                                                                                               & CBEAL-4                      & 60.7\% (0.03)                           & 67.9\% (0.03)                           & 64.8\% (0.02)                           & 63.5\% (0.03)                           & 67.6\% (0.03)                           & 66.7\% (0.03)                           \\
                             &                                                                                                               & CBEAL-6                      & \textbf{65.3\% (0.02)} & \textbf{72.0\% (0.02)} & \textbf{66.7\% (0.02)} & 68.9\% (0.03)                           & \textbf{69.4\% (0.03)} & 69.0\% (0.03)                           \\
                             &                                                                                                               & \multicolumn{1}{l}{CBEAL-4*} & 60.9\% (0.03)                           & 70.1\% (0.03)                           & 66.4\% (0.03)                           & 60.1\% (0.03)                           & 69.5\% (0.03)                           & 65.7\% (0.03)                           \\
                             &                                                                                                               & \multicolumn{1}{l}{CBEAL-8}  & 63.4\% (0.03)                           & 70.2\% (0.03)                           & 64.9\% (0.02)                           & \textbf{69.2\% (0.03)} & 68.6\% (0.03)                           & \textbf{75.0\% (0.03)} \\ \hline
\multirow{10}{*}{3\%}        & \multirow{5}{*}{5\%}                                                                                          & CBEAL-2                      & 67.7\% (0.03)                           & \textbf{72.2\% (0.03)} & 73.1\% (0.03)                           & 65.1\% (0.04)                           & 69.3\% (0.02)                           & 77.5\% (0.01)                           \\
                             &                                                                                                               & CBEAL-4                      & 68.2\% (0.03)                           & 67.0\% (0.03)                           & 73.2\% (0.03)                           & 66.7\% (0.03)                           & 66.4\% (0.02)                           & 74.5\% (0.03)                           \\
                             &                                                                                                               & CBEAL-6                      & \textbf{69.0\% (0.04)} & 71.6\% (0.03)                           & \textbf{73.8\% (0.02)} & \textbf{68.0\% (0.03)} & \textbf{70.8\% (0.02)} & \textbf{79.4\% (0.02)} \\
                             &                                                                                                               & \multicolumn{1}{l}{CBEAL-4*} & 68.7\% (0.03)                           & \textbf{72.2\% (0.03)} & 70.2\% (0.03)                           & 66.8\% (0.03)                           & 64.2\% (0.02)                           & 74.8\% (0.03)                           \\
                             &                                                                                                               & \multicolumn{1}{l}{CBEAL-8}  & 66.7\% (0.03)                           & 71.4\% (0.03)                           & 69.3\% (0.03)                           & 67.7\% (0.03)                           & 66.8\% (0.03)                           & 75.5\% (0.03)                           \\ \cline{2-9} 
                             & \multirow{5}{*}{10\%}                                                                                         & CBEAL-2                      & 61.5\% (0.02)                           & 65.2\% (0.04)                           & \textbf{65.6\% (0.03)} & \textbf{60.0\% (0.03)} & 72.3\% (0.03)                           & 77.9\% (0.02)                           \\
                             &                                                                                                               & CBEAL-4                      & 57.7\% (0.02)                           & 63.2\% (0.03)                           & 65.5\% (0.03)                           & 58.7\% (0.04)                           & 72.7\% (0.02)                           & 73.5\% (0.03)                           \\
                             &                                                                                                               & CBEAL-6                      & \textbf{61.9\% (0.02)} & \textbf{68.3\% (0.03)} & 64.5\% (0.03)                           & 58.8\% (0.03)                           & \textbf{74.6\% (0.03)} & \textbf{80.1\% (0.03)} \\
                             &                                                                                                               & \multicolumn{1}{l}{CBEAL-4*} & 60.2\% (0.02)                           & 64.3\% (0.03)                           & 65.2\% (0.03)                           & 56.8\% (0.03)                           & 71.2\% (0.02)                           & 74.2\% (0.03)                           \\
                             &                                                                                                               & \multicolumn{1}{l}{CBEAL-8}  & 60.5\% (0.03)                           & 64.5\% (0.03)                           & 64.5\% (0.03)                           & 55.8\% (0.03)                           & 73.2\% (0.02)                           & 77.5\% (0.03)\\\bottomrule
\end{tabular}}

\end{table}
Comparing \cbeal-4 and \cbeal-4*, the latter generally outperforms the former. As detailed in Section 3.2.2, this performance difference is anticipated given the distinct acquisition decisions of agents pair 1 and pair 3.
Therefore, the incorporation of pair a and pair 3 (i.e., \cbeal-4*) achieves better exploration-exploitation trade-off during the learning process.
However, we will find \cbeal-6 still achieves better results, which gains more heterogeneous decisions by the incorporation of three pairs of agents.
However, this does not imply that simply adding more active learning agents invariably enhances performance. For instance, \cbeal-8 excels only in 2 of the 24 scenarios, even if its average performance marginally surpasses \cbeal-4. This is attributed to the inefficiency of \textsc{Exp4.P} when managing a vast agent pool. 
In such cases, \textsc{Exp4.P} will become inefficient in learning since it requires keeping explicit weights over the agents \citep{besbes2014stochastic}.
Therefore, \cbeal-6 is recommended as a default setting while the authors are encouraged to tune the settings based on the specific application scenario.

\subsection{Number of Acquired Samples}
We show the result of the average number of the acquired samples by the proposed method and the candidate agents in the simulation study in Table~\ref{tab:simu_numb}.

\begin{table}[H]
\caption{The average values and standard errors (in parenthesis) of the number of the acquired samples over 10 replications. The smallest numbers are highlighted in \textbf{bold}.}
\label{tab:simu_numb}
\resizebox{\textwidth}{!}{
\begin{tabular}{ccccccccc}
\toprule
\multicolumn{2}{c}{Level} & \multirow{3}{*}{Method} & \multicolumn{3}{c}{Sparsity   = 30\%} & \multicolumn{3}{c}{Sparsity   = 70\%} \\ \cline{1-2} \cline{4-9} 
\multirow{3}{*}{Disturbance} & \multirow{3}{*}{\shortstack{Percentage \\ of Positive \\ Samples}} &  & \multicolumn{3}{c}{Size of Data Stream} & \multicolumn{3}{c}{Size of Data Stream} \\ \cline{4-9} 
 &  &  & \multirow{2}{*}{500} & \multirow{2}{*}{1000} & \multirow{2}{*}{1500} & \multirow{2}{*}{500} & \multirow{2}{*}{1000} & \multirow{2}{*}{1500}\\
&  & \\\hline
\multirow{8}{*}{0\%} & \multirow{4}{*}{10\%} & Opt. Explor. & 48.00   (0.00) & 97.80 (0.19) & 148.00 (0.00) & 48.00 (0.00) & 98.00 (0.00) & 148.00 (0.00) \\
 &  & Opt. Exploit. & \textbf{37.20 (3.30)} & \textbf{81.00 (3.18)} & \textbf{108.00 (4.06)} & \textbf{44.00 (1.05)} & \textbf{79.60 (6.47)} & \textbf{114.70 (2.55)} \\
 &  & \textbf{CBEAL-2} & 45.60 (2.28) & 91.90 (2.18) & 119.20 (9.57) & 47.60 (0.38) & 98.00 (0.00) & 123.20 (9.53) \\
 &  & \textbf{CBEAL-6} & 45.80 (1.42) & 89.20 (3.25) & 132.30 (2.42) & 48.00 (0.00) & 88.20 (4.31) & 124.00 (5.78) \\\cline{2-9}
 & \multirow{4}{*}{5\%} & Opt. Explor. & 48.00 (0.00) & 98.00 (0.00) & 147.40 (0.57) & 48.00 (0.00) & 98.00 (0.00) & 147.00 (0.95) \\
 &  & Opt. Exploit. & \textbf{30.90 (2.68)} & \textbf{58.00 (3.14)} & \textbf{71.80 (5.26)} & \textbf{37.50 (2.38)} & \textbf{59.20 (3.55)} & \textbf{80.10 (3.51)} \\
 &  & \textbf{CBEAL-2} & 40.30 (2.83) & 68.70 (2.48) & 78.20 (6.11) & 43.60 (1.86) & 70.60 (3.80) & 83.70 (5.56) \\
 &  & \textbf{CBEAL-6} & 44.10 (1.78) & 68.10 (2.73) & 83.40 (2.78) & 43.60 (1.02) & 68.70 (4.35) & 83.20 (4.19) \\\hline
\multirow{8}{*}{3\%} & \multirow{4}{*}{10\%} & Opt. Explor. & 47.40 (0.57) & 97.70 (0.29) & 148.00 (0.00) & 48.00 (0.00) & 97.80 (0.19) & 148.00 (0.00) \\
 &  & Opt. Exploit. & \textbf{41.50 (3.28)} & \textbf{71.90 (3.86)} & \textbf{109.10 (6.25)} & 44.80 (0.99) & \textbf{81.40 (6.74)} & \textbf{102.50 (9.60)} \\
 &  & \textbf{CBEAL-2} & 48.00 (0.00) & 90.70 (0.95) & 120.60 (9.60) & 48.00 (0.00) & 97.00 (0.63) & 127.60 (5.84) \\
 &  & \textbf{CBEAL-6} & 47.20 (0.76) & 87.80 (6.37) & 136.90 (3.50) & \textbf{43.70 (1.91)} & 92.40 (3.30) & 129.20 (6.65) \\ \cline{2-9} 
 & \multirow{4}{*}{5\%} & Opt. Explor. & 47.70 (0.20) & 97.00 (0.95) & 148.00 (0.00) & 47.70 (0.20) & 98.00 (0.00) & 148.00 (0.00) \\
 &  & Opt. Exploit. & \textbf{30.10 (1.78)} & \textbf{51.80 (4.29)} & 84.10 (2.72) & \textbf{30.10 (1.78)} & \textbf{56.90 (4.48)} & \textbf{86.20 (4.10)} \\
 &  & \textbf{CBEAL-2} & 40.90 (1.57) & 64.60 (4.18) & 89.80 (6.87) & 40.90 (1.57) & 73.10 (3.97) & 101.60 (2.91) \\
 &  & \textbf{CBEAL-6} & 38.70 (2.87) & 61.60 (4.10) & \textbf{83.00 (5.99)} & 38.70 (2.81) & 65.60 (4.52) & 87.60 (4.56) \\ \bottomrule
\end{tabular}}
\end{table}
It is shown that the exploitation-oriented agent tends to acquire less samples under most of the scenarios whereas \cbeal (i.e., \cbeal-2 and \cbeal-6) lie in the middle of the two incorporated agents.
This is reasonable since the ensemble method encodes the exploration behaviour during the process, which will lead to a larger number of queries compared to pure exploitation. 

\subsection{Evaluation of the Bandit's Solver}

To evaluate the performance of the proposed bandit's solver, we track and summarize the cumulative reward gained by the solvers at the end of the online updating process and summarize them in Table ~\ref{tab:simun_reward}.
Here, the variants of \cbeal (i.e., \cbeal-2, \cbeal-4, and \cbeal-6) are included along with the top-performing exploitation-oriented agents achieving the highest accuracy on average (the same agent as listed in Table 2).

\begin{table}[H]
\caption{The average values and standard errors (in parenthesis) of cumulative reward over 10 replications. The smallest numbers are highlighted in \textbf{bold}.}
\label{tab:simun_reward}
\resizebox{\textwidth}{!}{
\begin{tabular}{ccccccccc}
\toprule
\multicolumn{2}{c}{Level} & \multirow{3}{*}{Method} & \multicolumn{3}{c}{Sparsity   = 30\%} & \multicolumn{3}{c}{Sparsity   = 70\%} \\ \cline{1-2} \cline{4-9} 
\multirow{3}{*}{Disturbance} & \multirow{3}{*}{\shortstack{Percentage \\ of Positive \\ Samples}} &  & \multicolumn{3}{c}{Size of Data Stream} & \multicolumn{3}{c}{Size of Data Stream} \\ \cline{4-9} 
 &  &  & \multirow{2}{*}{500} & \multirow{2}{*}{1000} & \multirow{2}{*}{1500} & \multirow{2}{*}{500} & \multirow{2}{*}{1000} & \multirow{2}{*}{1500}\\
&  & \\\hline
\multirow{8}{*}{0\%}         & \multirow{4}{*}{10\%}                                                                                         & \textbf{CBEAL-2} & \textbf{11.15 (0.59)}            & 14.05 (1.39)             & 20.85 (1.54)             & 10.30 (1.07)            & 12.05 (1.64)             & 22.25 (2.00)             \\
                             &                                                                                                               & \textbf{CBEAL-4} & 11.50 (0.98)            & 15.00 (0.79)             & \textbf{23.60 (2.06)}             & 11.65 (0.84)            & 10.50 (2.07)             & \textbf{26.75 (1.63)}            \\
                             &                                                                                                               & \textbf{CBEAL-6} & 9.70 (0.78)             & \textbf{15.30 (1.36)}             & 22.05 (1.40)             & \textbf{13.35 (1.41)}            & \textbf{13.65 (2.51)}             & 25.00 (2.66)             \\
                             &                                                                                                               & Opt. Exploit.                   & 7.80 (0.97)             & 11.35 (1.45)             & 19.30 (2.03)             & 10.50 (1.52)            & 11.25 (2.20)             & 18.65 (2.14)             \\ \cline{2-9} 
                             & \multirow{4}{*}{5\%}                                                                                          & \textbf{CBEAL-2} & 8.55 (0.83)             & 9.95 (1.80)              & 18.90 (1.04)             & 8.10 (1.17)             & 9.30 (0.99)              & 19.40 (1.99)             \\
                             &                                                                                                               & \textbf{CBEAL-4} & 8.80 (0.47)             & 9.35 (1.41)              & 20.35 (1.24)             & 5.80 (0.75)             & 9.85 (1.00)              & 20.05 (1.34)             \\
                             &                                                                                                               & \textbf{CBEAL-6} & \textbf{9.45 (0.53)}             & \textbf{10.15 (1.38)}             & \textbf{22.95 (1.46)}             & \textbf{8.35 (0.92)}             & \textbf{10.90 (1.24)}             & \textbf{21.60 (1.15)}             \\
                             &                                                                                                               & Opt. Exploit.                   & 7.96 (0.68)             & 9.15 (1.28)              & 19.65 (1.65)             & 8.12 (0.96)             & 9.85 (1.13)              & 20.56 (1.27)             \\ \hline
\multirow{8}{*}{3\%}         & \multirow{4}{*}{10\%}                                                                                         & \textbf{CBEAL-2} & 13.65 (1.03)            & 15.35 (1.44)             & 22.40 (1.37)             & 10.85 (0.91)            & 15.55 (0.83)             & 26.80 (2.50)             \\
                             &                                                                                                               & \textbf{CBEAL-4} & 14.10 (0.67)            & \textbf{17.95 (1.56)}             & 20.70 (2.06)             & 11.70 (0.88)            & 14.75 (1.22)             & 26.40 (1.83)             \\
                             &                                                                                                               & \textbf{CBEAL-6} & \textbf{14.40 (0.90)}            & 15.15 (0.94)             & \textbf{24.35 (1.11)}             & 12.25 (0.59)            & \textbf{16.70 (1.29)}             & \textbf{29.10 (2.03)}             \\
                             &                                                                                                               & Opt. Exploit.                   & 10.25 (0.79)            & 15.05 (1.36)             & 22.56 (1.26)             & \textbf{12.80 (0.96)}            & 16.40 (1.37)             & 25.60 (2.25)             \\ \cline{2-9} 
                             & \multirow{4}{*}{5\%}                                                                                          & \textbf{CBEAL-2} & 7.95 (0.77)             & 11.80 (1.31)             & 12.45 (1.12)             & 8.15 (0.95)             & 10.90 (1.10)             & 17.70 (1.82)             \\
                             &                                                                                                               & \textbf{CBEAL-4} & 8.25 (0.91)             & 14.40 (1.82)             & 11.24 (2.05)             & 8.20 (0.71)             & 11.45 (1.32)             & 16.50 (2.23)             \\
                             &                                                                                                               & \textbf{CBEAL-6} & \textbf{8.90 (0.81)}             & \textbf{14.45 (1.59)}             & \textbf{14.85 (1.31)}             & 8.80 (0.81)             & 12.10 (1.12)             & \textbf{19.80 (2.02)}             \\
                             &                                                                                                               & Opt. Exploit.                  & 6.85 (1.03)             & 13.65 (1.70)             & 14.65 (1.75)             & \textbf{9.25 (0.85)}             & \textbf{12.90 (1.27)}             & 16.40 (2.22)   \\ \bottomrule         
\end{tabular}}
\end{table}

By investigating the results, it shows that under most scenarios, \cbeal-6 will gain the highest cumulative reward. However, it may result in higher standard error in the cumulative reward, which is caused by the ensemble of more agents compared to \cbeal-2 and \cbeal-4. 
In general, the bandit's solver for \cbeal and its variants are effectively learning to gain higher rewards during the online acquisition process based on the proposed contextual information and reward function.
While the designed reward is not representing the learning performance of the base learner in a straightforward manner (i.e., a higher cumulative reward does not necessarily lead to a better learning performance), we still observe the positive correlation between the cumulative reward and the learning accuracy (Table 2). 
This also indicates the effectiveness of the designed reward function.
Additionally, the cumulative rewards attained by \cbeal and its variants consistently surpass those by the RAL-agent in most scenarios. 
This implies the superior performance of the proposed bandit's solver compared to the simple reinforced mechanism in RAL.

We further depict the trend of the cumulative reward of four solvers in Fig~\ref{fig:Reward}.
During the early stage, most solvers get more penalties compared to reward while in the later stage, \cbeal-2, \cbeal-4, and \cbeal-6 efficiently gain the higher reward compared to the benchmark. 
The optimal RAL agent does not make sufficient acquisitions due to the overestimated confidence on the new samples. 
Among its variants, \cbeal-6 achieves the best performance in this scenario by making most efficient and informative acquisitions.
\begin{figure}[H]
    \centering
    \captionsetup{justification=centering}
    \includegraphics[width=0.6\textwidth]{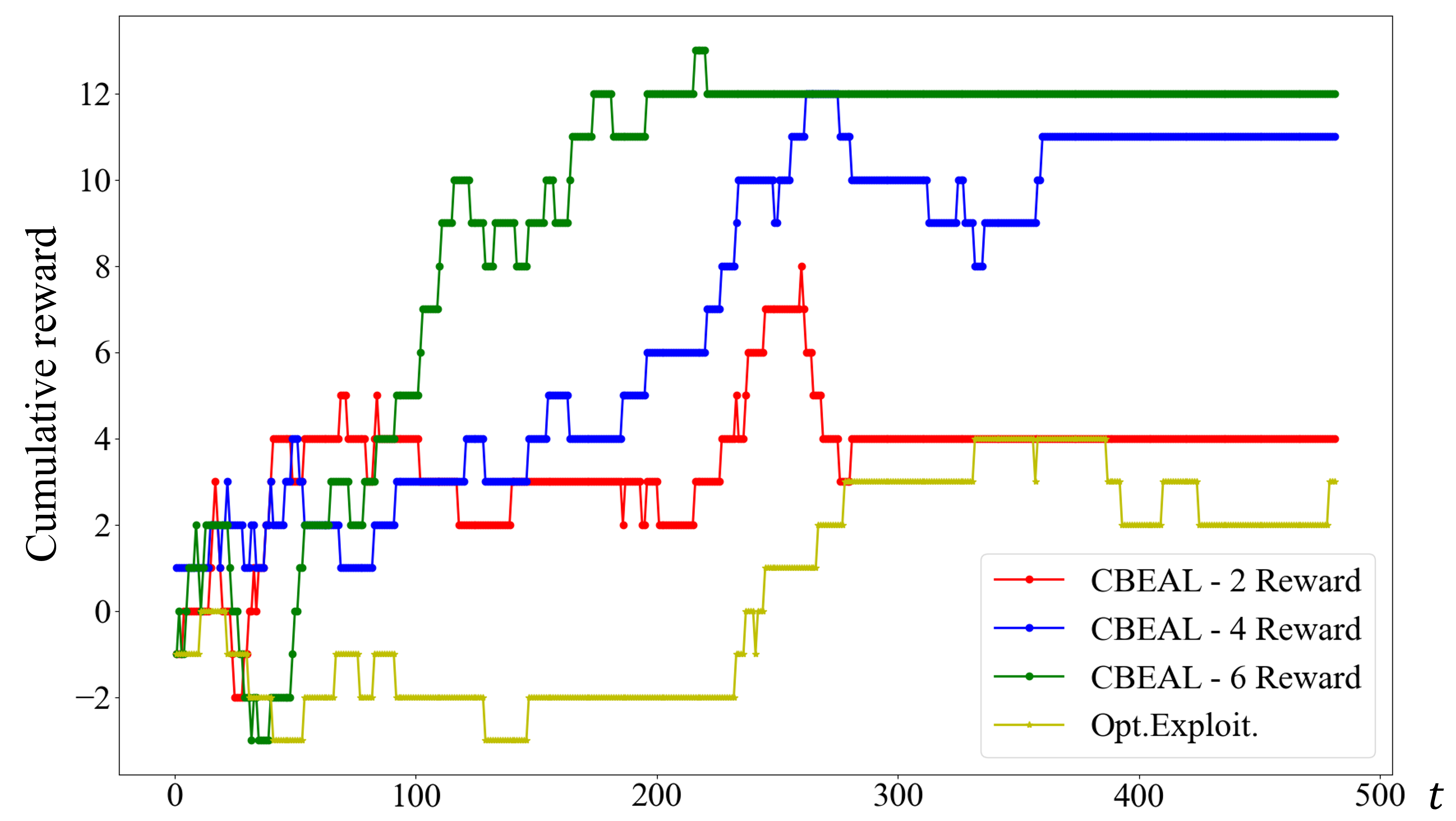}
    \caption{Cumulative reward gained during the learning scenario $n=400, ds = 0\%, sp = 30\%, PC = 10\%$}
    \label{fig:Reward}
\end{figure}

\subsection{Computational Time}
We measure the computational time of the proposed method in one replication of all simulation scenarios.
The simulation is implemented in Python 3.7.6 on a workstation with 3.70 GHz AMD Ryzen 5 5600X 6-Core Processor, 16.0 GB RAM and Windows 10.
It takes on average of 0.086 seconds to make an acquisition decision and train the \cbeal-6 model.
This guarantees the practical implementation of the proposed method in manufacturing processes.

\section{Assumption on Input Data Distribution}
\begin{figure}[H]
    \centering
    \captionsetup{justification=centering}
    \includegraphics[width=0.6\textwidth]{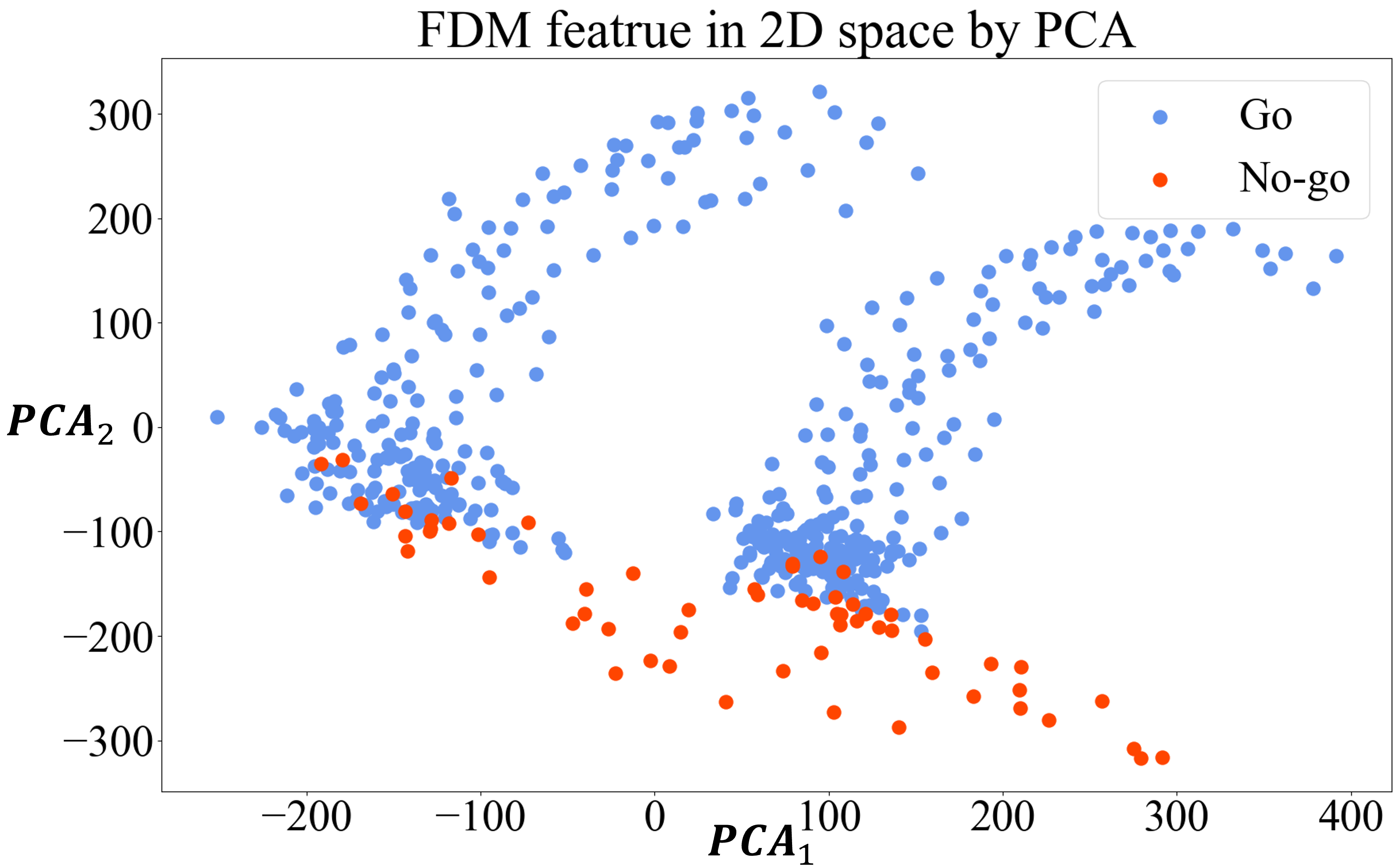}
    \caption{Distribution of the first 500 samples with reduced dimensions by PCA in case study }
    \label{fig:FDM_2d}
\end{figure}
To validate the third assumption made on the input data distribution, the input $\bm{X}$ of the first 500 samples in the FDM process is reduced to the two-dimensional space by principal component analysis (PCA) and visualized in Fig.\ref{fig:FDM_2d}.
It can be clearly observed that there exist multiple clusters in the class of good samples, which validates the assumption on the online data stream in the case study.

\end{document}

%% file: math_commands.tex
\usepackage{amsthm}
\usepackage{amsmath}
\usepackage{amssymb}

\newcommand{\diag}{\operatorname{diag}}

\numberwithin{equation}{section}
\newcommand{\cbeal}{\textsc{CbeAL}\xspace}
\newcommand{\norm}[1]{\left\lVert#1\right\rVert}
\newcommand{\RN}[1]{%
  \textup{\lowercase\expandafter{\romannumeral#1}}%
}